\documentclass{article}

\usepackage{microtype}
\usepackage{graphicx}
\usepackage{subfigure}
\usepackage{booktabs} %

\usepackage{hyperref}

\usepackage[accepted]{icml2023}

\usepackage{amsmath}
\usepackage{amssymb}
\usepackage{mathtools}
\usepackage{amsthm}

\usepackage{nicefrac}       %

\usepackage[capitalize,noabbrev]{cleveref}

\theoremstyle{plain}
\newtheorem{theorem}{Theorem}[section]

\newtheorem{lemma}[theorem]{Lemma}
\newtheorem{corollary}[theorem]{Corollary}
\theoremstyle{definition}

\newtheorem{assumption}[theorem]{Assumption}
\theoremstyle{remark}
\newtheorem{remark}[theorem]{Remark}

\usepackage[textsize=tiny]{todonotes}

\definecolor{darkgreen}{HTML}{02862a}

\usepackage{preamble}

\usepackage{xspace}
\newcommand{\algname}[1]{{\sf \footnotesize #1}\xspace}

\usepackage[shortlabels]{enumitem} 

\usepackage{hyperref}
\usepackage{url}

\icmltitlerunning{Reinforcement Learning with General Utilities: Simpler Variance Reduction and Large State-Action Space}

\begin{document}

\twocolumn[
\icmltitle{Reinforcement Learning with General Utilities:\\
Simpler Variance Reduction and Large State-Action Space}

\icmlsetsymbol{equal}{*}

\begin{icmlauthorlist}
\icmlauthor{Anas Barakat}{ETH_CS}
\icmlauthor{Ilyas Fatkhullin}{ETH_CS}
\icmlauthor{Niao He}{ETH_CS}
\end{icmlauthorlist}

\icmlaffiliation{ETH_CS}{Department of Computer Science, ETH Zurich, Switzerland}
\icmlcorrespondingauthor{A.B.}{anas.barakat@inf.ethz.ch}

\icmlkeywords{Reinforcement Learning with general utility, variance reduction}

\vskip 0.3in
]

\printAffiliationsAndNotice{}  %

\begin{abstract}
We consider the reinforcement learning (RL) problem with general utilities which consists in maximizing a function of the state-action occupancy measure. 
Beyond the standard cumulative reward RL setting,  this problem includes as particular cases constrained~RL, pure exploration and learning from demonstrations among others. For this problem, we propose a simpler single-loop parameter-free normalized policy gradient algorithm. Implementing a recursive momentum variance reduction mechanism, our algorithm achieves~$\tilde{\mathcal{O}}(\epsilon^{-3})$ and~$\tilde{\mathcal{O}}(\epsilon^{-2})$ sample complexities for~$\epsilon$-first-order stationarity and~$\epsilon$-global optimality respectively, under adequate assumptions. We further address the setting of large finite state action spaces via linear function approximation of the occupancy measure and show a~$\tilde{\mathcal{O}}(\epsilon^{-4})$  
sample complexity 
for a simple policy gradient method with a linear regression subroutine.
\end{abstract}

\section{Introduction}

While the classical Reinforcement Learning (RL) problem consists in learning a policy maximizing the expected cumulative sum of rewards through interaction with an environment,
several other problems of practical interest are concerned with objectives involving more general utilities.  
Examples of such problems 
include pure exploration in RL via maximizing the entropy of the state visitation distribution (see for e.g., \citet{hazan-et-al19,mutti-desanti-et-al22}), imitation learning via minimizing an $f$-divergence between state-action occupancy measures of an agent and an expert \citep{ghasemipour-et-al20}, risk-sensitive~\citep{zhang-et-al21cautious} or risk-averse RL maximizing, for instance, the Conditional Value-at-Risk \citep{garcia-fernandez15}, constrained RL (see for e.g.,  \citet{altman99cmdps,borkar05,bhatnagar-lakshmanan12,miryoosefi-et-al19rl-with-convex-constraints,efroni-mannor-pirotta20}), experiment design~\citep{mutny-et-al23} and diverse skill discovery \citep{eysenbach-et-al19} among others. 
We refer the interested reader to Table~1 in~\citet{zahavy-et-al21,mutti-et-al22}, \citet{zhang-et-al20variational} and references therein for further examples and a more comprehensive description of such problems. 

Recently, \citet{zhang-et-al20variational,zhang-et-al21} proposed a unified formulation encapsulating all the aforementioned problems as a maximization of a functional (which may not be concave) over the set of state-action occupancy measures. Interestingly, this formulation generalizes standard RL which corresponds to maximizing a linear functional of the state-action occupancy measure \citep{puterman14}. Subsuming the standard RL problem, the case where the objective functional is convex (concave for maximization) in the occupancy measure is known as Convex RL~\citep{zhang-et-al20variational,zahavy-et-al21,geist-et-al22,mutti-et-al22}. 

Unlike the standard RL problem which enjoys a nice additive structure, the more general nonlinear functional alters the additive structure of the problem, invalidates the classical Bellman equations as a consequence and hence hinders the standard use of the dynamical programming machinery (see for e.g., \citet{bertsekas19book,sutton-barto18}). While value-based methods are not meaningful anymore in this general (nonlinear) utilities setting, \citet{zhang-et-al20variational,zhang-et-al21} proposed a direct policy search method to solve the RL problem with general utilities. This class of methods directly updates a parametrized  policy along the gradient direction of the objective function. More precisely, \citet{zhang-et-al21} propose a double-loop Policy Gradient (PG) method called \algname{TSIVR-PG} implementing a variance reduction mechanism requiring two large batches and checkpoints. Similarly to existing variance-reduced PG methods in the standard RL setting, the algorithm makes use of importance sampling (IS) weights to account for the distribution shift inherent to the RL setting. 
Interestingly, while most existing variance-reduced PG methods make an unrealistic and unverifiable assumption which guarantees that the IS weights variance is bounded at each iteration of the algorithm, \citet{zhang-et-al21} alleviate this issue by using a gradient truncation mechanism. Such a strategy consisting in performing a truncated gradient step can be formulated as solving a trust-region subproblem at each iteration, which is reminiscent of trust-region based algorithms such as TRPO~\citep{schulman-et-al15trpo} and PPO~\citep{schulman-et-al17ppo}. In particular, implementing TSIVR-PG requires tuning a gradient truncation radius depending on problem parameters while also choosing adequate large batches. Besides these algorithmic considerations, a major limitation of recent prior work \citep{zhang-et-al21,zhang-et-al20variational,kumar-et-al22pg-general} 
is the need to estimate the unknown occupancy measure at each state-action pair. In several problems of practical scale, the number of states and/or actions is prohibitively large and renders tabular methods intractable. For instance, the size of a state space grows exponentially with the number of state variables. This is commonly known as the curse of dimensionality. 

In this paper, we consider the RL problem with general utilities. 
Our contributions are as follows: 
\begin{itemize}[leftmargin=4mm,topsep=0pt]
\item We propose a novel single-loop normalized PG algorithm called~\algname{N-VR-PG} using only a single trajectory per iteration. In particular, our %
algorithm does not require 
the knowledge of problem specific parameters, 
large batches nor checkpoints unlike TSIVR-PG in \citet{zhang-et-al21}. Instead of gradient truncation, we propose to use a normalized update rule for which no additional gradient truncation hyperparameter is needed. At the heart of our algorithm design is a recursive double variance reduction mechanism 
implemented with momentum
for both the stochastic policy gradient and the occupancy measure estimator (in the tabular setting), akin to STORM~\citep{cutkosky-orabona19} in stochastic optimization. 

\item We show that using a normalized gradient update guarantees %
bounded IS weights for the softmax parametrization. Unlike in most prior works focusing on the particular case of the standard RL setting, variance of IS weights is automatically bounded 
and no further assumption is needed. We further demonstrate that 
IS weights can also be similarly controlled  
when using a gaussian policy for continuous state-action spaces under mild assumptions. 

\item In the general utilities setting with finite state-action spaces and softmax policy, we show that our algorithm requires~$\tilde{\mathcal{O}}(\varepsilon^{-3})$ samples to reach an~$\varepsilon$-stationary point of the objective function and~$\tilde{\mathcal{O}}(\varepsilon^{-2})$ samples to reach an~$\varepsilon$-globally optimal policy by exploiting the hidden concavity of the problem  when the utility function is concave and the policy is overparametrized. 
In the standard RL setting, %
we further show that such sample complexity results also hold for continuous state-action spaces when using the gaussian policy under adequate assumptions. 

\item Beyond the tabular setting, we consider the case of large finite state and action spaces which has not been previously addressed in this general setting to the best of our knowledge. We consider approximating the unknown state-action occupancy measure itself by a linear combination of pre-selected basis functions via a least-mean-squares solver. This linear function approximation procedure combined with a stochastic policy gradient method results in an algorithm for solving the RL problem with general nonlinear utilities for large state and action spaces. Specifically, we show that our PG method requires~$\tilde{\mathcal{O}}(\varepsilon^{-4})$ samples to guarantee an~$\varepsilon$-first-order stationary point of the objective function up to an error floor due to function approximation.   
\end{itemize}

\noindent\textbf{Related works.}
We briefly discuss standard RL before closely related works for RL with general utility.

\noindent\textbf{Variance-reduced PG for standard RL.}
In the last few years, there has been a vast array of work around %
variance-reduced PG methods for solving the standard RL problem with a cumulative sum of rewards to reduce the high variance of the stochastic policy gradients (see for e.g., \citet{papini-et-al18,xu-et-al20iclr,pham-et-al20,gargiani-et-al22}). \citet{yuan-et-al20,huang-et-al20} proposed momentum-based policy gradient methods.
All the aforementioned works use IS and make an unverifiable assumption stipulating that the IS weights variance is bounded. To relax this unrealistic assumption, \citet{zhang-et-al21} provide a gradient truncation mechanism complementing IS for the specific case of the softmax parameterization whereas \citet{shen-et-al19,salehkaleybar-et-al22} incorporate second-order information for which IS is not needed. 
Even in the special case of standard cumulative reward, our algorithm differs from prior work in that it combines the following features: it is single-loop, runs with a single trajectory per iteration and uses a normalized update rule to control the IS weights without further assumption. In particular, our algorithm does not make use of second order information and thus our analysis does not require second-order smoothness conditions. 
Typically, %
variance-reduced PG methods guarantee a~$\tilde{\mathcal{O}}(\varepsilon^{-3})$ sample complexity to reach a first-order stationary policy, improving over its~$\tilde{\mathcal{O}}(\varepsilon^{-4})$ counterpart for vanilla~PG. Subsequently to the recent work of~\citet{agarwal-et-al21} which provided global optimality guarantees for PG methods despite the non-concavity of the problem, several works \citep{liu-et-al20,zhang-et-al21,ding-et-al21,ding-et-al22,Vanilla_PL_Yuan_21,Masiha_SCRN_KL,yuan-et-al22log-linear} established global optimality guarantees for stochastic PG methods with or without variance reduction under policy parametrization. The best known sample complexity to reach an~$\epsilon$-globally optimal policy is~$\tilde{\mathcal{O}}(\epsilon^{-2})$ and was achieved via policy mirror descent without parametrization \citep{lan22,xiao22}, with log-linear policies recently \citep{yuan-et-al22log-linear} and via variance-reduced PG for softmax parametrization by exploiting hidden convexity \citep{zhang-et-al21}. 
Very recently, \citet{Fatkhullin_SPG_FND_2023} obtained a~$\tilde{\cO}(\epsilon^{-2})$ sample complexity for Fisher-non-degenerate parametrized policies.%
  
\noindent\textbf{RL with General Utility.}
There is a huge literature %
addressing control problems with nonstandard utilities %
that we cannot hope to give justice to. Let us mention though some early examples %
in Operations Research such as inventory problems with constraints on the probability of shortage~\citep{derman-klein65} and variance-penalized MDPs~\citep{filar-et-al89variance-pen,kallenberg94survey} where the problem is formulated as a nonlinear program in the space of state-action frequencies. In the rest of this section, we briefly discuss the most relevant research to the present paper. %
\citet{zhang-et-al20variational} study the policy optimization problem where the objective function is a concave function of the state-action occupancy measure to include several known problems such as constrained MDPs, exploration and learning from demonstrations. To solve this problem for which dynamic programming cannot be employed, \citet{zhang-et-al20variational} investigate policy search methods and first define a variational policy gradient for RL with general utilities as the solution to a stochastic saddle point problem. Exploiting the hidden convexity structure of the problem, they further show global optimality guarantees when having access to exact policy gradients. However, the procedure to estimate even a single policy gradient via the proposed primal-dual stochastic approximation method from sample paths turns out to be complex. Leveraging the formulation of the RL problem as a stochastic composite optimization problem, \citet{zhang-et-al21} later proposed a (variance-reduced) stochastic PG approach for solving general utility RL ensuring a~$\tilde{\mathcal{O}}(\epsilon^{-3})$ sample complexity to find an~$\epsilon$-stationary policy under smoothness of the utility function and the policy parametrization and a~$\tilde{\mathcal{O}}(\epsilon^{-2})$ global optimality sample complexity for a concave utility with an overparametrized policy. When the utility is concave as a function of the occupancy measure, the corresponding RL problem is known as Convex RL or Convex MDPs.  
Using Fenchel duality, \citet{zahavy-et-al21} casted the convex MDP problem as a min-max game between a policy player and a cost player producing rewards that the policy player must maximize. An insightful consequence of this viewpoint is that any algorithm solving the standard RL problem 
can be used for
solving the more general convex MDP problem. 
In the present paper, we adopt the direct policy search approach with policy parametrization proposed in~\citet{zhang-et-al21} instead of the dual viewpoint.
\citet{geist-et-al22} show that Convex RL is a subclass of Mean-Field games. \citet{zhang_bedi_wang_koppel22} consider a decentralized version of the problem with general utilities with a network of agents. %

\vspace{-2mm}
\section{Preliminaries}
\noindent\textbf{Notations.} For a given finite set~$\mathcal{X}$, we use the notation~$|\mathcal{X}|$ for its cardinality  and~$\Delta(\mathcal{X})$ for the space of probability distributions over~$\mathcal{X}$\,. We equip any Euclidean space with its standard inner product denoted by~$\ps{\cdot , \cdot}$\,. The notation~$\|\cdot\|$ refers to both the standard $2$-norm and the spectral norm for vectors and matrices respectively. 

\noindent\textbf{Markov Decision Process with General Utility.}
Consider a discrete-time discounted Markov Decision Process (MDP) with a general utility function~$\mathbb{M}(\mathcal{S}, \mathcal{A}, \mathcal{P}, F, \rho, \gamma)$, where $\mathcal{S}$~and $\mathcal{A}$~are finite state and action spaces respectively, $\mathcal{P}: \mathcal{S}\times\mathcal{A} \to \Delta(\mathcal{S})$ is the state transition probability kernel, $F: \mathcal{M}(\mathcal{S} \times \mathcal{A}) \to \R$ is a general utility function defined over the space of measures~$\mathcal{M}(\mathcal{S} \times \mathcal{A})$ on the %
product space~$\mathcal{S} \times \mathcal{A}$, 
$\rho$~is the initial state distribution and $\gamma \in (0,1)$ is the discount factor. %
A stationary policy~$\pi: \mathcal{S} \to \Delta(\mathcal{A})$ maps each state~$s \in \mathcal{S}$ to a distribution~$\pi(\cdot|s)$ over the action space $\mathcal{A}$. The set of all stationary policies is denoted by~$\Pi$\,.
At each time step~$t \in \mathbb{N}$ in a state~$s_t \in \mathcal{S}$, the RL agent chooses an action~$a_t \in \mathcal{A}$ with probability~$\pi(a_t|s_t)$ and the environment transitions to a state~$s_{t+1}$ with probability~$\mathcal{P}(s_{t+1}|s_t, a_t)\,.$ 
We denote by~$\bb P_{\rho,\pi}$ the probability distribution of the Markov chain~$(s_t,a_t)_{t \in \mathbb{N}}$ induced by the policy~$\pi$ with initial state distribution~$\rho$. 
We use the notation~$\bb E_{\rho,\pi}$ (or often simply $\bb E$ instead) for the associated expectation. 
We define for any policy~$\pi \in \Pi$ the state-action occupancy measure~$\lambda^{\pi} \in \mathcal{M}(\mathcal{S} \times \mathcal{A})$ as: 
\begin{equation}
\label{eq:s-a-occup-measure}
\lambda^{\pi}(s,a) \eqdef \sum_{t=0}^{+\infty} \gamma^t \bb P_{\rho,\pi}(s_t = s, a_t = a)\,. %
\end{equation}
We denote by~$\Lambda$ the set of such %
occupancy measures, i.e., $\Lambda \eqdef \{ \lambda^{\pi} : \pi \in \Pi\}\,.$ 
Then, the general utility function~$F$ assigns a real to each occupancy measure~$\lambda^{\pi}$ induced by a policy~$\pi \in \Pi$\,. 
A state-action occupancy measure~$\lambda^{\pi}$ will also be seen as a vector of the Euclidean space~$\R^{|\mathcal{S}| \times |\mathcal{A}|}\,.$%

\noindent\textbf{Policy parametrization.} 
In this paper, we will consider 
the common softmax policy parametrization defined for every~$\theta \in \R^d, s \in \mathcal{S}, a \in \mathcal{A}$ by: 
\begin{equation}
\label{eq:softmax-param}
\pi_{\theta}(a|s) = \frac{\exp(\psi(s,a;\theta))}{\sum_{a' \in \mathcal{A}} \exp(\psi(s,a';\theta))}\,, 
\end{equation}
where~$\psi: \mathcal{S} \times \mathcal{A} \times \R^d \to \R$ is a smooth function. %
The softmax parametrization will be important for controlling IS weights for variance reduction.
However, some of our results will not require this specific parameterization and we will explicitly indicate it 
when appropriate. 

\noindent\textbf{Problem formulation.} 
The goal of the RL agent is to find a policy $\pi_\theta$ (determined by the vector~$\theta$) solving the problem: 
\begin{equation}
\label{eq:pb-gen-ut}
\max_{\theta \in \R^d} F(\lambda^{\pi_{\theta}})\,,
\end{equation}
where~$F$ is a smooth function supposed to be upper bounded and~$F^{\star}$ is used in the remainder of this paper to denote the maximum in~\eqref{eq:pb-gen-ut}. 
The agent has only access to (a) trajectories of finite length~$H$ generated from the MDP under the initial distribution~$\rho$ and the policy~$\pi_{\theta}$ and (b) the gradient of the utility function~$F$ with respect to (w.r.t.) its variable~$\lambda$. %
In particular, provided a time horizon~$H$ and a policy~$\pi_{\theta}$ 
with~$\theta \in \R^d$,
the learning agent can simulate a
trajectory~$\tau = (s_0, a_0, \cdots, s_{H-1}, a_{H-1})$ from the MDP whereas the
state transition kernel~$\mathcal{P}$ is unknown. 
This general utility problem was described, for instance, in \citet{zhang-et-al21} (see also~\citet{kumar-et-al22pg-general}). 
Recall that the standard RL problem corresponds to the particular case where the general utility function is a linear function, i.e., $F(\lambda^{\pi_{\theta}}) = \ps{r,\lambda^{\pi_{\theta}}}$ for some vector~$r \in \R^{\mathcal{S} \times \mathcal{A}}$ in which case we recover the expected return function as an objective: 
\begin{equation}
    \label{eq:J}
 V^{\pi_{\theta}}(r) \eqdef \mathbb{E}_{\rho, \pi_{\theta}} \left[\sum_{t=0}^{+\infty} \gamma^t r(s_t,a_t)\right]\,.
\end{equation}
In the standard RL case, we shall use the notation~$J(\theta) \eqdef V^{\pi_{\theta}}(r)\,$ %
where~$r$ is the corresponding reward function.

\noindent\textbf{Policy Gradient for General Utilities.} 
Following the exposition in \cite{zhang-et-al21} (see also more recently~\cite{kumar-et-al22pg-general}), we derive the policy gradient for the general utility objective. For convenience, we use the notation~$\lambda(\theta)$ for~$\lambda^{\pi_{\theta}}$\,.
Since the cumulative reward can be rewritten more compactly~$V^{\pi_{\theta}}(r) = \ps{\lambda^{\pi_{\theta}}, r}$, 
it follows from the policy gradient theorem that: 
\begin{multline}
\label{eq:expected-reinforce}
[\nabla_{\theta} \lambda(\theta)]^{T} r
= \nabla_{\theta} V^{\pi_{\theta}}(r)\\  %
= \mathbb{E}_{\rho, \pi_{\theta}}\left[\sum_{t=0}^{+\infty} \gamma^t r(s_t,a_t) \sum_{t'=0}^{t} \nabla \log \pi_{\theta}(a_{t'}|s_{t'})\right]\,,
\end{multline}
where~$\nabla_{\theta} \lambda(\theta)$ is the Jacobian matrix of the vector mapping~$\lambda(\theta)$\,.
Using the chain rule, we have
\begin{align}
\label{eq:policy-grad-general-utility}
\nabla_{\theta} F(\lambda(\theta)) 
&= [\nabla_{\theta} \lambda(\theta)]^{T} \nabla_{\lambda}F(\lambda(\theta))\nonumber\\
&= \nabla_{\theta} V^{\pi_{\theta}}(r)|_{r = \nabla_{\lambda}F(\lambda(\theta))}\,.
\end{align}

\noindent\textbf{Stochastic Policy Gradient.} 
In light of~\eqref{eq:policy-grad-general-utility}, in order to estimate the policy gradient~$\nabla_{\theta} F(\lambda(\theta))$ for general utilities, we can use the standard reinforce estimator suggested by Eq.~\eqref{eq:expected-reinforce} but we also need to estimate the state-action occupancy measure~$\lambda(\theta)$ (when~$F$ is nonlinear)\footnote{In the cumulative reward setting, notice that the general utility function~$F$ is linear and~$\nabla_{\lambda} F(\lambda(\theta))$ is independent of~$\lambda(\theta)$\,.}. 
Define for every reward function~$r$ (which is also seen as a vector in~$\R^{|\mathcal{S}| \times |\mathcal{A}|}$), every~$\theta \in \R^d$ and every $H$-length trajectory~$\tau$ simulated from the MDP with policy~$\pi_{\theta}$ and initial distribution~$\rho$  the (truncated) policy gradient estimate: 
\begin{equation}
\label{eq:pg-estimate}
g(\tau, \theta,r) = \sum_{t=0}^{H-1} \left( \sum_{h=t}^{H-1} \gamma^h r(s_h,a_h) \right) \nabla \log \pi_\theta(a_t|s_t)\,.
\end{equation}
 We also define an estimator for the state-action occupancy measure~$\lambda^{\pi_{\theta}} = \lambda(\theta)$ (see~\eqref{eq:s-a-occup-measure}) truncated at the horizon~$H$~by: 
 \begin{equation}
    \label{eq:lambda-tau}
        \lambda(\tau) = \sum_{h=0}^{H-1} \gamma^h 
        \delta_{s_h,a_h}\,,
 \end{equation}
 where for every~$(s,a) \in \mathcal{S} \times \mathcal{A}$, $\delta_{s,a} \in \R^{|\mathcal{S}| \times |\mathcal{A}|}$ is a vector of the canonical basis of~$\R^{|\mathcal{S}| \times |\mathcal{A}|}$, i.e.,  the vector whose only non-zero entry is the~$(s,a)$-th entry which is equal to~$1$.%

\textbf{Importance Sampling.} Given a trajectory~$\tau = (s_0, a_0, s_1, a_1, \cdots, s_{H-1}, a_{H-1})$ of length~$H$ generated under the initial distribution~$\rho$ and the policy~$\pi_{\theta}$ for some~$\theta \in \R^d$, we define for every~$\theta' \in \R^d$ the IS weight: 
\begin{equation}
\label{eq:IS-weights}
w(\tau|\theta', \theta) %
\eqdef \prod_{h=0}^{H-1}\frac{\pi_{\theta'}(a_h|s_h)}{\pi_{\theta}(a_h|s_h)}\,.
\end{equation}
Since the problem is nonstationary in the sense that updating the parameter~$\theta$ shifts the distribution over trajectories, it follows that for any~$r \in \R^{|\mathcal{S}| \times |\mathcal{A}|}, \mathbb{E}_{\rho, \pi_{\theta}}[g(\tau,\theta,r) - g(\tau,\theta',r)] \neq \nabla_{\theta} V^{\pi_{\theta}}(r) - \nabla_{\theta} V^{\pi_{\theta'}}(r)\,.$ Using the IS weights, we correct this bias to obtain
\begin{multline*}
\mathbb{E}_{\rho, \pi_{\theta}}[g(\tau,\theta,r) - w(\tau|\theta',\theta)g(\tau,\theta',r)] \\
= \nabla_{\theta} V^{\pi_{\theta}}(r) - \nabla_{\theta} V^{\pi_{\theta'}}(r)\,.
\end{multline*}
The use of IS weights is standard in variance-reduced PG.%

\section{Normalized Variance-Reduced Policy Gradient Algorithm}
\label{sec:norm-vr-pg}

In this section, we present our~\algname{N-VR-PG} algorithm (see Algorithm~\ref{algo-gen-ut}) to solve the RL problem with general utilities. This algorithm has two main distinctive features compared to vanilla PG and existing algorithms~\citep{zhang-et-al21}: 
(i) recursive variance reduction: instead of using the stochastic PG and occupancy measure estimators respectively reported in~\eqref{eq:pg-estimate} and~\eqref{eq:lambda-tau}, we use recursive variance-reduced estimators for both the PG and the state-action occupancy measure akin to STORM in stochastic optimization~\citep{cutkosky-orabona19}. 
This leads to a simple single-loop algorithm using a single trajectory per iteration and for which no checkpoints nor any second order information are needed; 
(ii) normalized PG update rule: normalization will be crucial to control the IS weights used in the estimators. We elaborate more on the motivation for using it in Section~\ref{subsec:normalization-boundedness-is}. 
\begin{algorithm}[h]
   \caption{\algname{N-VR-PG}(General Utilities)}
   \label{algo-gen-ut}
\begin{algorithmic}
   \STATE {\bfseries Input:} $\theta_0$, $T$, $H$, $\{\eta_t\}_{t\geq 0}$, $\{\al_t\}_{t\geq 0}\,.$
   \STATE Sample $\tau_0$ of length~$H$ from~$\mathbb{M}$ and~$\pi_{\theta_0}$
   \STATE $\lambda_0 = \lambda(\tau_0,\theta_0); r_0 = \nabla_{\lambda} F(\lambda_0); r_{-1} = r_0$ 
   \STATE $d_0 = g(\tau_0, \theta_0, r_0)$
   \STATE $\theta_1 = \theta_0 + \al_0 \frac{d_0}{\|d_0\|}$
   \FOR{$t=1, \ldots, T -  1$}
        \STATE Sample $\tau_t$ of length~$H$ from MDP~$\mathbb{M}$ and~$\pi_{\theta_t}$
        \STATE $u_t = \lambda(\tau_t) (1 - w(\tau_t|\theta_{t-1},\theta_t))$
        \STATE $\lambda_t = \eta_t \lambda(\tau_t) + (1-\eta_t) (\lambda_{t-1} + u_t)$
        \STATE $r_t = \nabla_{\lambda} F(\lambda_t)$
        \STATE $v_t = g(\tau_t, \theta_t, r_{t-1}) - w(\tau_t|\theta_{t-1},\theta_t)g(\tau_t, \theta_{t-1},r_{t-2})$
        \STATE $d_t = \eta_t g(\tau_t, \theta_t,r_{t-1}) + (1 - \eta_t) (d_{t-1} +  v_t)$
        \STATE $\theta_{t+1} = \theta_t + \al_t \frac{d_t}{\|d_t\|}$
    \ENDFOR
\end{algorithmic}
\end{algorithm}
\begin{remark}
In Algorithm~\ref{algo-gen-ut}, note that~$g(\tau_t, \theta_t, r_{t-1})$ and~$g(\tau_t, \theta_{t-1},r_{t-2})$ are used in~$v_t$ instead of~$g(\tau_t, \theta_t, r_{t})$ and~$g(\tau_t, \theta_{t-1},r_{t-1})$ respectively to address measurability and independence issues in the analysis. 
\end{remark}

\begin{remark}[Standard RL]
In the cumulative reward setting, estimating the occupancy measure is not needed.
Hence, Algorithm~\ref{algo-gen-ut} simplifies (see Algorithm~\ref{alg:N-VR-PG} in Appendix~\ref{app:nvrpg-standard-RL}). 
\end{remark}

\section{Convergence Analysis of \algname{N-VR-PG}}

We first introduce our assumptions regarding the regularity of the policy parametrization and the utility function~$F$.
\begin{assumption}
\label{hyp:policy-param}
In the softmax parametrization~\eqref{eq:softmax-param}, the map~$\psi(s,a;\cdot)$ is twice continuously differentiable and there exist~$l_{\psi}, L_{\psi} >  0$ s.t. (i) $\max_{s \in \mathcal{S}, a \in \mathcal{A}} \sup_{\theta} \|\nabla \psi(s,a;\theta)\| \leq l_{\psi}$ and~(ii) $\max_{s \in \mathcal{S}, a \in \mathcal{A}} \sup_{\theta} \|\nabla^2 \psi(s,a;\theta)\| \leq L_{\psi}\,.$
\end{assumption}

\begin{assumption}
\label{hyp:smoothness-F}
There exist constants~$l_{\lambda}, L_{\lambda}, L_{\lambda,\infty} > 0$ s.t. for all~$\lambda, \lambda' \in \Lambda$, $\|\nabla_{\lambda} F(\lambda)\|_{\infty} \leq l_{\lambda}$ and
\begin{align}
\|\nabla_{\lambda} F(\lambda) - \nabla_{\lambda} F(\lambda')\|_{\infty} & \leq L_{\lambda} \|\lambda - \lambda'\|_2\,,\nonumber\\
\|\nabla_{\lambda} F(\lambda) - \nabla_{\lambda} F(\lambda')\|_{\infty} & \leq L_{\lambda,\infty} \|\lambda - \lambda'\|_1\,.\nonumber
\end{align}
\end{assumption}

Assumptions~\ref{hyp:policy-param} and~\ref{hyp:smoothness-F} were previously considered in~\citet{zhang-et-al21,zhang-et-al20variational} and guarantee together that the objective function~$\theta \mapsto F(\lambda^{\pi_{\theta}})$ is smooth. %
Assumption~\ref{hyp:smoothness-F} is automatically satisfied for the cumulative reward setting (i.e., $F$ linear) if the reward function is bounded. 

\subsection{Normalization ensures boundedness of IS weights}
\label{subsec:normalization-boundedness-is}

Most prior works 
suppose that the variance of the IS weights is bounded. Such assumption cannot be verified. 
In this section we provide an alternative algorithmic way based on the softmax policy to control the IS weights without %
the aforementioned assumption. 
Since our algorithm only uses IS weights for two consecutive iterates, our key observation is that a normalized gradient update rule automatically guarantees bounded IS weights. In particular, compared to~\citet{zhang-et-al21}, we do not  use a gradient truncation mechanism which requires an additional truncation hyperparameter depending on the problem parameters and dictates a non-standard stationarity measure (see Remark~\ref{rem:delta-dependence}). This simple algorithmic modification requires several adjustments in the convergence analysis (see Appendix~\ref{sec:app-proofs-fos} and~\ref{sec:app-proofs-globopt}). 
We formalize the result in the following lemma.  
\begin{lemma}
\label{lem:variance-IS-weights-control}
Let Assumption~\ref{hyp:policy-param} hold true. Suppose that the sequence~$(\theta_t)$ is updated via~$\theta_{t+1} = \theta_t + \al_t \frac{d_t}{\|d_t\|}$ where~$d_t \in \R^d$ is any non-zero update direction and~$\al_t$ is a positive stepsize. Then, for every integer~$t$ and any trajectory~$\tau$ of length~$H$, we have~$w(\tau|\theta_{t},\theta_{t+1}) \leq \exp\{ 2 H l_{\psi} \al_t\}\,.$
If, in addition,~$H = \mathcal{O}(\frac{\log T}{1-\gamma})$ and~$\al_t = \al = T^{-\fr{2}{3}}$, then there exists a constant~$W >0$ s.t.~$ w(\tau|\theta_{t}, \theta_{t+1}) \leq W\,.$ 
Moreover, we have~$\Var{[w(\tau_{t+1}|\theta_t,\theta_{t+1})]} \leq C_w \al^2$ where~$\tau_{t+1}$ is a trajectory of length~$H$ sampled from~$\pi_{\theta_{t+1}}$ and~$C_w \eqdef H ((8H+2) l_{\psi}^2 + 2 L_{\psi}) (W+1)\,.$ 
\end{lemma}

In this lemma,  the variance of the IS weights decreases over time at a rate controlled by~$\alpha^2$ and this result will be crucial for our convergence analysis of~\algname{N-VR-PG}. We show in Lemma~\ref{lem:bounded-var-is-weights-gaussian} in the Appendix that such a result also holds for Gaussian policies for continuous state action spaces.

\subsection{First-order stationarity}
\label{subsec:fos}

In this section, we show that~\algname{N-VR-PG} requires~$\tilde{\mathcal{O}}(\varepsilon^{-3})$ samples to reach an~$\varepsilon$-first-order stationary (FOS) point of the objective function for RL with general utilities.\footnote{All the proofs of our results are provided in the Appendix.}

\begin{theorem}%
    \label{thm:fos-gen-ut-nvrpg}
    Let Assumptions~\ref{hyp:policy-param} and~\ref{hyp:smoothness-F} hold. 
    Let~$\alpha_0 > 0$ and let~$T \geq 1$ be an integer. 
    Set~$\al_t = \fr{\al_0}{T^{\nfr{2}{3}}}, \eta_t = \left( \fr{2}{t+1} \right)^{\nfr{2}{3}}$ and~%
    $H = \rb{1-\g}^{-1}{\log(T + 1)}$. 
    Then, $
    \Exp{ \norm{ \nabla_{\theta}F(\lambda(\bar{\theta}_T)) } } \leq \cO\rb{ \fr{1 + (1-\gamma)^3 \Delta \al_0^{-1} +  (1-\gamma)^{-1} \al_0 }{ (1-\gamma)^{3}  T^{\nfr{1}{3}} } },\,
    $ where~$\Delta \eqdef F^{\star} - \mathbb{E}[F(\lambda(\theta_1))]$ and~$\bar{\theta}_T$ is sampled uniformly at random from %
    ~$\{\theta_1, \cdots, \theta_T\}$ of Algorithm~\ref{algo-gen-ut}.
\end{theorem}

\begin{remark}
In terms of dependence on~$(1-\gamma)^{-1}$, we significantly improve over the result of~\citet{zhang-et-al21} which does not make it explicit. We defer a detailed comparison regarding this dependence to Appendix~\ref{app:dependence_(1-gamma)-1}. 
\end{remark}

\begin{remark}
\label{rem:delta-dependence}
Unlike~\citet{zhang-et-al21} which utilizes a gradient truncation radius, our sample complexity does not depend on the inverse of this gradient truncation hyperparameter which might be small. Indeed, to translate their guarantee from the non-standard gradient mapping dictated by gradient truncation to the standard stationarity measure (used in our result), one has to incur an additional multiplicative constant~$\delta^{-1}$ where~$\delta$ is the gradient truncation radius (see Lemma~5.4 in \cite{zhang-et-al21}). 
\end{remark} 

Recalling the notation~$J(\theta) = V^{\pi_{\theta}}(r)$ (see~\eqref{eq:J}) for the standard RL setting, we can state the following corollary. 
\begin{corollary}%
\label{cor:fos-standard-RL-softmax}
Under the setting of Theorem~\ref{thm:fos-gen-ut-nvrpg}, 
    if we set~$\al_0 = 1 - \gamma$, then
    $
    \Exp{ \norm{\nabla J(\bar{\theta}_T)} } \leq \cO\rb{ (1-\gamma)^{-2} T^{-\nfr{1}{3}}}\, .
    $
\end{corollary}

The next result addresses the case of continuous state-action spaces in the standard RL setting using a Gaussian policy. Notably, we rely on similar considerations as for the softmax policy to control the variance of IS weights. 
We defer a precise statement of this result to Appendix~\ref{subsec:fos-gaussian}. 

\begin{theorem}[informal]%
\label{thm:fos-standard-RL-gaussian}
Using the Gaussian policy under some regularity conditions, \algname{N-VR-PG} (see Algorithm~\ref{alg:N-VR-PG}) requires~$\tilde{\mathcal{O}}(\varepsilon^{-3})$ to reach an~$\varepsilon$-first-order stationary point of the expected return~$J$. 
\end{theorem}

\subsection{Global optimality}
\label{subsec:glob-opt}

In this section, we show that~\algname{N-VR-PG} only requires~$\tilde{\mathcal{O}}(\varepsilon^{-2})$ samples to reach an~$\varepsilon$-globally optimal policy under a concave reparametrization of the RL problem with concave utilities and an additional overparametrization assumption. Our results and assumptions match the recent results in~\citet{zhang-et-al21} for finite state-action spaces.   

\begin{assumption}
\label{hyp:F-concave}
The utility function~$F$ is concave. 
\end{assumption}

\begin{assumption}
  \label{hyp:overparam-global-opt}
  For the softmax policy parametrization in~\eqref{eq:softmax-param}, the following three requirements hold: (i) For any~$\theta \in \R^d,$ there exist relative neighborhoods~$\mathcal{U}_{\theta} \subset \R^d$ and~$\mathcal{V}_{\lambda(\theta)} \subset \Lambda$ respectively containing~$\theta$ and~$\lambda(\theta)$ s.t. the restriction~$\lambda|_{\mathcal{U}_{\theta}}$ forms a bijection between~$\mathcal{U}_{\theta}$ and~$\mathcal{V}_{\lambda(\theta)}$\,; 
  (ii) There exists~$l>0$ s.t. for every~$\theta \in \R^d$, the inverse~$(\lambda|_{\mathcal{U}_{\theta}})^{-1}$ is~$l$-Lipschitz continuous; 
  (iii) There exists~$\bar{\epsilon} > 0$ s.t. for every positive real~$\epsilon \leq \bar{\epsilon}$, $(1-\epsilon) \lambda(\theta) + \epsilon \lambda(\theta^{*}) \in \mathcal{V}_{\lambda(\theta)}$ where~$\pi_{\theta^{*}}$ is the optimal policy. 

\end{assumption}

For the tabular softmax parametrization~(i.e., $\psi(s,a;\theta) = \theta_{s,a}, d = |\mathcal{S}| |\mathcal{A}|$), a continuous local inverse can be defined whereas computing the Lipschitz constant~$l$ is more involved as reported in~\citet{zhang-et-al21} (see Appendix~\ref{app:overparam-assumption} for a discussion of Assumption~\ref{hyp:overparam-global-opt}). Relaxing this strong assumption is left for future work. 

\begin{remark}
Compared to Assumption~5.11 in \citet{zhang-et-al21}, Assumption~\ref{hyp:overparam-global-opt} is quasi-identical with the slight difference that it does not depend on the gradient truncation hyperparameter~$\delta$ used in~\citet{zhang-et-al21}. 
\end{remark}

Our global optimality convergence result is as follows. 
\begin{theorem}%
\label{thm:glob-opt-gen-ut-tabular}
    Let Assumptions~\ref{hyp:policy-param}, \ref{hyp:smoothness-F} and~\ref{hyp:F-concave} hold. Additionally, let Assumption~\ref{hyp:overparam-global-opt} be satisfied with $\bar{\epsilon} \geq \fr{ \al_0 (1-\g)}{2 \ell_{\theta} (T+1)^a}$ for some integer $T \geq 1$ and reals $\al_0>0$, $a\in(0,1)$.
    Set~$\al_t = \fr{\al_0}{(T+1)^a}$, $\eta_t = \fr{2}{t+1} $ and~$H = \rb{1-\g}^{-1}{\log(T + 1)}$. Then the output~$\theta_T$ of \algname{N-VR-PG} (see Algorithm~\ref{algo-gen-ut}) satisfies
    \begin{equation*}
        F^{\star} - \Exp{ F(\lambda(\theta_T))  } \leq \cO%
        \left( \fr{ \al_0^2  }{(1-\gamma)^3 (T+1)^{ 2 a -\fr{3}{2}}} \right),
    \end{equation*}
    Thus, setting $\al_0 = (1-\gamma)^{\nfr{3}{2}}$, the sample complexity to achieve $F^* - \Exp{ F(\lambda(\theta_T))  } \leq \varepsilon$ is $\cO\rb{ \varepsilon^{\fr{-2}{4a - 3}}}$.
\end{theorem}

\begin{corollary}
\label{cor:glob-opt-standard-rl-tabular}
In the setting of Theorem~\ref{thm:glob-opt-gen-ut-tabular}, \algname{N-VR-PG} (see Algorithm~\ref{alg:N-VR-PG}) requires~$\tilde{\cO}\rb{ \varepsilon^{\fr{-2}{2a - 1}}}$ samples to achieve $J^{\star} - \Exp{ J(\theta_T) } \leq \varepsilon$ where~$J^{\star}$ is the optimal expected return. 
\end{corollary}

\begin{remark}
We refer the reader to Appendix~\ref{subsec:app-globopt-standard-rl} for a precise statement of Corollary~\ref{cor:glob-opt-standard-rl-tabular}. 
If we know problem parameters and choose time varying step-sizes $\al_t = \fr{\al_0}{t}$, then we can obtain exactly~$\tilde{\mathcal{O}}(\varepsilon^{-2})$ sample complexity. 
\end{remark}

We can state a similar global optimality result to Corollary~\ref{cor:glob-opt-standard-rl-tabular} 
for continuous state-action spaces (see Appendix~\ref{subsec:app-globopt-gaussian}). 

\section{Large State-Action Space Setting}
\label{sec:large-state-action-space}

An important limitation of Algorithm~\ref{algo-gen-ut} and the prior work~\citep{zhang-et-al21} 
is the need to estimate the occupancy measure for each state-action pair in the case of general nonlinear utilities. 
This procedure is intractable if the state and/or action spaces are prohibitively large and finite or even worse infinite/continuous. 
In the case of infinite or continuous state-action spaces, the occupancy measure~$\lambda^{\pi_{\theta}}$ induced by a policy~$\pi_{\theta}$ cannot be represented by a vector in finite dimensions. Thus, the derivative of the utility function~$F$ w.r.t. its variable~$\lambda$ is not well defined in the chain rule in~\eqref{eq:policy-grad-general-utility} for the policy gradient. Therefore, %
more adequate  notions of derivative for optimization on the space of measures are probably needed and this would require different methodological and algorithmic tools which go beyond the scope of this work. In this paper, we propose to do a first step 
by considering the setting of large \textit{finite} state and action spaces which is already of practical interest.  

\subsection{PG for RL with General Utilities via linear function approximation of the occupancy measure}

Similarly to the classical linear function approximation of the (action-)value function in standard RL, we propose to approximate the (truncated) state-action occupancy measure by a linear combination of pre-selected basis functions in order to break the so-called curse of dimensionality. 
Our exposition is similar in spirit to the compatible function approximation framework~\citep{PGM_Sutton_1999} which was recently extended in~\citet{agarwal-et-al21} (see also~\citet{yuan-et-al22log-linear} for a recent example). However, we are not concerned here by the approximation of the action-value function nor are we considering the NPG (or Q-NPG) method but we are rather interested in approximating the discounted occupancy measure. Recall that we are considering the more general problem of RL with general utilities. 
Beyond this connection with existing work, we shall precise that our approach mostly shares the use of standard least squares regression for estimating an unknown function which is the state-action occupancy measure in our case.

Let~$m$ be a positive integer and let~$\phi: \mathcal{S} \times \mathcal{A} \to \R^m$ be a feature map. We shall approximate the truncated\footnote{We could use the non-truncated occupancy measure (see Appendix~\ref{appendix:large-sa-setting}). For simplicity of exposition, we use the truncated version, the difference between both quantities is of the order of~$\gamma^H$.} state-action occupancy measure for a given policy~$\pi_{\theta}$ ($\theta \in \R^d$ fixed) by a linear combination of feature vectors from the feature map, i.e., for every state-action pair~$(s,a) \in \mathcal{S} \times \mathcal{A}$, 
\begin{equation}
\lambda_H^{\pi_{\theta}}(s,a) \approx \ps{\phi(s,a), \omega_{\theta}}\,,
\end{equation}
for some~$\omega_{\theta} \in \R^m$ that we shall compute. 
Typically, the dimension~$m$ is much smaller than~$|\mathcal{S}| \times |\mathcal{A}|\,.$ The feature map summarizes the most important characteristics of state-action pairs. Typically, this map is designed based on experience and domain-specific knowledge or intuition regarding the MDP. Standard examples of basis functions for the feature map include radial basis functions, wavelet networks or polynomials. Nevertheless, designing such a feature map is an important practical question that is often problem-specific and we will not address it in this work.

In order to compute such a vector~$\omega_{\theta}$, we will use linear regression. Accordingly, we 
define the expected regression loss measuring the estimation quality of any parameter~$\omega$ for every~$\theta \in \R^d, \omega \in \R^{m}$ by: 
\begin{equation}
\label{eq:reg-loss1}
L_{\theta}(\omega) \eqdef \mathbb{E}_{s \sim \rho, a \sim \mathcal{U}(\mathcal{A})}[(\lambda_H^{\pi_{\theta}}(s,a) - \ps{\phi(s,a), \omega})^2]\,,
\end{equation}
where~$\rho$ is the initial distribution in the MDP and~$\mathcal{U}(\mathcal{A})$ is the uniform distribution over the action space~$\mathcal{A}\,.$\footnote{Other exploratory sampling distributions for~$s$ and~$a$ can be considered, we choose~$\rho$ and~$\mathcal{U}(\mathcal{A})$ for simplicity.}
In practice, we cannot minimize~$L_{\theta}$ exactly since this would require having access to the true state-action occupancy measure and averaging over all state-action pairs~$s \sim \rho, a \sim \mathcal{U}(\mathcal{A})\,.$ 
Therefore, we compute an approximate solution~$\hat{\omega}_{\theta} \approx \argmin_{\omega} L_{\theta}(\omega)\,.$ For this procedure, we need: (i) unbiased estimates of the true truncated state-action occupancy measure~$\lambda_H^{\pi_{\theta}}(s,a)$ (or the non-truncated one~$\lambda^{\pi_{\theta}}(s,a)$) for~$s \sim \rho, a \sim \mathcal{U}(\mathcal{A})$ and (ii) a regression solver based on samples to minimize~$L_{\theta}$ as defined in~\eqref{eq:reg-loss1}.
As for item~(i), we use a Monte-Carlo estimate~$\hat{\lambda}_H^{\pi_{\theta}}(s,a)$ of the truncated occupancy measure computed from a single rollout (see Algorithm~\ref{algo:MC-estimate-lambda-pi-theta} for details).\footnote{We can also compute an unbiased estimator of the true occupancy measure~$\lambda^{\pi_{\theta}}(s,a)$ %
via a standard procedure with a random horizon~$H$ following a geometric distribution %
(see Algorithm~\ref{algo:geom-rollout-estimate-true-occup-measure}).} 
An unbiased stochastic gradient of the function~$L_{\theta}$ in~\eqref{eq:reg-loss1} is then given by 
\begin{equation} 
\label{eq:stoch-grad-lin-reg}
\hat{\nabla}_{\omega} L_{\theta}(\omega) \eqdef 2 (\ps{\phi(s,a),\omega} - \hat{\lambda}_H^{\pi_{\theta}}(s,a))\, \phi(s,a)\,.
\end{equation}
We can then solve the regression problem consisting in minimizing~$L_{\theta}$ in~\eqref{eq:reg-loss1} via the averaged SGD algorithm (see Algorithm~\ref{algo:sgd-subroutine2}) as proposed in~\citet{bach-moulines13}. %
\begin{algorithm}[h]
   \caption{(averaged) SGD for Occupancy Measure Estimation via Linear Function Approximation}
   \label{algo:sgd-subroutine2}
\begin{algorithmic}
   \STATE {\bfseries Input:} $\omega_0 \in \R^{m}, K \geq 1, \beta > 0, \rho, \pi_{\theta}\,.$ 
   \FOR{$k = 0, \ldots, K-1$}
        \STATE Sample~$s \sim \rho; a \sim \mathcal{U}(\mathcal{A})$ 
        \STATE Compute an estimator~$\hat{\lambda}_H^{\pi_{\theta}}(s,a)$ via Algorithm~\ref{algo:MC-estimate-lambda-pi-theta}
        \STATE $\hat{\nabla}_{\omega} L_{\theta}(\omega_k) \eqdef 2 (\ps{\phi(s,a),\omega_k} - \hat{\lambda}_H^{\pi_{\theta}}(s,a))\, \phi(s,a)$
        \STATE $\omega_{k+1} = \omega_k - \beta\, \hat{\nabla}_{\omega} L_{\theta}(\omega_k)$ 
    \ENDFOR
	\STATE {\bfseries Return:} $\hat{\omega}_{\theta} = \frac{1}{K} \sum_{k=1}^K \omega_k$
\end{algorithmic}
\end{algorithm}

Using this procedure, we propose a simple stochastic PG algorithm for solving the RL problem with general utilities for large state action spaces. 
Since this large-scale setting has not been priorly addressed for general utilities to the best of our knowledge, we focus on a simpler PG algorithm without the variance reduction and normalization features of our algorithm in Section~\ref{sec:norm-vr-pg}. Incorporating variance reduction to occupancy measure estimates seems more involved with our linear regression procedure for function approximation. 
We leave it for future work to design a method with improved sample complexity using variance reduction. 

\begin{algorithm}[h]
   \caption{Stochastic PG for RL with General Utilities via Linear Function Approximation of the Occupancy Measure}
   \label{algo-gen-ut-func-approx2}
\begin{algorithmic}
   \STATE {\bfseries Input:} $\theta_0 \in \R^d, T, N \geq 1, \alpha > 0, K \geq 1, \beta > 0, H\,.$ 
    \STATE Run Algorithm~\ref{algo:sgd-subroutine2} with policy~$\pi_{\theta_0}$ and define from its output~$\hat{\lambda}_0(\cdot, \cdot) = \ps{\phi(\cdot, \cdot),\hat{\omega}_{\theta_0}}\,.$
    \STATE $r_{-1} = \nabla_{\lambda} F(\hat{\lambda}_0)$
    \FOR{$t=0, \ldots, T -  1$}
        \STATE Run Algorithm~\ref{algo:sgd-subroutine2} with policy~$\pi_{\theta_t}$ and define from its output~$\hat{\lambda}_t(\cdot, \cdot) = \ps{\phi(\cdot, \cdot),\hat{\omega}_{\theta_t}}\,.$
        \STATE $r_t = \nabla_{\lambda} F(\hat{\lambda}_t)$
        \STATE Sample a batch of~$N$ independent trajectories~$(\tau_t^{(i)})_{1 \leq i \leq N}$ of length~$H$ from~$\mathbb{M}$ and~$\pi_{\theta_t}$
        \STATE $\theta_{t+1} = \theta_t + \frac{\alpha}{N} \sum_{i = 1}^N g(\tau_t^{(i)}, \theta_t,r_{t-1})$
    \ENDFOR
	\STATE {\bfseries Return:} $\theta_T$
\end{algorithmic}
\end{algorithm}
\begin{remark}
When running Algorithm~\ref{algo-gen-ut-func-approx2}, notice that the vector~$\hat{\lambda}_t \in \R^{|\mathcal{S}| \times |\mathcal{A}|}$ (and hence the vector~$r_t$) does not need to be computed for all state-action pairs as this would be unrealistic and even impossible in the large state-action setting we are considering. Indeed, at each iteration, one does only need to compute~$(r_t(s_h^{(t)},a_h^{(t)}))_{0 \leq h \leq H-1}$ where~$\tau_t = (s_h^{(t)}, a_h^{(t)})_{0 \leq h \leq H-1}$ to obtain the stochastic policy gradient~$g(\tau_t, \theta_t, r_{t-1})$ as defined in~\eqref{eq:pg-estimate}.
\end{remark}

\subsection{Convergence and sample complexity analysis}

In this section, we provide a convergence analysis of Algorithm~\ref{algo-gen-ut-func-approx2}. 
For every integer~$t$, let~$\omega_*(\theta_t) \in \argmin_{\omega} L_{\theta_t}(\omega)$.
We decompose the regression loss into the statistical error measuring the accuracy of our approximate solution %
and the approximation error measuring the distance between the true occupancy measure and its best linear approximation using the feature map~$\phi$:
\begin{equation*}
L_{\theta_t}(\hat{\omega}_t) =  \underbrace{L_{\theta_t}(\hat{\omega}_t) - L_{\theta_t}(\omega_*(\theta_t))}_{\text{statistical error}} + \underbrace{L_{\theta_t}(\omega_*(\theta_t))}_{\text{approximation error}}\,,
\end{equation*}
where we use the shorthand notation~$\hat{\omega}_t = \hat{\omega}_{\theta_t}$ and~$\hat{\omega}_{\theta_t}$ is the output of Algorithm~\ref{algo:sgd-subroutine2} after~$K$ iterations. We assume that both the statistical and approximation errors are uniformly bounded along the iterates of our algorithm. Such assumptions have been considered for instance in a different context in the compatible function approximation framework (see Assumptions~6.1.1 and Corollary~21 in \citet{agarwal-et-al21}, also Assumptions~1 and~5 in \citet{yuan-et-al22log-linear}).  

\begin{assumption}[Bounded statistical error]
\label{hyp:bounded-stat-error}
There exists~$\epsilon_{\text{stat}} > 0$ s.t. for all iterations~$t \geq 0$ of Algorithm~\ref{algo-gen-ut-func-approx2}, we have
$\mathbb{E}[L_{\theta_t}(\hat{\omega}_{\theta_t}) - L_{\theta_t}(\omega_*(\theta_t)) ] \leq \epsilon_{\text{stat}}\,.$
\end{assumption}

We will see in the next section that we can guarantee~$\epsilon_{\text{stat}} = \mathcal{O}(1/K)$ where~$K$ is the number of iterations of SGD (Algorithm~\ref{algo:sgd-subroutine2}) to find the approximate solution~$\hat{\omega}_t$ at each iteration~$t$ of Algorithm~\ref{algo-gen-ut-func-approx2}.   

\begin{assumption}[Bounded approximation error]
\label{hyp:bounded-approx-error}
There exists~$\epsilon_{\text{approx}} > 0$ s.t. for all iterations~$t \geq 0$ of Algorithm~\ref{algo-gen-ut-func-approx2}, we have $\mathbb{E}[ L_{\theta_t}(\omega_*(\theta_t)) ] \leq \epsilon_{\text{approx}}\,.$
\end{assumption}
This error is due to function approximation and depends on the expressiveness of the approximating function class. 
The true state-action occupancy measure to be estimated may not lie in the function approximation class under consideration.

\begin{theorem}
\label{thm:fosp-gen-ut-lin-fa2-eps-stat-approx}
Let Assumptions~\ref{hyp:policy-param}, \ref{hyp:smoothness-F}, \ref{hyp:bounded-stat-error} and~\ref{hyp:bounded-approx-error} hold true. In addition, suppose that there exists~$\rho_{\min} > 0$ s.t. 
~$\rho(s) \geq \rho_{\min}$ for all~$s \in \mathcal{S}\,.$ Let~$T \geq 1$ be an integer and let~$(\theta_t)$ be the sequence generated by Algorithm~\ref{algo-gen-ut-func-approx2} with a positive step size~$\alpha = \mathcal{O}(1)$ and batch size~$N \geq 1$.
Then, %
\begin{multline}
\mathbb{E}[\|\nabla_{\theta} F(\lambda(\bar{\theta}_T))\|^2] \leq 
\mathcal{O}\left(\frac{1}{T}\right) + \mathcal{O}\left(\frac{1}{N}\right) 
+ \mathcal{O}(\gamma^{2H})\\
+ \mathcal{O}(\epsilon_{\text{stat}} + \epsilon_{\text{approx}})\,,
\end{multline}
where~$\bar{\theta}_T \in \{\theta_1, \cdots, \theta_T\}$ uniformly at random.
\end{theorem}

A few comments are in order regarding Theorem~\ref{thm:fosp-gen-ut-lin-fa2-eps-stat-approx} :
(1)~The specific structure of the softmax parametrization is not needed for Theorem~\ref{thm:fosp-gen-ut-lin-fa2-eps-stat-approx}. Indeed, this softmax parametrization is only useful to control IS weights used %
for variance reduction in Algorithm~\ref{algo-gen-ut}. Assumption~\ref{hyp:policy-param} can be replaced by any smooth policy parametrization satisfying the same standard conditions with~$\nabla \log \pi_{\theta}$ instead of~$\psi$\,; 
(2)~If the true (truncated) occupancy measure does not lie in the class of linear functions described, a positive function approximation error~$\epsilon_{\text{approx}}$ is incurred due to the bias induced by the limited expressiveness of the linear function approximation. A possible natural alternative is to consider richer classes such as neural networks to approximate the state-action occupancy measure and reduce the approximation bias. In this more involved case, the expected least squares (or other metrics) regression loss would likely become nonconvex and introduce further complications in our analysis. Such an extension would require other technical tools that are beyond the scope of the present paper and we leave it for future work.

In order to establish the total sample complexity of our algorithm, we need to compute the number of samples needed in the occupancy measure estimation subroutine of Algorithm~\ref{algo:sgd-subroutine2}. To do so, we now specify the number of SGD iterations required in Algorithm~\ref{algo:sgd-subroutine2} to approximately solve our regression problem. In particular, we will show that we can achieve~$\epsilon_{\text{stat}} = \mathcal{O}(1/K)$ where~$K$ is the number of iterations of the SGD subroutine using Theorem~1 in \citet{bach-moulines13}. Before stating our result, we make an additional standard assumption on the feature map~$\phi\,.$
\begin{assumption}
\label{hyp:feature-map}
The feature map~$\phi: \mathcal{S} \times \mathcal{A} \to \R^m$ satisfies: 
(i) There exists~$B > 0$ s.t. for all~$s \in \mathcal{S}, a \in \mathcal{A}$, $\|\phi(s,a)\| \leq B$\, and (ii) There exists~$\mu > 0$ s.t. $\mathbb{E}_{s \sim \rho, a \sim \mathcal{U}(\mathcal{A})}[\phi(s,a) \phi(s,a)^T] \succcurlyeq \mu I_m$ where~$I_m \in \R^{m \times m}$ is the identity matrix.
\end{assumption}
Assumption~\ref{hyp:feature-map} guarantees that the covariance matrix of the feature map is invertible. 
Similar standard assumptions have been commonly considered for linear function approximation settings \citep{tsitsiklis-vanroy97}. 

We are now ready to state a corollary of Theorem~\ref{thm:fosp-gen-ut-lin-fa2-eps-stat-approx} establishing the total sample complexity of Algorithm~\ref{algo-gen-ut-func-approx2} to achieve an $\epsilon$-stationary point of the objective function. 

\begin{corollary}%
\label{cor:fosp-gen-ut-lin-fa2}
Let Assumptions~\ref{hyp:policy-param}, \ref{hyp:smoothness-F}, \ref{hyp:bounded-approx-error} and~\ref{hyp:feature-map} hold in the setting of Theorem~\ref{thm:fosp-gen-ut-lin-fa2-eps-stat-approx} where we run the SGD subroutine of Algorithm~\ref{algo:sgd-subroutine2} with step size~$\beta = 1/8 B^2$ and~$\omega_0 = 0$ for~$K$ iterations at each timestep~$t$ of Algorithm~\ref{algo-gen-ut-func-approx2}. Then, for every~$\epsilon > 0$, setting~$T = \mathcal{O}(\epsilon^{-2})$, $N = \mathcal{O}(\epsilon^{-2})$, $K = \mathcal{O}(\epsilon^{-2})$ and~$H = \mathcal{O}(\log(\frac{1}{\epsilon}))$ guarantees that~$\mathbb{E}[\|\nabla_{\theta} F(\lambda(\bar{\theta}_T))\|] \leq \mathcal{O}(\epsilon) + \mathcal{O}(\sqrt{\epsilon_{\text{approx}}})$ where~$\bar{\theta}_T \in \{\theta_1, \cdots, \theta_T\}$ uniformly at random. 
The total sample complexity to reach an~$\epsilon$-stationary point %
(up to the~$\mathcal{O}(\sqrt{\epsilon_{\text{approx}}})$ error floor)
is given by~$T \times (K + M) \times H = \tilde{\mathcal{O}}(\epsilon^{-4})\,.$
\end{corollary}

In terms of the target accuracy~$\epsilon$, this result matches the optimal sample complexity to obtain an~$\epsilon$-FOSP for nonconvex smooth stochastic optimization via SGD (without variance reduction) up to a log factor.

\section{Numerical Simulations} 

\begin{figure*}[ht]
       \begin{minipage}{.5\textwidth}
        \centering
        \includegraphics[width=.9\textwidth]{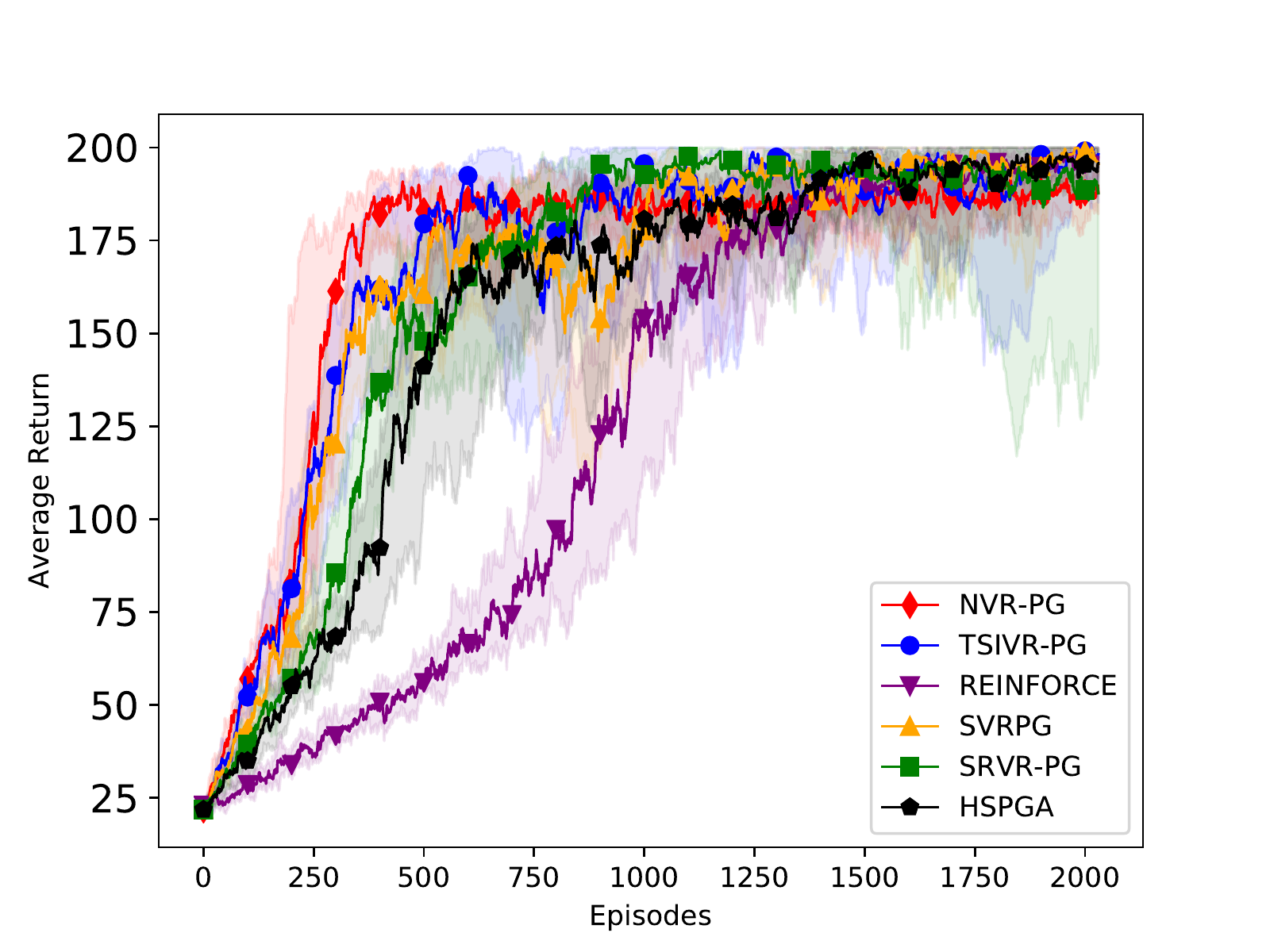}
        \end{minipage}
        \begin{minipage}{.5\textwidth}
        \includegraphics[width=.9\textwidth]{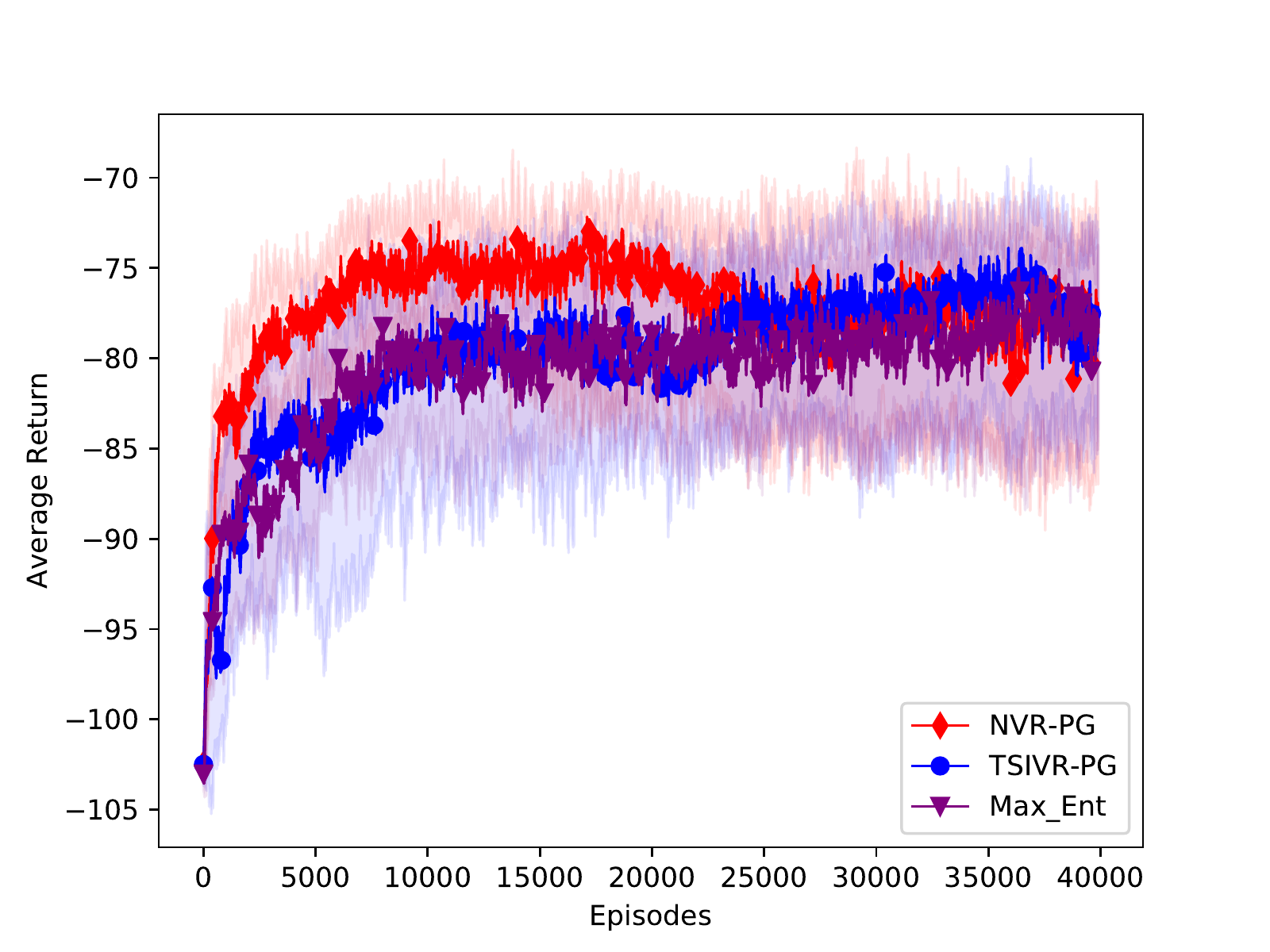}
       \end{minipage}
        \caption{(right) Nonlinear objective maximization in the FrozenLake environment and (left) Standard RL in the CartPole environment. In both cases, the performance curves represent the median return over 20 runs of the algorithms (with 20 seeds) and the shaded colored areas are computed with the 1/4 and 3/4 quantiles of the outcomes.}
        \label{fig:plots}
\end{figure*}

In this section, we present two simple numerical experiments to illustrate the performance of our algorithm compared to prior work and complement our theoretical contributions. Our implementation is based on the code provided in \citet{zhang-et-al21}.\footnote{Available in OpenReview (\url{https://openreview.net/forum?id=Re_VXFOyyO}).} Our goal is to show that our algorithm can be competitive compared to existing algorithms while gaining simplicity. We leave further experimental investigations in larger scale problems for future work.   

\noindent\textbf{(a) Nonlinear objective function maximization.} We consider a general utility RL problem where the objective function~$F: \R_{+}^{|\mathcal{S}| \times |\mathcal{A}|} \to \R$ is a nonlinear function of the occupancy measure defined for every~$\lambda \in \R_{+}^{|\mathcal{S}| \times |\mathcal{A}|}$ by: 
\begin{equation*}
F(\lambda) \eqdef \sum_{s \in \mathcal{S}} \log\left( \sum_{a \in \mathcal{A}} \lambda_{s,a} + \sigma \right)\,, 
\end{equation*}
where~$\sigma$ is a small constant which we set to~$\sigma = 0.125$. We test our algorithm in the FrozenLake8x8 benchmark environment available in OpenAI gym~\cite{brockman-et-al16openaigym}. 
The result of the experiment is illustrated in Figure~\ref{fig:plots}~(right). The performance curves show that our \algname{NVR-PG} algorithm shows a relatively faster convergence compared to the \algname{TSIVR-PG} algorithm~\cite{zhang-et-al21} and the \algname{MaxEnt} algorithm which is specific to the maximum entropy exploration problem (by~\citet{hazan-et-al19}) while the final performances are comparable (see also the overlapping shaded areas). We refer the reader to Section~6.3 in~\citet{zhang-et-al21} for further details regarding our setting. 

\noindent\textbf{(b) Standard RL.} While the focus of our work is on the general utility case beyond the standard RL setting, we also perform simulations for the particular case where the objective is a linear function of the state action occupancy measure (i.e., the standard cumulative reward setting) in the CartPole benchmark environment~\cite{brockman-et-al16openaigym}. Figure~\ref{fig:plots} (left) shows that our algorithm is competitive with \algname{TSIVR-PG} (actually even slightly faster, see between 250-500 episodes and see also the shaded areas) and all other algorithms which are not designed for the general utility case (\algname{REINFORCE}~\cite{REINFORCE_Williams_1992}, \algname{SVRPG}~\cite{xu-et-al20uai}, \algname{SRVR-PG}~\cite{xu-et-al20iclr}, \algname{HSPGA}~\cite{pham-et-al20}) 
while gaining simplicity compared to existing variance-reduced methods. Indeed, our algorithm is single-loop and does not require two distinct batch sizes and checkpoints nor does it require bounded importance sampling weights. Hyperparameters of the algorithms are tuned.

\section{Perspectives}%

Compared to the standard RL setting, the general utilities setting is much less studied. A better understanding of the hidden convexity structure of the problem and its interplay with general policy parametrization would be interesting to derive global optimality guarantees under milder assumptions which would accommodate more practical and expressive policy parametrizations such as neural networks. Regarding the case of large state action spaces, 
future avenues of research include designing more efficient procedures and guarantees for approximating and estimating the occupancy measure to better address the curse of dimensionality as well as investigating the dual point of view for designing more efficient algorithms. Addressing the case of continuous state-action spaces is also an interesting research direction.

\section*{Acknowledgements}
This work was supported by ETH AI Center doctoral fellowship, ETH Foundations of Data Science (ETH-FDS), and ETH Research Grant funded via ETH Zurich Foundation.

\clearpage

\bibliography{biblio}
\bibliographystyle{icml2023}

\newpage

\appendix
\onecolumn

\tableofcontents

\begin{center}
    \bfseries\Large Appendix
\end{center}

\section{\algname{N-VR-PG} algorithm for standard RL setting with cumulative reward}
\label{app:nvrpg-standard-RL}

In this section, we report the special case of Algorithm~\ref{algo-gen-ut-func-approx2} for the standard cumulative sum of rewards setting. 

\begin{algorithm}[h]
   \caption{\algname{N-VR-PG} (Standard Cumulative Reward)}
   \label{alg:N-VR-PG}
\begin{algorithmic}
   \STATE {\bfseries Input:} $\theta_0$, $T$, $H$, $\{\eta_t\}_{t\geq 0}$, $\{\al_t\}_{t\geq 0}\,.$
   \STATE Sample $\tau_0$ of length~$H$ from~$\mathbb{M}$
   \STATE $d_0 = g(\tau_0, \theta_0)$
   \STATE $\theta_1 = \theta_0 + \al_0 \frac{d_0}{\|d_0\|}$
   \FOR{$t=1, \ldots, T -  1$}
        \STATE Sample $\tau_t$ of length~$H$ from~$\mathbb{M}$ and~$\pi_{\theta_t}$
        \STATE $v_t = g(\tau_t, \theta_t) - w(\tau_t|\theta_{t-1},\theta_t)g(\tau_t, \theta_{t-1})$
        \STATE $d_t = \eta_t g(\tau_t, \theta_t) + (1 - \eta_t) (d_{t-1} +  v_t)$
        \STATE $\theta_{t+1} = \theta_t + \al_t \frac{d_t}{\|d_t\|}$
    \ENDFOR
\end{algorithmic}
\end{algorithm}

\begin{remark}
If the direction~$d_t$ in Algorithms~\ref{algo-gen-ut} and~\ref{alg:N-VR-PG} is null, then we formally take~$\theta_{t+1} = \theta_t$. Note that~$d_t \neq 0$ with probability~1 in general. Observe for instance that with~$\eta_t = 1$, $d_t = 0$ means that the stochastic policy gradient is equal to zero which means we are already at a first-order stationary point in expectation. 
\end{remark}

\section{Dependence on~$(1-\gamma)^{-1}$}
\label{app:dependence_(1-gamma)-1}

In this section, we discuss the dependence of our convergence guarantee in Theorem~\ref{thm:fos-gen-ut-nvrpg} on~$(1-\gamma)^{-1}$ where $\gamma$ is the discount factor of the MDP. The dependence on $1-\gamma$ of our proposed method is $\tilde{\mathcal{O}} ((1-\gamma)^{-6} \varepsilon^{-3})$ compared to~$\tilde{\mathcal{O}} ((1-\gamma)^{-25} \varepsilon^{-3})$ for TSIVR-PG of~\cite{zhang-et-al21}.

\noindent\textbf{Dependence on $(1-\gamma)^{-1}$ of our \algname{N-VR-PG} algorithm.}  It follows from Theorem 4.4 by setting~$\alpha_0 = (1-\gamma)^2$ that we need~$\tilde{\mathcal{O}}((1-\gamma)^{-6} \varepsilon^{-3})$ samples to reach an $\varepsilon$-stationary point of the utility function (i.e., $\mathbb E \|\nabla_{\theta}F(\lambda(\theta_{out})) \| \leq \varepsilon$). 

\noindent\textbf{Derivation of explicit dependence on $(1-\gamma)^{-1}$ in Theorem 5.9 of~\cite{zhang-et-al21}.} Although the dependence is not made explicit in the aforementioned work, we can use their intermediate results in the proofs in order to derive it.  We use their notations in the following.

From the last two lines of page 26 (in the proof of Theorem 5.9), we can infer that $\mathbb E \left[ \| \mathcal G (\theta_{out}) \| \right] \leq \mathcal O ( C_3 \varepsilon )$, 
where $C_3 = \mathcal{O}(H (1-\gamma)^{-7})$ which is defined in the statement of Lemma F.2 on page 23. 
Here in $\mathcal{O}$ we only hide the dependence on the smoothness constants and other numerical constants. 
The statement of Theorem 5.9 guarantees that in order to achieve this, we need $T (m B + N) H = \mathcal{O}(H \varepsilon^{-3})$ number of samples. 
By setting $\varepsilon_1 = C_3 \varepsilon $, this translates to 
$\mathcal{O}(H C_3^{3} \varepsilon_{1}^{-3}) = 
\mathcal{O}(H^4 (1-\gamma)^{-21} \varepsilon_{1}^{-3}) 
= \tilde{\mathcal{O}} ((1-\gamma)^{-25} \varepsilon_{1}^{-3})$ samples to achieve 
$\mathbb{E} \left[ \| \mathcal G (\theta_{out}) \| \right] \leq \varepsilon_1$,
where in the last step we used the expression of~$H = \frac{2}{1-\gamma} \log(1/\varepsilon)$.
Moreover, if we translate this guarantee to a more standard stationarity measure 
$\mathbb{E}\left[\|\nabla_{\theta}F(\lambda(\theta))\|\right]$, 
the dependence on~$(1-\gamma)^{-1}$ may further degrade for \algname{TSIVR-PG}, see Lemma~5.4 in \cite{zhang-et-al21} 
where they establish $\mathbb{E}[\|\nabla_{\theta} F(\lambda(\theta))\|]
= \mathcal{O}(\delta^{-1} \mathbb{E}\left[ \| \mathcal{G}(\theta) \| \right])$ and~$\delta = \mathcal{O}(H^{-1}) = \tilde{\mathcal O}(1-\gamma)$ is the truncation parameter.
Indeed, their convergence result is stated in terms of the gradient mapping (because their algorithm uses a truncation mechanism with hyperparameter~$\delta$) and then translated to the standard first-order stationarity measure we use in this work.

\section{Further discussion of Assumption~\ref{hyp:overparam-global-opt}}
\label{app:overparam-assumption}

A few comments are in order regarding Assumption~\ref{hyp:overparam-global-opt}: 

\begin{enumerate}

\item In Assumption~\ref{hyp:overparam-global-opt}, the uniformity of the Lipschitz constant~$l$ (independent of~$\theta$) and~$\bar{\epsilon}$ is important. %
Without this requirement, for instance, for item~(iii), 
the existence of~$\bar{\epsilon}_{\theta} > 0$ (depending on~$\theta$) with the desired property for every~$\theta \in \R^d$ is always guaranteed since~$\mathcal{V}_{\lambda(\theta)}$ is an open set.

\item As it was recently reported in \citet{zhang-et-al20variational}, the direct parametrization satisfies the bijection Assumption~\ref{hyp:overparam-global-opt}. We refer the reader to Appendix~H in \cite{zhang-et-al20variational} for a complete proof of this fact. Notice also that   Assumption~\ref{hyp:overparam-global-opt} which was first introduced in~\cite{zhang-et-al21} is a local version of Assumption~1 in \cite{zhang-et-al20variational} and is hence less restrictive. 

\item As for the softmax parametrization, verifying Assumption~\ref{hyp:overparam-global-opt} is more challenging as we explain in the main part of the present paper. It is a delicate and interesting question to investigate whether Assumption~\ref{hyp:overparam-global-opt} accommodates complex policy parametrizations such as practical neural networks or if it can be relaxed to do so.
\end{enumerate}

\section{Proof of Lemma~\ref{lem:variance-IS-weights-control} in Section~\ref{subsec:normalization-boundedness-is}}

The proof of this lemma is simple and combines an elementary technical lemma from~\citet{zhang-et-al21} (Lemma~5.6) upper bounding the IS weights for the special case of the softmax parametrization with Lemma~B.1 in \citet{xu-et-al20iclr} which provides a bound on the variance of IS weights which are bounded as a function of the squared euclidean distance between two policy parameters. 

\begin{proof}
Using the softmax parametrization~\eqref{eq:softmax-param} with Assumption~\ref{hyp:policy-param} satisfied, Lemma~5.6 in~\citet{zhang-et-al21} stipulates that for every~$\theta, \theta' \in \R^d$ and every truncated trajectory~$\tau = (s_t, a_t)_{0 \leq t \leq H-1}$ of length~$H$, the IS weights defined in~\eqref{eq:IS-weights} satisfy:
\begin{equation}
w(\tau|\theta',\theta) \leq \exp\{2 H l_{\psi} \|\theta - \theta'\|\}\,.
\end{equation}

It suffices to observe that~$\|\theta_{t+1} - \theta_t\| = \al_t$ with the normalized update rule to prove that for every integer~$t$ and any trajectory~$\tau$ of length~$H$, we have~$w(\tau|\theta_{t},\theta_{t+1}) \leq \exp\{ 2 H l_{\psi} \al_t\}\,.$ This proves the first part of the result. 

Combining this first part with Lemma~B.1 in \citet{xu-et-al20iclr},
we obtain the desired fine-grained control of the IS weights variance. Specifically, if~$\tau_{t+1}$ is a trajectory of length~$H$ generated following the initial distribution~$\rho$ and policy~$\pi_{\theta_{t+1}}$, then
\begin{align}
\mathbb{E}[w(\tau_{t+1}|\theta_t,\theta_{t+1})] &= 1\,,\\
\Var{[w(\tau_{t+1}|\theta_t,\theta_{t+1})]} &\leq C_w \al^2\,, 
\end{align}
where~$C_w \eqdef H ((8H+2) l_{\psi}^2 + 2 L_{\psi}) (W+1)\,.$ 
\end{proof}

\section{Proofs for Section~\ref{subsec:fos}: First-order stationarity}
\label{sec:app-proofs-fos}

From the technical point of view, our proofs for the general utility setting depart from the proof techniques of~\citet{zhang-et-al21} in several ways although we share several common points. First, normalized policy gradient requires an adequate ascent-like lemma which is different from the analysis of gradient truncation mechanism. Second, our algorithm uses a different variance reduction scheme which does not require checkpoints and consists of a single loop. This requires careful changes in the proof. Compared to standard STORM variance reduction proofs in stochastic optimization~\citep{cutkosky-orabona19}, our setting involves two estimators which are intertwined, namely the state-action occupancy measure estimate and the stochastic policy gradient. This makes the analysis more complex and we refer the reader to the decomposition in~\eqref{eq:decomp-e-t} and the subsequent lemmas to observe this. In contrast, STORM only involves the stochastic gradient and uses a different update rule. More broadly, we believe the techniques we use here could also be useful for stochastic composite optimization.

\subsection{Proof sketch}
For the convenience of the reader, we highlight the main steps of the proof in this subsection before diving into the full detailed proof. The main steps consist in: 
\begin{enumerate}[(a)]
\item showing an ascent-like lemma on the general utility function (see Lemma~\ref{lem:ascent-like-lemma-gen-ut}). Notice that our algorithm is not a standard policy gradient algorithm but features normalization which requires a particular treatment; 

\item controlling the variance error due to two coupled stochastic estimates: the stochastic estimates of the state-action occupancy measures (for distinct policy parameters) and the stochastic policy gradients. Controlling these coupled estimates in our single-loop batch-free algorithm constitutes one of the main challenges of the proofs. More precisely, we use the following steps: 
    \begin{enumerate}[(i)]
        \item We decompose the overall stochastic policy gradients errors in estimating the true policy gradients in~\eqref{eq:decomp-e-t} into two errors (see~\eqref{eq:def-hat-e-tilde-e}): the error due to the state action occupancy measure estimation (which provides an estimate of the reward sequence) and the error due to the policy gradient given an estimate of the reward sequence. 

        \item We control each one of the aforementioned errors by establishing recursions in Lemma~\ref{lem:recursive-tilde-e-t} and Lemma~\ref{lem:recursive-hat-e-t} respectively. Then, we solve the resulting recursions in the second parts of the lemmas.  

        \item We sum up each one of the expected errors over time in Lemma~\ref{lem:recursive-tilde-e-t_sum} and Lemma~\ref{lem:recursive-tilde-e-t_sum} and we obtain an estimation of the overall error in Lemma~\ref{lem:overall-error-sum} by combining both errors using Lemma~\ref{lem:e-t-bound}. 
        
    \end{enumerate}

Further technical steps needed are described in the full proof (see for e.g. Lemma~\ref{lem:lambda-t-lambda-t-1}).

\item incorporating the estimates obtained in the second step to the descent lemma and telescoping  the obtained inequality to derive our final convergence guarantee (see the proof of Theorem~\ref{thm:end-of-proof} for the concluding steps). 
\end{enumerate}

\subsection{Proof of Theorem~\ref{thm:fos-gen-ut-nvrpg} (General utilities setting)}

In this section, we provide a proof for the case of general utilities. Notice that the case of cumulative rewards is a particular case. 
The first lemma is a an ascent-like lemma which follows from smoothness of the objective function.
Before stating the lemma, we define the error sequence~$(e_t)$ 
for every integer~$t$ as follows: 
\begin{equation}
\label{eq:e-t-def-gen-utilities}
e_t \eqdef d_t - \nabla_{\theta} F(\lambda_H(\theta_t))\,.
\end{equation}

\begin{lemma}
\label{lem:ascent-like-lemma-gen-ut}
Let Assumptions~\ref{hyp:policy-param} and~\ref{hyp:smoothness-F} hold true. Then, the sequence~$(\theta_t)$ generated by Algorithm~\ref{algo-gen-ut} and the sequence~$(e_t)$ satisfy for every integer~$t \geq 0$, 
\begin{equation}
F(\lambda(\theta_{t+1}))  \geq F(\lambda(\theta_t)) + \frac{\al_t}{3} \|\nabla_{\theta} F(\lambda(\theta_t))\| - 2 \al_t \|e_t\| - \frac{4}{3} D_{\lambda} \gamma^H \al_t 
- \frac{L_{\theta}}{2} \al_t^2\,.
\end{equation}
\end{lemma}

\begin{proof}
Since the objective function~$\theta \mapsto F(\lambda(\theta))$ is $L_{\theta}$-smooth by Lemma~\ref{lem:smoothness-obj}, we obtain the following by using the update rule of the sequence~$(\theta_t)$:
\begin{align}
\label{eq:smoothness-gen-ut}
F(\lambda(\theta_{t+1})) 
&\geq F(\lambda(\theta_{t})) + \ps{\nabla_{\theta} F(\lambda(\theta_t)), \theta_{t+1} - \theta_t} - \frac{L_{\theta}}{2} \|\theta_{t+1} - \theta_t\|^2\nonumber\\
&= F(\lambda(\theta_{t})) + \al_t \ps{\nabla_{\theta} F(\lambda(\theta_t)), \frac{d_t}{\|d_t\|}} - \frac{L_{\theta}}{2} \al_t^2\nonumber\\
&= F(\lambda(\theta_{t})) + \al_t \ps{\nabla_{\theta} F(\lambda_H(\theta_t)), \frac{d_t}{\|d_t\|}} + \al_t \ps{\nabla_{\theta} F(\lambda(\theta_t)) - \nabla_{\theta} F(\lambda_H(\theta_t)), \frac{d_t}{\|d_t\|}} - \frac{L_{\theta}}{2} \al_t^2\nonumber\\
&\geq F(\lambda(\theta_{t})) + \al_t \ps{\nabla_{\theta} F(\lambda_H(\theta_t)), \frac{d_t}{\|d_t\|}} - \al_t \|\nabla_{\theta} F(\lambda(\theta_t)) - \nabla_{\theta} F(\lambda_H(\theta_t))\|
- \frac{L_{\theta}}{2} \al_t^2\nonumber\\
&\stackrel{(a)}{\geq} F(\lambda(\theta_{t})) + \al_t \ps{\nabla_{\theta} F(\lambda_H(\theta_t)), \frac{d_t}{\|d_t\|}} - \al_t D_{\lambda} \gamma^H 
- \frac{L_{\theta}}{2} \al_t^2\,,
\end{align}
where~(a) follows from Lemma~\ref{lem:trunc_grad}-(i).

Then we control the scalar product term. We distinguish two different cases:\\

\noindent\textbf{Case 1:} $\|e_t\| \leq \frac{1}{2}\|\nabla_{\theta} F(\lambda_H(\theta_t))\|$\,.
In this case, we have 
\begin{align*}
\ps{\nabla_{\theta} F(\lambda_H(\theta_t)),\frac{d_t}{\|d_t\|}} &= \frac{1}{\|d_t\|}(\|\nabla_{\theta} F(\lambda_H(\theta_t))\|^2 + \ps{\nabla_{\theta} F(\lambda_H(\theta_t)), e_t})\\
& \geq \frac{1}{\|d_t\|} (\|\nabla_{\theta} F(\lambda_H(\theta_t))\|^2  - \|\nabla_{\theta} F(\lambda_H(\theta_t))\| \cdot \|e_t\|)\\
&= \frac{1}{\|d_t\|} \|\nabla_{\theta} F(\lambda_H(\theta_t))\|\cdot (\|\nabla_{\theta} F(\lambda_H(\theta_t))\|  - \|e_t\|)\\
&\geq \frac{1}{3}\|\nabla_{\theta} F(\lambda_H(\theta_t))\|\,,
\end{align*}
where the last inequality follows from observing that~$\|d_t\| \leq \|e_t\| + \|\nabla_{\theta} F(\lambda_H(\theta_t))\| \leq \frac{3}{2} \|\nabla_{\theta} F(\lambda_H(\theta_t))\|\,.$

\noindent\textbf{Case 2:} $\|e_t\| \geq \frac{1}{2}\|\nabla_{\theta} F(\lambda_H(\theta_t))\|$\,. In this case, we simply have 
\begin{equation*}
\ps{\nabla_{\theta} F(\lambda_H(\theta_t)),\frac{d_t}{\|d_t\|}} \geq - \|\nabla_{\theta} F(\lambda_H(\theta_t))\| \geq - 2 \|e_t\|\,.
\end{equation*}
Combining both cases, we obtain: 
\begin{align}
\label{eq:scalar-prod}
\al_t \ps{\nabla_{\theta} F(\lambda_H(\theta_t)),\frac{d_t}{\|d_t\|}} &\geq \frac{\al_t}{3} \|\nabla_{\theta} F(\lambda_H(\theta_t))\| - 2 \al_t \|e_t\| \nonumber\\
&\geq \frac{\al_t}{3} \|\nabla_{\theta} F(\lambda(\theta_t))\| -\frac{\al_t}{3} D_{\lambda}\gamma^H -2\al_t \|e_t\|\,.
\end{align}
Combining \eqref{eq:smoothness-gen-ut} and~\eqref{eq:scalar-prod}, we get
\begin{equation*}
F(\lambda(\theta_{t+1}))  \geq F(\lambda(\theta_t)) + \frac{\al_t}{3} \|\nabla_{\theta} F(\lambda(\theta_t))\| - 2 \al_t \|e_t\| - \frac{4}{3} D_{\lambda} \gamma^H \al_t 
- \frac{L_{\theta}}{2} \al_t^2\,,
\end{equation*}
which completes the proof. 
\end{proof}

We now proceed with some preliminary results in order to control the error term~$\|e_t\|$ in expectation. 

The next lemma appeared in Prop. E.1 \cite{zhang-et-al21}. 
\begin{remark}
With a slight abuse of notation, $\tau \sim \pi_{\theta}$ means that the trajectory~$\tau$ (of length~$H$) is sampled from the MDP controlled by the policy~$\pi_{\theta}$. We adopt this notation to highlight the dependence on the parametrized policy~$\pi_{\theta}$, the MDP being fixed in the problem formulation.
\end{remark}
\begin{lemma}
\label{lem:pg-estimate-lambda-estimate}
For any reward vector~$r \in \R^{|\mathcal{S}|\times|\mathcal{A}|},$ we have
\begin{align}
\mathbb{E}_{\tau \sim \pi_{\theta}}[\lambda(\tau)] &= \lambda_H(\theta)\,,\\
\mathbb{E}_{\tau \sim \pi_{\theta}}[g(\tau,\theta,r)]&= [\nabla_{\theta} \lambda_H(\theta)]^{T} r\,.
\end{align}
In particular, under Assumption~\ref{hyp:smoothness-F}, 
\begin{equation}
\mathbb{E}_{\tau \sim p(\cdot|\pi_{\theta})}[g(\tau,\theta,\nabla_{\lambda}F(\lambda_H(\theta)))] = [\nabla_{\theta} \lambda_H(\theta)]^{T} \nabla_{\lambda}F(\lambda_H(\theta)) = \nabla_{\theta} F(\lambda_H(\theta))\,.
\end{equation}
\end{lemma}
\begin{proof}
The proof follows from the definitions of the estimators~$\lambda(\tau)$ and~$g(\tau,\theta,r)$ in Eqs.~\eqref{eq:lambda-tau}-\eqref{eq:pg-estimate} and the definition of the truncated state-action occupancy measure~$\lambda_H(\theta)$ (see Eq.\eqref{eq:s-a-occup-measure}) as well as the policy gradient theorem in Eq.~\eqref{eq:expected-reinforce}. 
\end{proof}

In view of controlling the error sequence~$(e_t)$, we first observe the following decomposition:
\begin{align}
\label{eq:decomp-e-t}
e_t &= d_t - \nabla_{\theta} F(\lambda_H(\theta_t))\nonumber\\
    &= d_t - [\nabla_{\theta} \lambda_H(\theta_t)]^{T} r_{t-1} 
    + [\nabla_{\theta} \lambda_H(\theta_t)]^{T} (r_{t-1} - \nabla_{\lambda} F(\lambda_H(\theta_t)))\nonumber\\
    &= d_t - [\nabla_{\theta} \lambda_H(\theta_t)]^{T} r_{t-1} 
    + [\nabla_{\theta} \lambda_H(\theta_t)]^{T} (\nabla_{\lambda} F(\lambda_{t-1}) - \nabla_{\lambda} F(\lambda_H(\theta_t)))\,.
\end{align}

Given the previous decomposition, we define two useful additional notations:
\begin{align}
\label{eq:def-hat-e-tilde-e}
\hat{e}_t &\eqdef d_t - [\nabla_{\theta} \lambda_H(\theta_t)]^{T} r_{t-1}\,,\\
\tilde{e}_t &\eqdef \lambda_t - \lambda_H(\theta_t)\label{eq:def-hat-e-tilde-e2}\,. 
\end{align}

Using these notations, we establish the following result relating the error~$\|e_t\|^2$ to the errors~$\|\hat{e}_t\|^2$ and~$\|\tilde{e}_t\|^2$ in expectation. 

\begin{lemma}
\label{lem:e-t-bound}
Let Assumptions~\ref{hyp:policy-param} and~\ref{hyp:smoothness-F} hold true. Then we have for every integer~$t \geq 1$,
\begin{equation}
\mathbb{E}[\|e_t\|] \leq \mathbb{E}[\|\hat{e}_t\|] + C_1\, \mathbb{E}[\|\tilde{e}_{t-1}\|]  + C_2\, \al_{t-1}\,, 
\end{equation}
where~$C_1 \eqdef \frac{2 L_{\lambda}^2 l_{\psi}}{(1-\gamma)^2}\,$ and~$C_2 \eqdef \frac{2 L_{\lambda} L_{\lambda,\infty} l_{\psi}}{(1-\gamma)^2}\,.$
\end{lemma}

\begin{proof}
It follows from the decomposition in~\eqref{eq:decomp-e-t} and the definitions~\eqref{eq:def-hat-e-tilde-e}-\eqref{eq:def-hat-e-tilde-e2} that
\begin{equation}
\label{eq:ineq1-e-t}
\mathbb{E}[\|e_t\|] \leq \mathbb{E}[\|\hat{e}_t\|] +  \mathbb{E}[\|[\nabla_{\theta} \lambda_H(\theta_t)]^{T} (\nabla_{\lambda} F(\lambda_{t-1}) - \nabla_{\lambda} F(\lambda_H(\theta_t)))\|]\,.
\end{equation}
We now control the second term in the above inequality. The following step is similar to the treatment in \cite{zhang-et-al21}(Eq.(18)). Indeed, the policy gradient theorem (see \eqref{eq:expected-reinforce}) yields

\begingroup\makeatletter\def\f@size{9.5}\check@mathfonts
$$
[\nabla_{\theta} \lambda_H(\theta_t)]^{T} (\nabla_{\lambda} F(\lambda_{t-1}) - \nabla_{\lambda} F(\lambda_H(\theta_t))) 
= \mathbb{E}\left[\sum_{t'=0}^{H-1} \gamma^{t'}  [\nabla_{\lambda} F(\lambda_{t-1}) - \nabla_{\lambda} F(\lambda_H(\theta_t))]_{s_{t'}, a_{t'}} \cdot \left( \sum_{h=0}^{t'} \nabla_{\theta} \log \pi_{\theta}(a_{h},s_{h})\right) \right] 
$$
\endgroup
As a consequence, 
\begingroup\makeatletter\def\f@size{9.5}\check@mathfonts
\begin{align}
\label{eq:grad-theta-lambda-grad-lambda-diff}
\|[\nabla_{\theta} \lambda_H(\theta_t)]^{T} (\nabla_{\lambda} F(\lambda_{t-1}) - \nabla_{\lambda} F(\lambda_H(\theta_t)))\| 
&\leq \mathbb{E}\left[ \sum_{t'=0}^{H-1} \gamma^{t'} \|\nabla_{\lambda} F(\lambda_{t-1}) - \nabla_{\lambda} F(\lambda_H(\theta_t))\|_{\infty} \left\|\sum_{h=0}^{t'} \nabla_{\theta} \log \pi_{\theta}(a_{h},s_{h}) \right\| \right]\,.
\end{align}
\endgroup
Then, using Assumption~\ref{hyp:smoothness-F}, we have
\begin{align}
\|\nabla_{\lambda} F(\lambda_{t-1}) - \nabla_{\lambda} F(\lambda_H(\theta_t))\|_{\infty} 
&\leq \|\nabla_{\lambda} F(\lambda_{t-1}) - \nabla_{\lambda} F(\lambda_H(\theta_{t-1}))\|_{\infty} + \|\nabla_{\lambda} F(\lambda_H(\theta_{t-1})) - \nabla_{\lambda} F(\lambda_H(\theta_t))\|_{\infty} \nonumber\\
&\leq L_{\lambda}\|\lambda_{t-1} - \lambda_{H}(\theta_{t-1})\| + L_{\lambda,\infty} \|\lambda_{H}(\theta_{t-1}) - \lambda_{H}(\theta_{t})\|_1 \nonumber\\
& \leq L_{\lambda} \|\tilde{e}_{t-1}\| + L_{\lambda,\infty} \|\theta_t - \theta_{t-1}\|\,,
\end{align}
where the last inequality follows from Lemma~\ref{lem:smoothness-obj}-\ref{lem:smoothness-obj-ii}. Plugging this inequality in~\eqref{eq:grad-theta-lambda-grad-lambda-diff} and using Lemma~\ref{lem:smoothness-obj}-\ref{lem:smoothness-obj-i} yields 
\begin{align}
\label{eq:bound-pg-thm-temp-35}
\|[\nabla_{\theta} \lambda_H(\theta_t)]^{T} (\nabla_{\lambda} F(\lambda_{t-1}) - \nabla_{\lambda} F(\lambda_H(\theta_t)))\| &\leq \mathbb{E}\left[ \sum_{t'=0}^{H-1} 2 (t'+1) l_{\psi} \gamma^{t'} (L_{\lambda} \|\tilde{e}_{t-1}\| + L_{\lambda,\infty} \|\theta_t - \theta_{t-1}\|) \right]\nonumber\\
&= \left(\sum_{t'=0}^{H-1} 2 (t'+1) l_{\psi} \gamma^{t'} L_{\lambda} \right) (L_{\lambda} \mathbb{E}[\|\tilde{e}_{t-1}\|] + L_{\lambda,\infty} \al_{t-1})\nonumber\\
&\leq \frac{2 L_{\lambda} l_{\psi}}{(1-\gamma)^2}(L_{\lambda} \mathbb{E}[\|\tilde{e}_{t-1}\|] + L_{\lambda,\infty} \al_{t-1})\,.
\end{align}

Hence, taking the total expectation, we obtain 
\begin{equation}
\label{eq:ineq2-e-t}
\mathbb{E}[\|[\nabla_{\theta} \lambda_H(\theta_t)]^{T} (\nabla_{\lambda} F(\lambda_{t-1}) - \nabla_{\lambda} F(\lambda_H(\theta_t)))\|] 
\leq \frac{2 L_{\lambda}^2 l_{\psi}}{(1-\gamma)^2} \mathbb{E}[\|\tilde{e}_{t-1}\|] 
+ \frac{2 L_{\lambda} L_{\lambda,\infty} l_{\psi}}{(1-\gamma)^2} \al_{t-1}\,.
\end{equation}

Combining~\eqref{eq:ineq1-e-t} and~\eqref{eq:ineq2-e-t} yields the desired inequality. 
\end{proof}

We now control each one of the errors in the right-hand side of the previous lemma in what follows. We start with the error~$\tilde{e}_t$ (see~\eqref{eq:def-hat-e-tilde-e2}) induced by the (truncated) state-action occupancy measure estimation. 

\begin{lemma}
\label{lem:recursive-tilde-e-t}
Let Assumption~\ref{hyp:policy-param} hold. Then, for every integer~$t \geq 1$, if~$\eta_t \in [0,1]$ we have
\begin{equation}\label{eq:recursive-tilde-e-t}
\mathbb{E}[\|\tilde{e}_t\|^2] \leq (1-\eta_t) \mathbb{E}[\|\tilde{e}_{t-1}\|^2] + \frac{2 C_w}{(1-\gamma)^2} \al_{t-1}^2 + \frac{2}{(1-\gamma)^2}\eta_t^2\,,
\end{equation}
where we recall that~$\tilde{e}_t = \lambda_t - \lambda_H(\theta_t)$ and~$C_w =  H ((8H+2) l_{\psi}^2 + 2 L_{\psi}) (W+1)\,$ as defined in Lemma~\ref{lem:variance-IS-weights-control}. Moreover, 
\begin{enumerate}[label=(\roman*)]
    \item if $\eta_t = \fr{2}{t+1}$, then for all integers $t \geq 1$, we have
\begin{equation}\label{eq:recursive-tilde-e-t_solved}
\mathbb{E}[\|\tilde{e}_t\|] \leq  \frac{4}{(1-\gamma)}\eta_t \cdot t^{\nfr{1}{2}} + \frac{2 C_w^{\nfr{1}{2}}}{(1-\gamma)} \al_{t-1}  \cdot t^{\nfr{1}{2}}   .
\end{equation}

    \item if $\eta_t = \rb{ \fr{2}{t+1} }^{q}$ and $\al_t = \al \rb{ \fr{2  }{t+1} }^{p}$ for some reals~$\alpha > 0$, $q \in (0,1)$, $p \geq 0$ and all integers $ t \geq 1$, then we have
\begin{eqnarray}\label{eq:recursive-tilde-e-t_solved_2}
\Exp{\|\tilde{e}_{t}\|^2} 
&\leq& \fr{(2C+1) \eta_{t+1} }{(1-\gamma)^2 } + \frac{2 C C_w}{(1-\gamma)^2} \al_{t}^2 \eta_{t+1}^{-1}
\,,
\end{eqnarray}
\end{enumerate}
where $C > 0$ is an absolute numerical constant. 
\end{lemma}

\begin{proof}
We start with the proof of \eqref{eq:recursive-tilde-e-t}. Using the update rule of the sequence~$(\lambda_t)$ in Algorithm~\ref{algo-gen-ut}, we first derive a recursion on the error sequence~$\tilde{e}_t$ from the following decomposition: 
\begin{align*}
\tilde{e}_t &= \lambda_t - \lambda_H(\theta_t)\\
&= \eta_t \lambda(\tau_t) + (1-\eta_t) (\lambda_{t-1} + u_t) - \lambda_H(\theta_t)\\
&= (1-\eta_t) \tilde{e}_{t-1} + (1-\eta_t) (\lambda_H(\theta_{t-1}) + u_t) + \eta_t (\lambda(\tau_t) - \lambda_H(\theta_t)) -(1-\eta_t) \lambda_H(\theta_t)\\
&= (1-\eta_t) \tilde{e}_{t-1} + (1-\eta_t) \tilde{z}_t + \eta_t \tilde{y}_t\,,
\end{align*}
where 
\begin{align}
\tilde{y}_t &\eqdef \lambda(\tau_t) - \lambda_H(\theta_t)\,,\label{def:tilde-y}\\
\tilde{z}_t &\eqdef u_t - (\lambda_H(\theta_t) - \lambda_H(\theta_{t-1})) = \lambda(\tau_t)(1-w(\tau_t|\theta_{t-1},\theta_t)) - (\lambda_H(\theta_t) - \lambda_H(\theta_{t-1}))   \label{def:tilde-z}\,.
\end{align}
Using these notations, we have 
\begin{align}
\label{eq:tilde-e-t-expansion}
\mathbb{E}[\|\tilde{e}_t\|^2] = (1-\eta_t)^2 \mathbb{E}[\|\tilde{e}_{t-1}\|^2] + \mathbb{E}[\|(1-\eta_t) \tilde{z}_t + \eta_t \tilde{y}_t\|^2] + \mathbb{E}[\ps{(1-\eta_t) \tilde{e}_{t-1},(1-\eta_t) \tilde{z}_t + \eta_t \tilde{y}_t}]\,.
\end{align}
Then, we notice that the scalar product term is equal to zero. We consider for this the filtration~$(\mathcal{F}_t)$ of $\sigma$-algebras defined s.t. for every integer~$t$, $\mathcal{F}_t \eqdef \sigma(\theta_k, \tau_k : k \leq t)$ where~$\tau_t$ is a (random) trajectory of length~$H$ generated following the policy~$\pi_{\theta_t}$. This $\sigma$-algebra represents the history of all the random variables until time~$t$. As a consequence, the random variable~$\tilde{e}_t$ being~$\mathcal{F}_{t-1}$-measurable, it follows from the tower property of the conditional expectation that 
\begin{align}
\label{eq:zero-inner-prod-lambda}
\mathbb{E}[\ps{(1-\eta_t) \tilde{e}_{t-1},(1-\eta_t) \tilde{z}_t + \eta_t \tilde{y}_t}] 
&= \mathbb{E}[\mathbb{E}[\ps{(1-\eta_t) \tilde{e}_{t-1},(1-\eta_t) \tilde{z}_t + \eta_t \tilde{y}_t} |\mathcal{F}_{t-1}]]\nonumber\\
&= \mathbb{E}[\ps{(1-\eta_t) \tilde{e}_{t-1}, \mathbb{E}[(1-\eta_t) \tilde{z}_t + \eta_t \tilde{y}_t} |\mathcal{F}_{t-1}]]\nonumber\\
&= 0\,,
\end{align}
where the last step stems from the fact that~$\mathbb{E}[\tilde{z}_t  |\mathcal{F}_{t-1}] = \mathbb{E}[\tilde{y}_t |\mathcal{F}_{t-1}] = 0$, recall for this that~$\tau_t \sim \pi_{\theta_t}$ for every~$t$ and see the definitions~\eqref{def:tilde-y} and~\eqref{def:tilde-z}. 

It follows from~\eqref{eq:tilde-e-t-expansion} and~\eqref{eq:zero-inner-prod-lambda} that
\begin{equation}
\label{eq:bound-tilde-e}
\mathbb{E}[\|\tilde{e}_t\|^2] \leq (1-\eta_t)^2 \mathbb{E}[\|\tilde{e}_{t-1}\|^2] + 2 (1-\eta_t)^2 \mathbb{E}[\|\tilde{z}_t\|^2] + 2 \eta_t^2 \mathbb{E}[\|\tilde{y}_t\|^2]\,.
\end{equation}
Then, we upperbound each one of the last two terms in~\eqref{eq:bound-tilde-e}. As for the first term, since~$\mathbb{E}[\tilde{z}_t] = 0$, we have the following 
\begin{equation}
\label{eq:control-var-is-weights}
\mathbb{E}[\|\tilde{z}_t\|^2] 
\leq \mathbb{E}[\|u_t\|^2] 
= \mathbb{E}[\| \lambda(\tau_t)(1-w(\tau_t|\theta_{t-1},\theta_t))\|^2]
=  \mathbb{E}[(1-w(\tau_t|\theta_{t-1},\theta_t))^2 \|\lambda(\tau_t)\|^2]\,.
\end{equation}
Given the definition of~$\lambda(\tau_t)$ in~\eqref{eq:lambda-tau}, we first observe that with probability one, 
\begin{equation}
\label{eq:lambda-tau-t-bound}
\|\lambda(\tau_t)\| \leq \sum_{t=0}^{H-1} \gamma^t \|\delta_{s_t,a_t}\| = \sum_{t=0}^{H-1} \gamma^t \leq \frac{1}{1-\gamma}\,.
\end{equation}
Using Lemma~\ref{lem:variance-IS-weights-control} together with the previous bound, we get
\begin{equation}
\label{eq:var-IS-weight-bound-1}
\mathbb{E}[(1-w(\tau_t|\theta_{t-1},\theta_t))^2 \|\lambda(\tau_t)\|^2] 
\leq \frac{1}{(1-\gamma)^2} \mathbb{E}[(1-w(\tau_t|\theta_{t-1},\theta_t))^2] = \frac{1}{(1-\gamma)^2} \Var{[w(\tau_t|\theta_{t-1},\theta_t)]} \leq \frac{C_w}{(1-\gamma)^2} \al_{t-1}^2\,.
\end{equation}
We deduce from~\eqref{eq:control-var-is-weights} and~\eqref{eq:var-IS-weight-bound-1} together that
\begin{equation}
\label{eq:tilde-z-bound}
\mathbb{E}[\|\tilde{z}_t\|^2]  \leq \frac{C_w}{(1-\gamma)^2} \al_{t-1}^2\,.
\end{equation}
Regarding the last term in~\eqref{eq:bound-tilde-e}, since~$\mathbb{E}[\tilde{y}_t] = 0$, we observe that
\begin{equation}
\label{eq:tilde-y-t-bound}
\mathbb{E}[\|\tilde{y}_t\|^2] \leq \mathbb{E}[\|\lambda(\tau_t)\|^2] \leq \frac{1}{(1-\gamma)^2}\,,
\end{equation}
where the last inequality stems from~\eqref{eq:lambda-tau-t-bound}.
Incorporating~\eqref{eq:tilde-z-bound} and~\eqref{eq:tilde-y-t-bound} into~\eqref{eq:bound-tilde-e} leads to the following inequality 
\begin{equation}\label{eq:recursive-tilde-e-t_proof}
\mathbb{E}[\|\tilde{e}_t\|^2] \leq (1-\eta_t)^2 \mathbb{E}[\|\tilde{e}_{t-1}\|^2] + \frac{2 C_w}{(1-\gamma)^2} (1-\eta_t)^2 \al_{t-1}^2 + \frac{2}{(1-\gamma)^2}\eta_t^2\,,
\end{equation}
which concludes the proof of the first since~$\eta_t \in [0,1]$.

\noindent\textbf{Proof of~\eqref{eq:recursive-tilde-e-t_solved}:} In order to derive~\eqref{eq:recursive-tilde-e-t_solved}, we apply Lemma~\ref{le:aux_rec0} with $ \eta_t = \fr{2}{t+1}$, $\beta_t =  \frac{2}{(1-\gamma)^2}\eta_t^2 + \frac{2 C_w}{(1-\gamma)^2} \al_{t-1}^2 $. Using $\mathbb{E}[\|\tilde{e}_{0}\|^2] \leq \mathbb{E}[\|\lambda(\tau_t)\|^2] \leq \frac{1}{(1-\gamma)^2}$, we derive

\begin{eqnarray}
\mathbb{E}[\|\tilde{e}_t\|] &\leq& \rb{ \mathbb{E}[\|\tilde{e}_t\|^2] }^{\fr{1}{2}} \leq \rb{ \fr{4 }{(1-\gamma)^2(t+1)^2} + \frac{2}{(1-\gamma)^2}\eta_t^2 \cdot t + \frac{2 C_w}{(1-\gamma)^2} \al_{t-1}^2  \cdot t }^{\nfr{1}{2}} \notag \\
&\leq & \fr{2 }{(1-\gamma)(t+1)} + \frac{2}{(1-\gamma)}\eta_t \cdot t^{\nfr{1}{2}} + \frac{2 C_w^{\nfr{1}{2}}}{(1-\gamma)} \al_{t-1}  \cdot t^{\nfr{1}{2}}  \notag \\
 &\leq &   \frac{4}{(1-\gamma)}\eta_t \cdot t^{\nfr{1}{2}} + \frac{2 C_w^{\nfr{1}{2}}}{(1-\gamma)} \al_{t-1}  \cdot t^{\nfr{1}{2}}  \,.
\end{eqnarray}

\noindent\textbf{Proof of~\eqref{eq:recursive-tilde-e-t_solved_2}:} Let $\eta_t = \rb{ \fr{2}{t+1} }^{q}$ for some $q \in (0,1)$. In order to derive~\eqref{eq:recursive-tilde-e-t_solved_2}, we unroll the recursion~\eqref{eq:recursive-tilde-e-t} from $t = 1$ to $t = t'$ where $t' \leq T-1$. Denoting $\beta_t =  \frac{2}{(1-\gamma)^2}\eta_t^2 + \frac{2 C_w}{(1-\gamma)^2} \al_{t-1}^2 $ ,  we have
\begin{eqnarray}\label{eq:et-prime-tilde}
\Exp{\|\tilde{e}_{t'}\|^2} 
&\leq& \prod_{\tau = 1}^{t'} (1-\eta_{\tau}) \Exp{\|\tilde{e}_{0}\|^2} + \sum_{t = 1}^{t'} \beta_{t} \prod_{\tau = t+1 }^{t'} (1 - \eta_{\tau}) \notag \\
&\leq& \fr{\eta_{t'+1} }{(1-\gamma)^2 } + C \beta_{t'+1} \eta_{t'+1}^{-1}
\,,
\end{eqnarray}
where we used $\Exp{ \|\tilde{e}_{0}\|^2 } \leq \Exp{\|\lambda(\tau_t)\|^2} \leq \frac{1}{(1-\gamma)^2}$ and the results of Lemmas~\ref{le:prod_bound}-\ref{le:sum_prod_bound1} with $C > 1$ being a numerical constant.

\end{proof}

\begin{lemma}
\label{lem:recursive-tilde-e-t_sum}
Let Assumption~\ref{hyp:policy-param} hold. Let~$\alpha_0 > 0$ and consider an integer~$T \geq 1\,.$ Set~$\eta_t = \rb{ \fr{2}{t+1} }^{\nfr{2}{3}}$ and~$\al_t = \fr{\al_0}{T^{\nfr{2}{3}}}$ for every nonzero integer~$t \leq T$. Then, we have 
        \begin{equation}\label{eq:recursive-tilde-e-t_sum}
        \fr{1}{T} \sum_{t=1}^{T} \mathbb{E}[\|\tilde{e}_t\|] \leq  \frac{ C \rb{ 1 + C_w^{\nfr{1}{2}} \al_0 } }{(1-\gamma)} \fr{1}{T^{\nfr{1}{3}} }   ,
        \end{equation}
        where $C > 0$ is an absolute numerical constant. 
\end{lemma}
\begin{proof}
    Summing up inequality \eqref{eq:recursive-tilde-e-t_solved_2} from Lemma~\ref{lem:recursive-tilde-e-t_sum} from $t = 1$ to $t = T$ and choosing $\al_t = \al = \fr{\al_0}{T^{\nfr{2}{3}}}$, we obtain
\begin{eqnarray}
\label{eq:average-expected-tilde-e-proof}
\fr{1}{T}\sum_{t=1}^{T} \Exp{\|\tilde{e}_t\|} &\leq&  \fr{1}{T} \sum_{t=1}^{T} \rb{\Exp{\|\tilde{e}_t\|^2} }^{\nfr{1}{2}} \notag \\
&\leq& \rb{ \fr{1}{T} \sum_{t=1}^{T} \Exp{\|\tilde{e}_t\|^2}   }^{\nfr{1}{2}}  \notag \\
&\leq& \rb{ \fr{1}{T} \sum_{t=1}^{T}  \fr{\eta_{t+1} }{(1-\gamma)^2 } + C \frac{2}{(1-\gamma)^2}\eta_{t+1} + \frac{2 C_w}{(1-\gamma)^2} \al^2  \eta_{t+1}^{-1}  }^{\nfr{1}{2}}  \notag \\
&\overset{(i)}{\leq}& \rb{ \fr{3 \eta_{T-1} }{(1-\gamma)^2 } +  \frac{6 C \eta_{T-1}  }{(1-\gamma)^2}+ \frac{2 C C_w}{(1-\gamma)^2} \fr{\al^2}{\eta_{T+1}}  }^{\nfr{1}{2}}  \notag \\
&\leq& \rb{  \frac{9 C \eta_{T-1} }{(1-\gamma)^2} + \frac{2 C C_w}{(1-\gamma)^2} \fr{\al^2}{\eta_{T+1}}  }^{\nfr{1}{2}}  \notag \\
&\leq& \rb{  \frac{ 18 C}{(1-\gamma)^2 T^{\nfr{2}{3} } } + \frac{ 3 C C_w}{(1-\gamma)^2} \fr{\al_0^2}{T^{\nfr{2}{3}} }  }^{\nfr{1}{2}}  \notag \\
&\leq&  \frac{ 12 C^{\nfr{1}{2}} \rb{ 1 + C_w^{\nfr{1}{2}} \al_0 } }{(1-\gamma)} \fr{1}{T^{\nfr{1}{3}} }\,.
\end{eqnarray}
where $(i)$ follows from observing that
\begin{equation}\label{eq:sum_by_int_bound}
\sum_{t=1}^{T} \eta_{t+1} \leq 2^{\fr{2}{3}}\sum_{t=1}^{T}\fr{1}{(t+2)^{\fr{2}{3}}} \leq 2^{\fr{2}{3}}\int_{t=1}^{T}\fr{1}{(t+1)^{\fr{2}{3}}}\leq 3 \cdot 2^{\fr{2}{3}} T^{\fr{1}{3}} = 3 \, T \fr{ 2^{\fr{2}{3}}}{T^{\fr{2}{3}}} = 3 \, T \, \eta_{T-1}.
\end{equation}
\end{proof}

In view of controlling the error~$\hat{e}_t$, we first state a technical lemma that will be useful. This result controls the expected squared difference between two consecutive estimates of the (truncated) state-action occupancy measure. 

\begin{lemma}
\label{lem:lambda-t-lambda-t-1}
Suppose Assumption~\ref{hyp:policy-param} holds. Then for all integers~$t \geq 1$, 
\begin{equation}
\label{eq:1stpart-lambda-t-lambda-t-1}
\mathbb{E}[\|\lambda_{t-1} - \lambda_t\|^2] \leq \frac{3 \eta_t^2}{(1-\eta_t)^2} \mathbb{E}[\|\tilde{e}_t\|^2]  + \frac{3 \eta_t^2}{(1-\eta_t)^2(1-\gamma)^2}  + \frac{3C_w}{(1-\gamma)^2} \al_{t-1}^2\,,
\end{equation}
where~$C_w =  H ((8H+2) l_{\psi}^2 + 2 L_{\psi}) (W+1)\,$ as defined in Lemma~\ref{lem:variance-IS-weights-control}. Moreover, if in addition~$\eta_t = \rb{ \fr{2}{t+1} }^{q}$ for some $q \in [0,1)$ then for every integer~$t \geq 1$,
\begin{equation}
\mathbb{E}[\|\lambda_{t-1} - \lambda_t\|^2] \leq \frac{12 C  \eta_t^2}{ (1-\gamma)^2} 
+ \frac{6 C_w}{ (1-\gamma)^2} \al_{t-1}^2\,,
\end{equation}
where~$C > 0$ is a numerical constant.
\end{lemma}

\begin{proof}
Using the update rule of the truncated occupancy measure estimate sequence~$(\lambda_t)$, we have
\begin{align*}
\lambda_{t-1}- \lambda_t &= \lambda_{t-1} - [\eta_t \lambda(\tau_t) + (1-\eta_t) (\lambda_{t-1} + u_t)]\\
&= \eta_t (\lambda_{t-1} - \lambda(\tau_t)) - (1-\eta_t) u_t\\
&= \eta_t (\lambda_{t-1} - \lambda_t) + \eta_t (\lambda_{t} - \lambda(\tau_t)) - (1-\eta_t) u_t\,.
\end{align*}
As a consequence, we have
\begin{align}
\lambda_{t-1}- \lambda_t &= \frac{\eta_t}{1-\eta_t} (\lambda_{t} - \lambda(\tau_t)) - u_t \nonumber\\
&= \frac{\eta_t}{1-\eta_t} \tilde{e}_t + \frac{\eta_t}{1-\eta_t}(\lambda_H(\theta_t) - \lambda(\tau_t)) - u_t\,.
\end{align}

Taking expectation of the square of the previous identity, we obtain the following bound: 
\begin{align}
\label{eq:bound-lambda-t-t-1}
\mathbb{E}[\|\lambda_{t-1} - \lambda_t\|^2] \leq \frac{3 \eta_t^2}{(1-\eta_t)^2} \mathbb{E}[\|\tilde{e}_t\|^2] +  \frac{3 \eta_t^2}{(1-\eta_t)^2} \mathbb{E}[\|\lambda(\tau_t) - \lambda_H(\theta_t)\|^2] + 3 \mathbb{E}[\|u_t\|^2]\,.
\end{align}

Recall now that~$\mathbb{E}[\|u_t\|^2] \leq \frac{C_w}{(1-\gamma)^2} \al_{t-1}^2$ and~$\mathbb{E}[\|\lambda(\tau_t) - \lambda_H(\theta_t)\|^2] \leq \frac{1}{(1-\gamma)^2}$ from~\eqref{eq:control-var-is-weights}-\eqref{eq:tilde-z-bound} and~\eqref{eq:tilde-y-t-bound} respectively. 
Incorporating these bounds into~\eqref{eq:bound-lambda-t-t-1} yields: 
\begin{align*}
\mathbb{E}[\|\lambda_{t-1} - \lambda_t\|^2] &\leq \frac{3 \eta_t^2}{(1-\eta_t)^2} \mathbb{E}[\|\tilde{e}_t\|^2] 
+ \frac{3 \eta_t^2}{(1-\eta_t)^2(1-\gamma)^2} 
+ \frac{3C_w}{(1-\gamma)^2} \al_{t-1}^2\,.
\end{align*}

This completes the proof of~\eqref{eq:1stpart-lambda-t-lambda-t-1}. We now set $\eta_t = \rb{ \fr{2}{t+1} }^{q}$ for some $q \in (0,1)$. By \eqref{eq:et-prime-tilde} in the proof of Lemma~\ref{lem:recursive-tilde-e-t}, we have 
\begin{eqnarray}
\Exp{\|\tilde{e}_{t}\|^2} 
&\leq& \fr{\eta_{t} }{(1-\gamma)^2 } + C \beta_{t} \eta_{t}^{-1}
\,,
\end{eqnarray}
where $\beta_t =  \frac{2}{(1-\gamma)^2}\eta_t^2 + \frac{2 C_w}{(1-\gamma)^2} \al_{t-1}^2 $, and $ C > 0$ is a numerical constant. 
Thus,
\begin{align*}
\mathbb{E}[\|\lambda_{t-1} - \lambda_t\|^2] &\leq \frac{3 \eta_t^2}{(1-\eta_t)^2} \rb{ \fr{\eta_{t} }{(1-\gamma)^2 } + C \beta_{t} \eta_{t}^{-1} }
+ \frac{3 \eta_t^2}{(1-\eta_t)^2(1-\gamma)^2} 
+ \frac{3C_w}{(1-\gamma)^2} \al_{t-1}^2\, \notag \\
&\leq  \frac{12 C  \eta_t^2}{ (1-\gamma)^2} 
+ \frac{6 C_w}{ (1-\gamma)^2} \al_{t-1}^2\,.
\end{align*}

\end{proof}

We are now ready to prove a recursive upper bound on the error sequence~$(\hat{e}_t)$ defined in~\eqref{eq:def-hat-e-tilde-e}. Notice that this result is of the same flavor as Lemma~\ref{lem:recursive-tilde-e-t} which we already proved. In particular, the result illustrates a variance reduction effect stemming from the variance reduction updates used for both the stochastic policy gradients and the state-action occupancy measure estimates.   
\begin{lemma}
\label{lem:recursive-hat-e-t}
Suppose Assumptions~\ref{hyp:policy-param} and~\ref{hyp:smoothness-F} hold. %
Then, for every integer~$t \geq 2$, 
\begin{eqnarray}
    \mathbb{E}[\|\hat{e}_t\|^2] \leq (1-\eta_t)^2 \mathbb{E}[\|\hat{e}_{t-1}\|^2]   + C_3 \eta_{t-1}^2 + 
C_4 \al_{t-2}^2,
\end{eqnarray}
where~$C_3 \eqdef  \frac{288 C l_{\psi}^2 L_{\lambda}^2}{(1-\gamma)^6} 
+ \frac{32 l_{\lambda}^2 l_{\psi}^2}{(1-\gamma)^4} $, $C_4 \eqdef  \frac{12 l_{\lambda}^2[(l_{\psi}^2 + L_{\psi})^2 + C_w l_{\psi}^2]}{(1-\gamma)^4}  + \frac{144 C_w l_{\psi}^2 L_{\lambda}^2}{(1-\gamma)^6}  $, and $C_w =  H ((8H+2) l_{\psi}^2 + 2 L_{\psi}) (W+1)\,$ as defined in Lemma~\ref{lem:variance-IS-weights-control}. Moreover,
\begin{enumerate}[label=(\roman*)]
    \item if $\eta_t = \fr{2}{t+1}$, then for all integers $t \geq 1$, we have
\begin{eqnarray}\label{eq:recursive-hat-e-t_solved}
\mathbb{E}[\|\hat{e}_t\|]
&\leq & \fr{2 \hat{E} }{t+1} + 2 C_3^{\nfr{1}{2}} \eta_{t} \cdot t^{\nfr{1}{2}} + C_4^{\nfr{1}{2}} \al_{t-2}  \cdot t^{\nfr{1}{2}}   \,.
\end{eqnarray}

    \item if $\eta_t = \rb{ \fr{2}{t+1} }^{q}$ and $\al_t = \al \rb{ \fr{2  }{t+1} }^{p}$ for some reals~$\alpha > 0$, $q \in (0,1)$, $p \geq 0$ and all integers $ t \geq 1$, then we have
\begin{eqnarray}\label{eq:recursive-hat-e-t_solved_2}
    \Exp{\|\hat{e}_{t}\|^2} 
&\leq& \hat{E}^2 \eta_{t+1} + 2 C C_3 \eta_{t+1}  + C C_4 \al_{t-1}^2 \eta_{t+1}^{-1} ,
\end{eqnarray}
\end{enumerate}
where $\hat{E} = \frac{4 l_{\lambda} l_{\psi}}{(1-\gamma)^2} $ .
\end{lemma}

\begin{proof}
The first step of the proof consists in decomposing the error~$\hat{e}_t$ in a suitable way using the update rule of the sequence~$(d_t)$ so that for every integer~$t \geq 2$, 
\begin{align*}
\hat{e}_t &= d_t - [\nabla_{\theta} \lambda_H(\theta_t)]^T r_{t-1}\\
&= (1-\eta_t) (d_{t-1} + v_t) + \eta_t g(\tau_t, \theta_t, r_{t-1}) - [\nabla_{\theta} \lambda_H(\theta_t)]^T r_{t-1}\\
&= (1-\eta_t) (\hat{e}_{t-1} + [\nabla_{\theta} \lambda_H(\theta_{t-1})]^T r_{t-2} + v_t) + \eta_t (g(\tau_t,\theta_t,r_{t-1}) -  [\nabla_{\theta} \lambda_H(\theta_t)]^T r_{t-1}) - (1-\eta_t) [\nabla_{\theta} \lambda_H(\theta_t)]^T r_{t-1}\\
&= (1-\eta_t) \hat{e}_{t-1} + (1-\eta_t) \hat{z}_t + \eta_t \hat{y}_t\,,
\end{align*}
where 
\begin{align}
\hat{y}_t &\eqdef g(\tau_t,\theta_t,r_{t-1}) -  [\nabla_{\theta} \lambda_H(\theta_t)]^T r_{t-1}\,,  \label{def:hat-y-t}\\
\hat{z}_t &\eqdef v_t - ([\nabla_{\theta} \lambda_H(\theta_t)]^T r_{t-1} - [\nabla_{\theta} \lambda_H(\theta_{t-1})]^T r_{t-2})\\
&= g(\tau_t,\theta_t,r_{t-1}) - [\nabla_{\theta} \lambda_H(\theta_t)]^T r_{t-1} + [\nabla_{\theta} \lambda_H(\theta_{t-1})]^T r_{t-2} - w(\tau_t|\theta_{t-1},\theta_t)g(\tau_t,\theta_{t-1},r_{t-2})\,.
\end{align}

Then we use similar derivations to~\eqref{eq:tilde-e-t-expansion} and~\eqref{eq:zero-inner-prod-lambda}. We consider again the same filtration~$(\mathcal{F}_t)$ of $\sigma$-algebras where~$\mathcal{F}_t$ represents the randomness until time~$t$ (including time~$t$ and random trajectories~$\tau_t$ of length~$H$ generated by following the policy~$\pi_{\theta_t}$). Therefore, we have (see Lemma~\ref{lem:pg-estimate-lambda-estimate})
\begin{align}
\label{eq:zero-cond-expect-hat-y-hat-z}
\mathbb{E}[ \hat{y}_t| \mathcal{F}_{t-1}] = 0\,;\quad \mathbb{E}[ \hat{z}_t| \mathcal{F}_{t-1}] = 0\,.
\end{align}
Note here as a comment that the reason why we used~$r_{t-1}$ instead of~$r_t$ in Algorithm~\ref{algo-gen-ut} (for the sequence~$(v_t)$) becomes clearer here in the previous identities: $r_{t-1}$ is~$\mathcal{F}_{t-1}$- measurable unlike~$r_t$ and this allows to obtain a null conditional expectation avoiding  in particular dependency issues between~$r_t$ and~$\theta_t$. 

Employing~\eqref{eq:zero-cond-expect-hat-y-hat-z} and using the same derivations as in~\eqref{eq:zero-inner-prod-lambda} leads to 
\begin{align}
    \label{eq:hat-e-t-decomp-expected-square}
    \mathbb{E}[\|\hat{e}_t\|^2] &= \mathbb{E}[\|(1-\eta_t) \hat{e}_{t-1} + (1-\eta_t) \hat{z}_t + \eta_t \hat{y}_t \|^2]\nonumber\\
    &= (1-\eta_t)^2 \mathbb{E}[\|\hat{e}_{t-1}\|^2] + \mathbb{E}[\| (1-\eta_t) \hat{z}_t + \eta_t \hat{y}_t\|^2] \nonumber\\
    &\leq  (1-\eta_t)^2 \mathbb{E}[\|\hat{e}_{t-1}\|^2] + 2 (1-\eta_t)^2 \mathbb{E}[\|\hat{z}_t\|^2] + 2 \eta_t^2 \mathbb{E}[\|\hat{y}_t\|^2]\,.
\end{align}

We now derive a bound for the term~$\mathbb{E}[\|\hat{z}_t\|^2]\,.$ Observe for this that 
\begin{align}
\label{eq:bound-z-hat-t}
\mathbb{E}[\|\hat{z}_t\|^2] &\leq \mathbb{E}[\|g(\tau_t,\theta_t,r_{t-1}) - w(\tau_t|\theta_{t-1},\theta_t)g(\tau_t,\theta_{t-1},r_{t-2}) \|^2]\nonumber\\
&= \mathbb{E}[\|g(\tau_t,\theta_t,r_{t-1}) - g(\tau_t,\theta_{t-1},r_{t-1})+ g(\tau_t,\theta_{t-1},r_{t-1})  - g(\tau_t,\theta_{t-1},r_{t-2}) \nonumber \\
& \qquad + g(\tau_t,\theta_{t-1},r_{t-2})  (1- w(\tau_t|\theta_{t-1},\theta_t))\|^2]\nonumber\\
&\leq 3 \mathbb{E}[\|g(\tau_t,\theta_t,r_{t-1}) - g(\tau_t,\theta_{t-1},r_{t-1})\|^2] + 3 \mathbb{E}[\|g(\tau_t,\theta_{t-1},r_{t-1})   - g(\tau_t,\theta_{t-1},r_{t-2})\|^2]\nonumber\\
&+ 3 \mathbb{E}[(1- w(\tau_t|\theta_{t-1},\theta_t))^2\|g(\tau_t,\theta_{t-1},r_{t-2})\|^2]\,.
\end{align}
Each term in this inequality is upper bounded separately in what follows. 

\noindent\textbf{Term 1: $\mathbb{E}[\|g(\tau_t,\theta_t,r_{t-1}) - g(\tau_t,\theta_{t-1},r_{t-1})\|^2]$.} Using Lemma~\ref{lem:lipschitz-pg-estimate}-\ref{lem:lipschitz-pg-estimate-theta}, we obtain 
\begin{align}
\|g(\tau_t,\theta_t,r_{t-1}) - g(\tau_t,\theta_{t-1},r_{t-1})\|
&\leq \frac{2 (l_{\psi}^2 + L_{\psi})}{(1-\gamma)^2} \|r_{t-1}\|_{\infty} \cdot \|\theta_t - \theta_{t-1}\|\nonumber\\
&= \frac{2 (l_{\psi}^2 + L_{\psi})}{(1-\gamma)^2} \|\nabla_{\lambda} F(\lambda_{t-1})\|_{\infty} \cdot \al_{t-1}\nonumber\\
&\leq \frac{2 (l_{\psi}^2 + L_{\psi}) l_{\lambda}}{(1-\gamma)^2}  \al_{t-1}\,,
\end{align}
where the last inequality stems from Assumption~\ref{hyp:smoothness-F}. We deduce from this the bound for the first term: 
\begin{equation}
\label{eq:term1-bound}
\mathbb{E}[\|g(\tau_t,\theta_t,r_{t-1}) - g(\tau_t,\theta_{t-1},r_{t-1})\|^2] \leq \frac{4 (l_{\psi}^2 + L_{\psi})^2 l_{\lambda}^2}{(1-\gamma)^4}  \al_{t-1}^2\,.
\end{equation}

\noindent\textbf{Term 2: $\mathbb{E}[\|g(\tau_t,\theta_{t-1},r_{t-1})   - g(\tau_t,\theta_{t-1},r_{t-2})\|^2]$.} Together  with Assumption~\ref{hyp:smoothness-F}, Lemma~\ref{lem:lipschitz-pg-estimate}-\ref{lem:lipschitz-pg-estimate-r} yields
\begin{align}
\label{eq:stochastic-pg-r-diff}
\|g(\tau_t,\theta_{t-1},r_{t-1}) - g(\tau_t,\theta_{t-1},r_{t-2})\| &\leq \frac{2 l_{\psi}}{(1-\gamma)^2} \|r_{t-1} - r_{t-2}\|_{\infty}\nonumber\\ 
&= \frac{2 l_{\psi}}{(1-\gamma)^2} \|\nabla_{\lambda} F(\lambda_{t-1}) - \nabla_{\lambda} F(\lambda_{t-2})\|_{\infty} \nonumber\\
&\leq \frac{2 l_{\psi} L_{\lambda}}{(1-\gamma)^2} \|\lambda_{t-1} - \lambda_{t-2}\|\,.
\end{align}
Invoking Lemma~\ref{lem:lambda-t-lambda-t-1}, we obtain from~\eqref{eq:stochastic-pg-r-diff} 
\begin{align}
\label{eq:term2-bound} 
\mathbb{E}[\|g(\tau_t,\theta_{t-1},r_{t-1})   - g(\tau_t,\theta_{t-1},r_{t-2})\|^2] 
&\leq  \frac{4 l_{\psi}^2 L_{\lambda}^2}{(1-\gamma)^4} \mathbb{E}[\|\lambda_{t-1} - \lambda_{t-2}\|^2] \nonumber\\
&\leq \frac{4 l_{\psi}^2 L_{\lambda}^2}{(1-\gamma)^4} 
\left( \frac{12 C  \eta_{t-1}^2}{ (1-\gamma)^2} 
+ \frac{6 C_w}{ (1-\gamma)^2} \al_{t-2}^2 \right) \nonumber \\
&\leq \frac{48 l_{\psi}^2 L_{\lambda}^2}{(1-\gamma)^6} 
\left(  C  \eta_{t-1}^2 
+  C_w \al_{t-2}^2 \right) \,.
\end{align}

\noindent\textbf{Term 3: $\mathbb{E}[(1- w(\tau_t|\theta_{t-1},\theta_t))^2\|g(\tau_t,\theta_{t-1},r_{t-2})\|^2]$.}

First, observe that
\begin{align}
\label{eq:stochastic-pg-bound}
\|g(\tau_t, \theta_{t-1}, r_{t-2})\| 
&\stackrel{(a)}{=} \left\| \sum_{t=0}^{H-1} \left( \sum_{h=t}^{H-1} \gamma^h r_{t-2}(s_h,a_h) \right) \nabla \log \pi_{\theta}(a_t|s_t) \right\|\nonumber\\
&\leq \sum_{t=0}^{H-1} \sum_{h=t}^{H-1} \gamma^h \|r_{t-2}\|_{\infty} \cdot \|\nabla \log \pi_{\theta}(a_t|s_t)\|\nonumber\\
&\stackrel{(b)}{\leq} l_{\lambda} \sum_{t=0}^{H-1} \sum_{h=t}^{H-1} \gamma^h \|\nabla \log \pi_{\theta}(a_t|s_t)\|\nonumber\\
&\stackrel{(c)}{\leq} 2 l_{\lambda}l_{\psi} \sum_{t=0}^{H-1} \sum_{h=t}^{H-1} \gamma^h\nonumber\\
&= 2 l_{\lambda}l_{\psi}  \sum_{h=0}^{H-1} \sum_{t=0}^h \gamma^h\nonumber\\
&\leq 2 l_{\lambda}l_{\psi}\sum_{h=0}^{H-1} (h+1) \gamma^h\nonumber\\
&\leq \frac{2 l_{\lambda}l_{\psi}}{(1-\gamma)^2}\,,
\end{align}
where (a)~follows from the expression of the stochastic policy gradient~\eqref{eq:pg-estimate}, (b)~stems from Assumption~\ref{hyp:smoothness-F} and (c)~is a consequence of Lemma~\ref{lem:smoothness-obj}-\ref{lem:smoothness-obj-i}.

Using Lemma~\ref{lem:variance-IS-weights-control} together with the previous bound yields 
\begin{align}
\label{eq:term3-bound}
\mathbb{E}[(1-w(\tau_t|\theta_{t-1},\theta_t))^2\|g(\tau_t,\theta_{t-1},r_{t-2})\|^2] 
&\leq \frac{4 l_{\lambda}^2  l_{\psi}^2}{(1-\gamma)^4} \mathbb{E}[(1-w(\tau_t|\theta_{t-1},\theta_t))^2]\nonumber\\
&= \frac{4 l_{\lambda}^2  l_{\psi}^2}{(1-\gamma)^4} \Var{[w(\tau_t|\theta_{t-1},\theta_t)]}\nonumber\\ 
&\leq \frac{4 l_{\lambda}^2  l_{\psi}^2 C_w}{(1-\gamma)^4} \al_{t-1}^2\,.
\end{align}

Collecting~\eqref{eq:term1-bound}, \eqref{eq:term2-bound} and~\eqref{eq:term3-bound} in~\eqref{eq:bound-z-hat-t}, we obtain 
\begin{align}
\label{eq:bound-hat-z-t}
\mathbb{E}[\|\hat{z}_t\|^2] &\leq 
\frac{12 l_{\lambda}^2[(l_{\psi}^2 + L_{\psi})^2 + C_w l_{\psi}^2]}{(1-\gamma)^4} \al_{t-1}^2 
+ \frac{144 l_{\psi}^2 L_{\lambda}^2}{(1-\gamma)^6} 
\left(  C  \eta_{t-1}^2 
+  C_w \al_{t-2}^2 \right)\, \nonumber \\
&\leq 
 \frac{144 C l_{\psi}^2 L_{\lambda}^2}{(1-\gamma)^6} 
  \eta_{t-1}^2 + \rb{ \frac{12 l_{\lambda}^2[(l_{\psi}^2 + L_{\psi})^2 + C_w l_{\psi}^2]}{(1-\gamma)^4}  + \frac{144 C_w l_{\psi}^2 L_{\lambda}^2}{(1-\gamma)^6} } \al_{t-2}^2 
\,.
\end{align}

We now bound the term~$\mathbb{E}[\|\hat{y}_t\|^2]$ in~\eqref{eq:hat-e-t-decomp-expected-square}. First, recall from~\eqref{def:hat-y-t} that $\hat{y}_t =  g(\tau_t,\theta_t,r_{t-1}) -  [\nabla_{\theta} \lambda_H(\theta_t)]^T r_{t-1}$\,. Then, %
with probability one, 
\begin{align}
\label{eq:bound1-hat-y-t}
\|\hat{y}_t\| &\leq \|g(\tau_t, \theta_t, r_{t-1})\| + \|[\nabla_{\theta}\lambda_H(\theta_t)]^{T}  r_{t-1}\| \nonumber\\
              &\stackrel{(a)}{\leq} \frac{2 l_{\lambda} l_{\psi}}{(1-\gamma)^2} + 
              \|[\nabla_{\theta}\lambda_H(\theta_t)]^{T}  r_{t-1}\|\nonumber\\
              &\stackrel{(b)}{\leq} \frac{4 l_{\lambda} l_{\psi}}{(1-\gamma)^2}\,,
\end{align}
where (a) stems from the same bound as in~\eqref{eq:stochastic-pg-bound} and (b) also follows from a similar bound to~\eqref{eq:stochastic-pg-bound}. Indeed, notice using~\eqref{eq:expected-reinforce} that
\begin{align}
\|[\nabla_{\theta}\lambda_H(\theta_t)]^{T}  r_{t-1}\| 
&= \left\| \mathbb{E}\left[  \sum_{t' = 0}^{H-1} \gamma^{t'} r_{t-1}(s_t', a_t') \left( \sum_{h=0}^{t'} \nabla\log \pi_{\theta}(a_h|s_h)  \right) \right] \right\|\nonumber\\
&\leq \mathbb{E}\left[ \sum_{t'=0}^{H-1} \gamma^{t'} \|r_{t-1}\|_{\infty} \sum_{h=0}^{t'} \|\nabla \log \pi_{\theta}(a_h|s_h)\| \right]\nonumber\\
&\stackrel{(a)}{\leq} 2 l_{\lambda} l_{\psi} \sum_{t'= 0}^{H-1} (t'+ 1) \gamma^{t'}\nonumber\\
&\leq \frac{2 l_{\lambda} l_{\psi}}{(1-\gamma)^2}\,,
\end{align}
where again (a) stems from Assumption~\ref{hyp:smoothness-F} and Lemma~\ref{lem:smoothness-obj}-\ref{lem:smoothness-obj-i}. 

We conclude from~\eqref{eq:hat-e-t-decomp-expected-square}, \eqref{eq:bound-hat-z-t} and~\eqref{eq:bound1-hat-y-t} that
\begin{equation*} \mathbb{E}[\|\hat{e}_t\|^2] \leq (1-\eta_t)^2 \mathbb{E}[\|\hat{e}_{t-1}\|^2] 
+ \rb{ \frac{288 C l_{\psi}^2 L_{\lambda}^2}{(1-\gamma)^6} 
+ \frac{32 l_{\lambda}^2 l_{\psi}^2}{(1-\gamma)^4} } \eta_{t-1}^2 + 
\rb{ \frac{12 l_{\lambda}^2[(l_{\psi}^2 + L_{\psi})^2 + C_w l_{\psi}^2]}{(1-\gamma)^4}  + \frac{144 C_w l_{\psi}^2 L_{\lambda}^2}{(1-\gamma)^6} } \al_{t-2}^2  \,,
\end{equation*}
where~$C_w =  H ((8H+2) l_{\psi}^2 + 2 L_{\psi}) (W+1)\,$ as defined in Lemma~\ref{lem:variance-IS-weights-control}.

\noindent\textbf{Proof of~\eqref{eq:recursive-hat-e-t_solved}:} In order to derive~\eqref{eq:recursive-hat-e-t_solved}, we apply Lemma~\ref{le:aux_rec0} with $ \eta_t = \fr{2}{t+1}$, $\beta_t =  C_3 \eta_{t-1}^2 + C_4 \al_{t-2}^2$. Using $\mathbb{E}[\|\hat{e}_{0}\|^2] \leq\hat{E}^2 $, we derive

\begin{eqnarray}
\mathbb{E}[\|\hat{e}_t\|] &\leq& \rb{ \mathbb{E}[\|\hat{e}_t\|^2] }^{\fr{1}{2}} \leq \rb{ \fr{4 \hat{E}^2 }{(t+1)^2} + 4 C_3 \eta_t^2 \cdot t + C_4 \al_{t-2}^2  \cdot t }^{\nfr{1}{2}} \notag \\
&\leq & \fr{2 \hat{E} }{t+1} + 2 C_3^{\nfr{1}{2}} \eta_{t} \cdot t^{\nfr{1}{2}} + C_4^{\nfr{1}{2}} \al_{t-2}  \cdot t^{\nfr{1}{2}}   \,.
\end{eqnarray}

\noindent\textbf{Proof of~\eqref{eq:recursive-hat-e-t_solved_2}:} Let $\eta_t = \rb{ \fr{2}{t+1} }^{q}$ for some $q \in (0,1)$. In order to derive~\eqref{eq:recursive-hat-e-t_solved_2}, we unroll the recursion from $t = 1$ to $t = t' \leq T$. Denoting $\beta_t =  C_3 \eta_{t-1}^2 + C_4 \al_{t-2}^2 $, we derive
\begin{eqnarray}\label{eq:et-prime-tilde}
\Exp{\|\hat{e}_{t'}\|^2} 
&\leq& \rb{ \prod_{\tau = 1}^{t'} (1-\eta_{\tau}) } \Exp{\|\hat{e}_{0}\|^2} + \sum_{t = 1}^{t'} \beta_{t} \prod_{\tau = t+1 }^{t'} (1 - \eta_{\tau}) \notag \\
&\leq& \hat{E}^2 \eta_{t'+1} + C \beta_{t'+1} \eta_{t'+1}^{-1}\notag \\
&\leq& \hat{E}^2 \eta_{t'+1} + C C_3 \eta_{t'}^2 \eta_{t'+1}^{-1} + C C_4 \al_{t'-1}^2 \eta_{t'+1}^{-1} \notag \\
&\leq& \hat{E}^2 \eta_{t'+1} + 2 C C_3 \eta_{t'+1}  + C C_4 \al_{t'-1}^2 \eta_{t'+1}^{-1}.
\, 
\,,
\end{eqnarray}
where we used the results of Lemmas~\ref{le:prod_bound}-\ref{le:sum_prod_bound1}
 and $ \Exp{\sqnorm{\hat{e}_0}}  \leq \hat{E}^2$ with $\hat{E} = \frac{4 l_{\lambda} l_{\psi}}{(1-\gamma)^2} $, which can be derived similarly to~\eqref{eq:bound1-hat-y-t}.

\end{proof}

In the next lemma, we derive an estimate of the average expected error~$\mathbb{E}[\|\hat{e}_t\|]$ from the recursion we have just established in Lemma~\ref{lem:recursive-hat-e-t}. 

\begin{lemma}\label{lem:recursive-hat-e-t_sum}
    Suppose Assumptions~\ref{hyp:policy-param} and~\ref{hyp:smoothness-F} hold. 
    Let~$T\geq 1$ be an integer, let~$\alpha_0 > 0$ and set~$\eta_t = \rb{\fr{2}{t+1}}^{\nfr{2}{3}}$, $\al_t = \fr{\al_0}{T^{\nfr{2}{3}}}$ for every integer~$t$. Then 
    \begin{equation}
        \fr{1}{T } \sum_{t=1}^{T} \Exp{\norm{ \hat{e}_t} } \leq  \fr{ C \rb{  \hat{E}+   C_3^{\nfr{1}{2}} +  C_4^{\nfr{1}{2}} \al_0 } }{T^{\nfr{1}{3}} } ,
    \end{equation}
where~$\hat{E} = \frac{4 l_{\lambda} l_{\psi}}{(1-\gamma)^2} $, $C_3 =  \frac{288 C l_{\psi}^2 L_{\lambda}^2}{(1-\gamma)^6} 
+ \frac{32 l_{\lambda}^2 l_{\psi}^2}{(1-\gamma)^4} $, $C_4 =  \frac{12 l_{\lambda}^2[(l_{\psi}^2 + L_{\psi})^2 + C_w l_{\psi}^2]}{(1-\gamma)^4}  + \frac{144 C_w l_{\psi}^2 L_{\lambda}^2}{(1-\gamma)^6}  $, $C_w =  H ((8H+2) l_{\psi}^2 + 2 L_{\psi}) (W+1)\,$ as defined in Lemma~\ref{lem:variance-IS-weights-control}, and $C > 1$ is a numerical constant.
\end{lemma}

\begin{proof}

Summing up inequality \eqref{eq:recursive-hat-e-t_solved_2} from Lemma~\ref{lem:recursive-hat-e-t} from $t = 1$ to $t = T$ and choosing $\al_t = \al = \fr{\al_0}{T^{\nfr{2}{3}}}$, we obtain
\begin{eqnarray}
\fr{1}{T}\sum_{t=1}^{T} \Exp{\|\hat{e}_t\|} &\leq&  \fr{1}{T} \sum_{t=1}^{T} \rb{\Exp{\|\hat{e}_t\|^2} }^{\nfr{1}{2}} \notag \\
&\leq& \rb{ \fr{1}{T} \sum_{t=1}^{T} \Exp{\|\hat{e}_t\|^2}   }^{\nfr{1}{2}}  \notag \\
&\leq& \rb{ \fr{1}{T} \sum_{t=1}^{T} \hat{E}^2 \eta_{t+1} + 2 C  C_3  \eta_{t+1} + C C_4 \al^2 \eta_{t+1}^{-1}   }^{\nfr{1}{2}}  \notag \\
&\overset{(i)}{\leq}& \rb{  3 (\hat{E}^2 + 2C C_3) \eta_{T-1}  + C  C_3 \al^2 \eta_{T+1}^{-1}  }^{\nfr{1}{2}}  \notag \\
&\leq& \rb{   \fr{12 (\hat{E}^2 +  C C_3)  }{T^{\nfr{2}{3}}}  +\fr{2 C C_4 \al_0^2 }{T^{\nfr{2}{3}}}  }^{\nfr{1}{2}} \notag \\
&\leq&   \fr{ 4 C^{\nfr{1}{2}} \rb{ \hat{E} +   C_3^{\nfr{1}{2}} +  C_4^{\nfr{1}{2}} \al_0 } }{T^{\nfr{1}{3}} } \,,   \notag
\end{eqnarray}
where $(i)$ holds by \eqref{eq:sum_by_int_bound}.
\end{proof}
We are now ready to state the main lemma controlling the error sequence~$(e_t)$ in expectation as defined in~\eqref{eq:e-t-def-gen-utilities}. 

\begin{lemma}
\label{lem:overall-error-sum}
Suppose Assumptions~\ref{hyp:policy-param} and~\ref{hyp:smoothness-F} hold. 
Let~$T\geq 1$ be an integer, let~$\alpha_0 > 0$ and set~$\eta_t = \rb{\fr{2}{t+1}}^{\nfr{2}{3}}$, $\al_t = \fr{\al_0}{T^{\nfr{2}{3}}}$ for every integer~$t$. Then 
\begin{equation}
\fr{1}{T} \sum_{t=1}^{T}\mathbb{E}[\|e_t\|] \leq  \fr{ C \rb{ \hat{E} +  C_3^{\nfr{1}{2}} +  C_4^{\nfr{1}{2}} \al_0 } }{T^{\nfr{1}{3}} } +  \frac{ C C_1 \rb{ 1 + C_w^{\nfr{1}{2}} \al_0 }  }{(1-\gamma)} \fr{1}{T^{\nfr{1}{3}} }   +  \fr{C_2 \al_0 }{T^{\nfr{2}{3}}} ,
\end{equation}
where~$C_1 = \frac{2 L_{\lambda}^2 l_{\psi}}{(1-\gamma)^2}\,$ and~$C_2 = \frac{2 L_{\lambda} L_{\lambda,\infty} l_{\psi}}{(1-\gamma)^2}\,$, $C_3 =  \frac{288 C l_{\psi}^2 L_{\lambda}^2}{(1-\gamma)^6} 
+ \frac{32 l_{\lambda}^2 l_{\psi}^2}{(1-\gamma)^4} $, $C_4 =  \frac{12 l_{\lambda}^2[(l_{\psi}^2 + L_{\psi})^2 + C_w l_{\psi}^2]}{(1-\gamma)^4}  + \frac{144 C_w l_{\psi}^2 L_{\lambda}^2}{(1-\gamma)^6}$, $C$ is a numerical constant and $C_w =  H ((8H+2) l_{\psi}^2 + 2 L_{\psi}) (W+1)\,$ as defined in Lemma~\ref{lem:variance-IS-weights-control}.
\end{lemma}

\begin{proof}
By Lemma~\ref{lem:recursive-tilde-e-t} and~\ref{lem:recursive-hat-e-t_sum} we have the bounds
\begin{equation}
\fr{1}{T } \sum_{t=1}^{T} \Exp{\norm{ \hat{e}_t} } \leq  \fr{ C \rb{ \hat{E} +  C_3^{\nfr{1}{2}} +  C_4^{\nfr{1}{2}} \al_0 } }{T^{\nfr{1}{3}} } , \qquad \fr{1}{T} \sum_{t=1}^{T} \mathbb{E}[\|\tilde{e}_t\|] \leq  \frac{ C \rb{ 1 + C_w^{\nfr{1}{2}} \al_0 } }{(1-\gamma)} \fr{1}{T^{\nfr{1}{3}} }\,.
\end{equation}

Summing up the result of Lemma~\ref{lem:e-t-bound} from $ t = 1$ to $t = T$ and using the above bounds, we obtain
\begin{eqnarray}
\fr{1}{T}\sum_{t=1}^{T} \mathbb{E}[\|e_t\|] &\leq& \fr{1}{T}\sum_{t=1}^{T}  \mathbb{E}[\|\hat{e}_t\|] +  \fr{C_1}{T}\sum_{t=1}^{T} \mathbb{E}[\|\tilde{e}_{t-1}\|]  +  \fr{C_2}{T}\sum_{t=1}^{T}  \al_{t-1}\, \notag \\
&\leq& \fr{ C \rb{ \hat{E} +  C_3^{\nfr{1}{2}} +  C_4^{\nfr{1}{2}} \al_0 } }{T^{\nfr{1}{3}} } +  \frac{ C C_1 \rb{ 1 + C_w^{\nfr{1}{2}} \al_0 }  }{(1-\gamma)} \fr{1}{T^{\nfr{1}{3}} }   +  \fr{C_2 \al_0 }{T^{\nfr{2}{3}}}  \,.
\end{eqnarray}
\end{proof}

\noindent\textbf{End of the proof of Theorem~\ref{thm:fos-gen-ut-nvrpg}.} We conclude the proof of Theorem~\ref{thm:fos-gen-ut-nvrpg} which is first recalled in the following for the convenience of the reader. 

\begin{theorem}\label{thm:end-of-proof}
    Let Assumptions~\ref{hyp:policy-param} and~\ref{hyp:smoothness-F} hold. Let~$\alpha_0 > 0$ and consider an integer~$T \geq 1$. 
    Set~$\al_t = \fr{\al_0}{T^{\nfr{2}{3}}}$, $\eta_t = \left( \fr{2}{t+1} \right)^{\nfr{2}{3}}$ and~$H = \rb{1-\g}^{-1}{\log(T + 1)}$. Let $\bar{\theta}_T$ be sampled from the iterates~$\{\theta_1, \cdots, \theta_T\}$ of Algorithm~\ref{algo-gen-ut} uniformly at random. %
    Then, we have
    \begin{equation}
    \Exp{ \norm{ \nabla_{\theta}F(\lambda(\bar{\theta}_T)) } } \leq \cO\rb{ \fr{1 + (1-\gamma)^3 \Delta \al_0^{-1} +  (1-\gamma)^{-1} \al_0 }{ (1-\gamma)^{3}  T^{\nfr{1}{3}} } } .
    \end{equation}
    If moreover $\al_0 = (1-\gamma)^2 \sqrt{\Delta}$, then $\Exp{ \norm{ \nabla_{\theta}F(\lambda(\bar{\theta}_T)) } } \leq \cO\rb{\fr{ 1 + (1-\gamma) \sqrt{\Delta} }{ (1-\gamma)^{3} T^{\nfr{1}{3}}}} \,. $
\end{theorem}

\begin{proof}
By Lemma~\ref{lem:ascent-like-lemma-gen-ut}, we have for every integer~$t$, 
    \begin{equation}
F(\lambda(\theta_{t+1}))  \geq F(\lambda(\theta_t)) + \frac{\al_t}{3} \|\nabla_{\theta} F(\lambda(\theta_t))\| - 2 \al_t \|e_t\| - \frac{4}{3} D_{\lambda} \gamma^H \al_t 
- \frac{L_{\theta}}{2} \al_t^2\,.
\end{equation}

    Setting constant step-size $\al_t = \al = \fr{\al_0}{T^{\nfr{2}{3}}}$, taking expectation, telescoping and rearranging, we get

    \begin{eqnarray}
    \fr{1}{T} \sum_{t=1}^{T} \Exp{\norm{\nabla_{\theta} F(\lambda(\theta_t))}} 
    &\leq & \fr{3 (F^{\star} - F(\lambda(\theta_1)))}{\al T}  +  \fr{6}{T} \sum_{t=1}^{T}\Exp{\norm{e_t}}  +  \fr{3 L_{\theta} \al}{2}  + 4 D_g \g^H \notag \\
    &\leq & \fr{3 (F^{\star} - F(\lambda(\theta_1)))}{\al_0 T^{\nfr{1}{3}}}  +  \fr{ 6 C \rb{ \hat{E} + C_3^{\nfr{1}{2}} +  C_4^{\nfr{1}{2}} \al_0 } }{T^{\nfr{1}{3}} } +  \frac{ 6 C C_1  \rb{ 1 + C_w^{\nfr{1}{2}} \al_0 } }{(1-\gamma)} \fr{1}{T^{\nfr{1}{3}} }   \notag \\
    && \qquad +  \fr{C_2 \al_0 }{T^{\nfr{2}{3}}}   +  \fr{3 L_{\theta} \al_0}{2 T^{\nfr{2}{3}} }   + 4 D_g \g^H \notag \\
    & = & \cO\rb{ \fr{(1-\gamma)^{-3} + \Delta \al_0^{-1} + \al_0 (1-\gamma)^{-4}}{ T^{\nfr{1}{3}} } }  ,
\end{eqnarray}

where we set $H = (1-\g)^{-1} \log(T)$. 
Setting $\al_0 = (1-\gamma)^2 \sqrt{\Delta}$, we obtain the desired result. 
\end{proof}

\subsection{Proof of Corollary~\ref{cor:fos-standard-RL-softmax} (Cumulative reward setting)}\label{subsec:soft-max-cumsum-fos}

For this particular case, we redefine the error sequence~$(e_t)$ by overloading the notation since it plays a similar role. 
Define the error sequence~$(e_t)$ for every integer~$t$ as follows: 
\begin{equation}
\label{eq:e-t-def-cumul-reward}
e_t \eqdef d_t - \nabla J_H(\theta_t)\,,
\end{equation}
where the truncated cumulative reward~$J_H(\theta)$ is defined as follows for any policy parameter~$\theta \in \R^d$: 
\begin{equation}
J_H(\theta) = \mathbb{E}\left[\sum_{t=0}^{H-1} \gamma^t r(s_t,a_t)\right]\,.
\end{equation}

We start by stating a complete version of Corollary~\ref{cor:fos-standard-RL-softmax} which we shall prove in this section.  

\begin{corollary}[\textbf{FOS convergence of \algname{N-VR-PG}}]
    \label{thm:N-VR-PG_stat}
    Let Assumptions~\ref{hyp:policy-param} and~\ref{hyp:smoothness-F} hold. 
    Let~$\alpha_0 > 0$ and let~$T$ be an integer larger than~$1$. 
    Set~$\al_t = \fr{\al_0}{T^{\nfr{2}{3}}}$, $\eta_t = \rb{ \fr{2}{t+1} }^{\nfr{2}{3}}$ and~$H = \rb{1-\g}^{-1}{\log(T + 1)}$. Let~$\bar{\theta}_T$ be sampled from the iterates~$\{\theta_1, \cdots, \theta_T\}$ of \algname{N-VR-PG} (Algorithm~\ref{alg:N-VR-PG}) uniformly at random. Then we have %
    \begin{equation}
    \Exp{\norm{\nabla J(\bar{\theta}_T)}}
     \leq  \cO\rb{ \fr{ J^* - J(\theta_1) }{\al_0  T^{\nfr{1}{3}} } +  \fr{ V + (L_g + G C_{w}^{\nfr{1}{2}}) \al_0  }{T^{\nfr{1}{3}}}   } ,
    \end{equation}
    where $V$, $L_g$, $G$, and $C_w$ are defined in Lemma~\ref{lem:variance-IS-weights-control}, \ref{lem:lipschitz-pg-estimate}, and \ref{lem:for-cum-sum}. Moreover, if we set $ \al_0 = 1-\gamma$, then
    $$\Exp{ \norm{\nabla J(\bar{\theta}_T)} } \leq \cO\rb{   \fr{ 1   }{(1-\gamma)^{2} T^{\nfr{1}{3}}}   }.$$
\end{corollary}

The proof of this result follows the same lines as the proof of Theorem~\ref{thm:fos-gen-ut-nvrpg} which addresses the more general setting of general utilities. In the special case of cumulative rewards, recall that the estimation of the state-action occupancy measure is not required. In the following, we provide for clarity the intermediate results required to prove our result, mirroring the proof of the more general result of Theorem~\ref{thm:fos-gen-ut-nvrpg}.  

In order to derive the improved dependence on the $(1-\gamma)$ factor in the final rate compared to Theorem~\ref{thm:fos-gen-ut-nvrpg}, we will apply the following results from \citet[Lemma 4.2, (68) and Lemma 4.4, (19)]{Vanilla_PL_Yuan_21} and \citep[Proposition 4.2 (1) and (3)]{xu-et-al20iclr}, which offer a tighter dependence on $1-\gamma$ in the standard cumulative reward setting. We use the notation~$g(\tau, \theta)$ instead of~$g(\tau, \theta, r)$ in this simpler standard RL setting.  

\begin{lemma}\label{lem:for-cum-sum}
    Let Assumption~\ref{hyp:policy-param} hold true and 
let~$\tau = \{s_0, a_0, \cdots, s_{H-1}, a_{H-1}\}$ be an arbitrary trajectory of length~$H$. Then the following statements hold: 
    \begin{enumerate}[label=(\roman*)]
        \item \label{lem:smoothness-cum-sum} The objective function~$\theta \mapsto J(\theta)$ is $L_{\theta}$-smooth with $L_{\theta} \eqdef  \frac{ 2 \norm{r}_{\infty} ( L_{\psi} + 3 l_{\psi}^2 )}{(1-\gamma)^2}\,.$
        \item \label{lem:lipschitz-pg-estimate-theta-cum-sum} For all $ \theta_1, \theta_2 \in \R^d, \|g(\tau, \theta_1) - g(\tau, \theta_2) \| \leq L_{g} \|\theta_1 - \theta_2\| \,$  where~$L_{g} \eqdef \frac{2 (l_{\psi}^2 + L_{\psi}) \|r\|_{\infty}}{(1-\gamma)^2}\,,$
        \item \label{lem:BV-cum-sum} For all $ \theta \in \R^d,\,  \Exp{\sqnorm{g(\tau,\theta) - \nabla_{\theta} J_H(\theta) } } \leq V^2$ where~$V \eqdef \frac{2  l_{\psi} \|r\|_{\infty}}{(1-\gamma)^{\nfr{3}{2}}}$. 
        \item \label{lem:BG-cum-sum} For all $ \theta \in \R^d,\, \|g(\tau, \theta )\| \leq G$ where~$G \eqdef \frac{2 l_{\lambda}l_{\psi}}{(1-\gamma)^2}\,.$
    \end{enumerate}
\end{lemma}

We start with the next lemma which corresponds exactly to Lemma~\ref{lem:ascent-like-lemma-gen-ut}. 
\begin{lemma}
\label{lem:ascent-like-cumul-rewards}
Let Assumption~\ref{hyp:policy-param} hold true. Then, the sequence~$(\theta_t)$ generated by Algorithm~\ref{alg:N-VR-PG} and the sequence~$(e_t)$ defined in~\eqref{eq:e-t-def-cumul-reward} satisfy for every integer~$t \geq 0$, 
\begin{equation}
J(\theta_{t+1})  \geq J(\theta_t) + \frac{\al_t}{3} \|\nabla J(\theta_t)\| - 2 \al_t \|e_t\| - \frac{4}{3} D_g \gamma^H \al_t - \frac{L_{\theta}}{2} \al_t^2\,.
\end{equation}
\end{lemma}

\begin{proof}
The proof is identical to the proof of Lemma~\ref{lem:ascent-like-lemma-gen-ut} upon noticing that Assumption~\ref{hyp:smoothness-F} is not needed here since smoothness of the objective function~$J$ is a standard result in the RL literature following from Assumption~\ref{hyp:policy-param} (see for e.g., Lemma~4.4 in \citet{Vanilla_PL_Yuan_21}).  
\end{proof}

Then next lemma is similar to Lemma~\ref{lem:recursive-hat-e-t} and controls the error~$e_t$ in Lemma~\ref{lem:ascent-like-cumul-rewards}. We provide here a complete statement and proof of this result for clarity and completeness since the corresponding lemma in the more general case is more involved. Indeed, the latter result involves an additional error due to the occupancy measure estimation which is not required in our present setting. 

\begin{lemma}\label{lem:et_bound_cumul}
Under Assumption~\ref{hyp:policy-param}, we have for every integer~$t \geq 1$ 
\begin{equation}\label{eq:recursive-e-t_cum_case}
\mathbb{E}[\|e_t\|^2] \leq (1-\eta_t)^2 \mathbb{E}[\|e_{t-1}\|^2] + 2 V^2 \eta_t^2 + 4(L_g^2 + G^2 C_w)(1-\eta_t)^2 \al_{t-1}^2\,.
\end{equation}
where $V$, $G$, $L_g$ are constants defined in Lemma~\ref{lem:lipschitz-pg-estimate} and~\ref{lem:for-cum-sum}. If in addition 
\begin{enumerate}[label=(\roman*)]
    \item $\eta_t = \fr{2}{t+1}$, then for all $t \geq 1$, we have 
    \begin{equation}
\mathbb{E}[\|e_t\|] \leq   4 V \eta_t \cdot t^{\nfr{1}{2}} + 2(L_g + G C_w^{\nfr{1}{2}} )  \al_{t-1} \cdot t^{\nfr{1}{2}}\,. 
\end{equation}

    \item $\eta_t = \rb{ \fr{2}{t+1} }^{\nfr{2}{3}}$, then for all integers $T \geq 1$, if~$\al_t = \fr{\al_0}{T^{\nfr{2}{3}}}$ for some~$\alpha_0 > 0$, we have
        \begin{equation}\label{eq:recursive-e-t_sum_cum_case}
        \sum_{t=1}^{T} \mathbb{E}[\|e_t\|] \leq \frac{ C \rb{ V + \rb{ L_g + G C_w^{\nfr{1}{2}} } \al_0 } }{T^{\nfr{1}{3}} }\,,
        \end{equation}
        where~$C  > 0$ is an absolute numerical constant.
\end{enumerate}
\end{lemma}

\begin{proof}
Using the update rule of the sequence~$(d_t)$ and recalling the definition of the error~$e_t = d_t - \nabla J_H(\theta_t)$, we have
\begin{align*}
e_t &= d_t  - \nabla J_H(\theta_t)\\
    &= (1-\eta_t) (d_{t-1} + v_t) + \eta_t g(\tau_t, \theta_t) - \nabla J_H(\theta_t)\\
    &= (1-\eta_t) (e_{t-1} + \nabla J_H(\theta_{t-1}) + v_t) + \eta_t g(\tau_t, \theta_t) - \nabla J_H(\theta_t)\\
    &= (1-\eta_t) e_{t-1} + \eta_t (g(\tau_t, \theta_t)- \nabla J_H(\theta_t)) + (1-\eta_t) (v_t - (\nabla J_H(\theta_t) -\nabla J_H(\theta_{t-1})))\,. 
\end{align*}
Introducing additional notation for convenience: 
\begin{align}
    y_t &\eqdef g(\tau_t,\theta_t) - \nabla J_H(\theta_t)\,,\\
    z_t &\eqdef v_t - (\nabla J_H(\theta_t) -\nabla J_H(\theta_{t-1}))\,,
\end{align}
we obtain the following useful decomposition:
\begin{equation}
        e_t = (1-\eta_t) e_{t-1} + \eta_t y_t + (1-\eta_t)z_t\,.
\end{equation}
Defining $\mathcal{F}_t$ as the $\sigma$-algebra generated by all the random variables until time~$t$, we observe that~$\mathbb{E}[y_t|\mathcal{F}_{t-1}] = \mathbb{E}[z_t|\mathcal{F}_{t-1}] = 0\,.$ As a consequence, we have
\begin{align}
\label{eq:recursive-et-control}
\mathbb{E}[\|e_t\|^2] &= (1-\eta_t)^2 \mathbb{E}[\|e_{t-1}\|^2] + \mathbb{E}[\|\eta_t y_t + (1-\eta_t) z_t\|^2] \nonumber\\
               &\leq (1-\eta_t)^2 \mathbb{E}[\|e_{t-1}\|^2] + 2 \eta_t^2 \mathbb{E}[\|y_t\|^2] + 2 (1-\eta_t)^2 \mathbb{E}[\|z_t\|^2]\,.
\end{align}
We now control each one of the last two terms in the previous inequality. 
For the first term, we have by Lemma~\ref{lem:for-cum-sum}-\ref{lem:BV-cum-sum}
\begin{equation}
\label{eq:var-control}
\mathbb{E}[\|y_t\|^2] \leq V^2\,.
\end{equation}
Concerning the second remaining term, using Lemma~\ref{lem:lipschitz-pg-estimate}-\ref{lem:lipschitz-pg-estimate-theta} and Lemma~\ref{lem:for-cum-sum}-\ref{lem:BG-cum-sum}, we write 
\begin{align}
\label{eq:is-term-control}
\mathbb{E}[\|z_t\|^2] &\leq \mathbb{E}[\|v_t\|^2]\nonumber\\
               &= \mathbb{E}[\|g(\tau_t, \theta_t) - w(\tau_t|\theta_{t-1},\theta_t)g(\tau_t, \theta_{t-1})\|^2]\nonumber\\
               &= \mathbb{E}[\|g(\tau_t, \theta_t) - g(\tau_t, \theta_{t-1}) + g(\tau_t, \theta_{t-1}) (1 - w(\tau_t|\theta_{t-1},\theta_t))\|^2]\nonumber\\
               &\leq 2 L_g^2 \mathbb{E}[\|\theta_t - \theta_{t-1}\|^2] + 2 G^2 \mathbb{E}[(1-w(\tau_t|\theta_{t-1},\theta_t))^2]\nonumber\\
               &= 2 L_g^2 \mathbb{E}[\|\theta_t - \theta_{t-1}\|^2] + 2 G^2 \Var(w(\tau_t|\theta_{t-1},\theta_t))\nonumber\\
               &\leq 2(L_g^2 + G^2 C_w) \mathbb{E}[\|\theta_t - \theta_{t-1}\|^2]\nonumber\\
               &= 2(L_g^2 + G^2 C_w) \al_{t-1}^2\,.
\end{align}

Combining~\eqref{eq:recursive-et-control} with~\eqref{eq:var-control} and~\eqref{eq:is-term-control} concludes the first part of the proof. 

Applying Lemma~\ref{le:aux_rec0} with $ \eta_t = \fr{2}{t+1}$, $\beta_t = 2 V^2 \eta_t^2 + 4(L_g^2 + G^2 C_w)(1-\eta_t)^2 \al_{t-1}^2 $ and using $\mathbb{E}[\|e_{0}\|^2] \leq V^2$, we get
\begin{eqnarray}
\mathbb{E}[\|e_t\|] &\leq& \rb{ \mathbb{E}[\|e_t\|^2] }^{\fr{1}{2}} \leq \rb{ \fr{4 V^2 }{(t+1)^2} + 2 V^2 \eta_t^2 \cdot t + 4(L_g^2 + G^2 C_w)  \al_{t-1}^2 \cdot t }^{\nfr{1}{2}} \notag \\
&\leq& \fr{2 V }{(t+1)} + 2 V \eta_t \cdot t^{\nfr{1}{2}} + 2(L_g + G C_w^{\nfr{1}{2}} )  \al_{t-1} \cdot t^{\nfr{1}{2}} \notag \\
&\leq&  4 V \eta_t \cdot t^{\nfr{1}{2}} + 2(L_g + G C_w^{\nfr{1}{2}} )  \al_{t-1} \cdot t^{\nfr{1}{2}} \,.
\end{eqnarray}

In order to derive \eqref{eq:recursive-e-t_sum_cum_case}, we unroll the recursion \eqref{eq:recursive-e-t_cum_case} from $t = 1$ to $t = t'$, where $t' \leq T$. Denoting $\beta_t =  2 V^2 \eta_t^2 + 4(L_g^2 + G^2 C_w)(1-\eta_t)^2 \al_{t-1}^2 $ , we have
\begin{eqnarray}
\Exp{\|{e}_{t'}\|^2} &\leq& \prod_{\tau = 1}^{t'} (1-\eta_{\tau}) \Exp{\|{e}_{0}\|^2} + \sum_{t = 0}^{t'} \beta_{t} \prod_{\tau = t+1 }^{t'} (1 - \eta_{\tau}) \notag \\
&\leq& V^2 \eta_{t'+1}  + C \beta_{t+1} \eta_{t'+1}^{-1}
\,,
\end{eqnarray}
where we used $\Exp{ \|e_{0}\|^2 }  \leq V^2$ and the result of Lemma~\ref{le:sum_prod_bound1} with $C > 0$ being a numerical constant. Finally, summing up the above inequality from $t' = 1$ to $t' = T$ and choosing $\al_t = \al = \fr{\al_0}{T^{\nfr{2}{3}}}$, we obtain 
\begin{eqnarray}
\fr{1}{T}\sum_{t=1}^{T} \Exp{\|{e}_t\|} &\leq&  \fr{1}{T} \sum_{t=1}^{T} \rb{\Exp{\|{e}_t\|^2} }^{\nfr{1}{2}} \notag \\
&\leq& \rb{ \fr{1}{T} \sum_{t=1}^{T} \Exp{\|{e}_t\|^2}   }^{\nfr{1}{2}}  \notag \\
&\leq& \rb{ \fr{1}{T} \sum_{t=1}^{T}  V^2 \eta_{T+1} + 2 C V^2 \eta_{t+1} + 4 C (L_g^2 + G^2 C_w) \al^2  \eta_{t+1}^{-1}  }^{\nfr{1}{2}}  \notag \\
&\overset{(i)}{\leq} & \rb{ 3 V^2 \eta_{T-1} +  6 C V^2 \eta_{T-1} + 4 C (L_g^2 + G^2 C_w) \fr{\al^2}{\eta_{T+1}}  }^{\nfr{1}{2}}  \notag \\
&\leq& \rb{  9 C V^2 \eta_T + 6 C (L_g^2 + G^2 C_w) \fr{\al_{0}^2}{\eta_T} \fr{1}{T^{\nfr{4}{3}}}  }^{\nfr{1}{2}}  \notag \\
&\leq&  \frac{ 4 C^{\nfr{1}{2}}\rb{ V + \rb{ L_g + G C_w^{\nfr{1}{2}} } \al_0 } }{T^{\nfr{1}{3}} }\,.    \notag 
\end{eqnarray}
where in $(i)$ we used \eqref{eq:sum_by_int_bound}.
\end{proof}

\noindent\textbf{End of Proof of Corollary~\ref{cor:fos-standard-RL-softmax}.}
The last steps of the proof are standard. Taking expectation on both sides of the result of Lemma~\ref{lem:ascent-like-cumul-rewards}, telescoping and rearranging, we have for every integer~$T \geq 1$ and constant step-size $\al_t = \fr{\al_0}{T^{\nfr{2}{3}}}\,,$
\begin{eqnarray}
    \fr{1}{T} \sum_{t=1}^{T} \Exp{\norm{\nabla J(\theta_t)}} &\leq & \fr{3 (J^* - J(\theta_1))}{\al_0 T} T^{\nfr{2}{3}} +  \fr{6}{T} \sum_{t=1}^{T}\Exp{\norm{e_t}}  +  \fr{3 L_{\theta} \al_0}{T^{\nfr{2}{3}}}  + 4 D_g \g^H \notag \\
    &\leq & \fr{3 (J^* - J(\theta_1))}{\al_0 T^{\nfr{1}{3}}} +  \frac{ 6 C \rb{ V + \rb{ L_g + G C_w^{\nfr{1}{2}} } \al_0 } }{T^{\nfr{1}{3}} }   +  \fr{3 L_{\theta} \al_0}{T^{\nfr{2}{3}}} + 4 D_g \g^H 
\end{eqnarray}
with $C > 0$ being a numerical constant, where we applied Lemma~\ref{lem:et_bound_cumul} to bound $\sum_{t=1}^{T}\Exp{\norm{e_t}}$. Then choosing $H$ large enough, we have after $T$ iterations
\begin{eqnarray}
    \fr{1}{T} \sum_{t=1}^{T} \Exp{\norm{\nabla J(\theta_t)}} 
    & \leq & \cO\rb{ \fr{ J^* - J(\theta_1) }{\al_0  T^{\nfr{1}{3}} } +  \fr{ V + (L_g + G C_{w}^{\nfr{1}{2}}) \al_0  }{T^{\nfr{1}{3}}}   } , \notag 
\end{eqnarray}
which concludes the first part of the corollary. 

As for the second part of the statement, we know from Lemma~\ref{lem:lipschitz-pg-estimate} that 

$$
 J^* - J(\theta_1) = \cO\rb{ \fr{1}{1-\gamma}}, \quad V = \cO\rb{ \fr{1}{(1-\gamma)^{\nfr{3}{2}}}}, \quad G = \cO\rb{ \fr{1}{(1-\gamma)^2}}, 
 $$
 $$
 L_g = \cO\rb{ \fr{1}{(1-\gamma)^2}}, \quad C_w^{\nfr{1}{2}} = \cO\rb{ \fr{1}{1-\gamma}}.
$$

If we set $\al_0 = 1-\gamma$, we derive the desired bound. 

\subsection{Proof of Theorem~\ref{thm:fos-standard-RL-gaussian} (Cumulative reward setting for continuous state-action space and Gaussian policy)}
\label{subsec:fos-gaussian}

In this section, we consider continuous state and action spaces where~$\mathcal{S} = \R^p$ and~$\mathcal{A} = \R^q$ for two positive integers~$p, q \geq 1\,.$ Our focus is on the popular class of Gaussian policies which are common to handle the case of continuous state action spaces in practice. 

Let~$\sigma > 0\,.$ Define for every~$\theta \in \R^d$ a map~$\mu_{\theta}: \mathcal{S} \to \R^q$. Then, we define the Gaussian policy~$\pi_{\theta}$ for each parameter~$\theta \in \R^d$ and each state-action pair~$(s,a) \in \mathcal{S} \times \mathcal{A}$ as follows: 
\begin{equation}
\label{eq:gauss-param}
\pi_\theta(a | s)=\frac{1}{\sigma \sqrt{2 \pi}} \exp \left(-\frac{\sqnorm{\mu_\theta(s)-a}}{2 \sigma^2}\right)\,.
\end{equation} 
Let us mention that~$\mu_{\theta}$ is a parametrization of the Gaussian mean which can be a neural network in practice. The standard deviation~$\sigma$ can be fixed or parametrized as well in practice. We consider a fixed standard deviation for the purpose of our discussion. 

\begin{remark}
Note that one can consider even more general parametrizations such as the exponential family or symmetric $\alpha$-stable policies which include the Gaussian policy as a particular case. We refer the interested reader to the nice exposition in \citet{bedi-et-al21heavy-tailed-policy-search} for a discussion around such heavy-tailed policy parametrizations (see also~\citet{bedi-et-al22contin-action-space}). 
\end{remark}

We make the following standard smoothness assumption on our Gaussian policy parametrization.
\begin{assumption}
\label{hyp:gauss-policy-param}
In the Gaussian parametrization~\eqref{eq:gauss-param}, the map~$\theta  \mapsto \mu_{\theta}(s)$ is continuously differentiable for every~$s \in \mathcal{S}$,  $l_{\mu}$-Lipschitz continuous (uniformly in~$s \in \mathcal{S}$) and there exist~$M_g > 0, M_h > 0$ s.t. for every~$\theta \in \R^d, (s,a) \in \mathcal{S} \times \mathcal{A}$, $\|\nabla \log \pi_{\theta}(a|s)\| \leq M_g \,, \|\nabla_{\theta}^2 \log \pi_{\theta}(a|s)\| \leq M_h\,.$
\end{assumption}

Notice that conditions on the map~$\theta \mapsto \mu_{\theta}(s)$ and its higher-order derivatives can be enforced for every~$s \in \mathcal{S}$ so that the desired regularity conditions on the policy parametrization in Assumption~\ref{hyp:gauss-policy-param} are satisfied upon considering a set of actions lying in a compact set. Consider for instance the simpler case where~$q=1$ and the mean of the policy is parametrized with a linear function, i.e., $\mu_{\theta}(s) = \phi(s)^T \theta$ for some feature map~$\phi: \mathcal{S} \to \R^d\,.$ Then, the boundedness of~$\|\nabla_{\theta}^2 \log \pi_{\theta}(a|s)\|$ is automatically satisfied since~$\nabla_{\theta}^2 \log \pi_{\theta}(a|s)$ is the matrix~$-\frac{1}{\sigma^2} \phi(s) \phi(s)^T$ which is independent from the parameter~$\theta\,.$ As for the first condition, it is satisfied if the feature map~$\phi$ as well as~$\phi_{\theta}(s)$ are bounded over the state space and the policy parameter space while the action set is also bounded. 
Notice though that Assumption~\ref{hyp:gauss-policy-param} can be relaxed to hold in expectation (over state-action pairs) in order to include an even larger class of policies \citep{Vanilla_PL_Yuan_21}. In this work, we do not pursue such relaxations and assume the standard bound for all $s\in \mathcal{S}$, $a \in \mathcal{A}$ for simplicity \citep{xu-et-al20iclr,liu-et-al20}. 

Similarly to the softmax parametrization setting with Lemma~\ref{lem:for-cum-sum}, under Assumption~\ref{hyp:gauss-policy-param}, one can show smoothness of the expected return function~$J$ and derive useful bounds for the norm and the variance of stochastic gradients, see \citep[Lemma 4.2, (68) and Lemma 4.4, (19)]{Vanilla_PL_Yuan_21}, and \citep[Proposition 4.2 (1) and (3)]{xu-et-al20iclr}.
\begin{lemma}\label{lem:for-cum-sum-gaussian}
    Let Assumption~\ref{hyp:gauss-policy-param} hold true and 
let~$\tau = \{s_0, a_0, \cdots, s_{H-1}, a_{H-1}\}$ be an arbitrary trajectory of length~$H$. Then the following statements hold: 
    \begin{enumerate}[label=(\roman*)]
        \item \label{lem:smoothness-cum-sum-gauss} The objective function~$\theta \mapsto J(\theta)$ is $L_{\theta}$-smooth with $L_{\theta} \eqdef  \frac{  \norm{r}_{\infty} ( M_g^2 + M_h ) }{(1-\gamma)^2}\,.$
        \item \label{lem:lipschitz-pg-estimate-theta-gauss} For all $ \theta_1, \theta_2 \in \R^d, \|g(\tau, \theta_1) -g(\tau, \theta_2) \| \leq L_{g} \|\theta_1 - \theta_2\| \,$  with~$L_{g} \eqdef \frac{2 M_g^2 \|r\|_{\infty}}{(1-\gamma)^3} + \frac{ M_h \|r\|_{\infty}}{(1-\gamma)^2}\,,$
        \item \label{lem:BV-cum-sum-gauss} For all $ \theta \in \R^d,\,  \Exp{\sqnorm{g(\tau,\theta) - \nabla_{\theta} J_H(\theta) } } \leq V^2$ with~$V \eqdef \frac{ M_g \|r\|_{\infty}}{(1-\gamma)^{\nfr{3}{2}}}$. 
        \item \label{lem:BG-cum-sum-gauss} For all $ \theta \in \R^d,\, \|g(\tau, \theta )\| \leq G$ with~$G \eqdef \frac{ M_g \norm{r}_{\infty} }{(1-\gamma)^2}\,.$
    \end{enumerate}
\end{lemma}

Given a trajectory~$\tau = (s_0, a_0, s_1, a_1, \cdots, s_{H-1}, a_{H-1})$ of length~$H$ generated under the initial distribution~$\rho$ and the Gaussian policy~$\pi_{\theta}$ as defined in~\eqref{eq:gauss-param} for some~$\theta \in \R^d$, recall the definition of the IS weight for every~$\theta' \in \R^d$: 
\begin{equation}
\label{eq:IS-weights-gauss}
w(\tau|\theta', \theta) %
\eqdef \prod_{h=0}^{H-1}\frac{\pi_{\theta'}(a_h|s_h)}{\pi_{\theta}(a_h|s_h)}\,.
\end{equation}

\begin{lemma}
\label{lem:bounded-var-is-weights-gaussian}
Let~$H \geq 1$ be an integer and let Assumption~\ref{hyp:gauss-policy-param} be satisfied. Suppose that the sequence~$(\theta_t)$ is updated via~$\theta_{t+1} = \theta_t + \al_t \frac{d_t}{\|d_t\|}$ where~$d_t \in \R^d$ is any nonzero update direction and~$\al_t$ is a positive stepsize. If~$\tau_{t+1}$ is a (random) trajectory of length~$H$ generated following the initial distribution~$\rho$ and the Gaussian policy~$\pi_{\theta_{t+1}}$ as defined in~\eqref{eq:gauss-param}, then
\begin{align}
\mathbb{E}[w(\tau_{t+1}|\theta_t,\theta_{t+1})] &= 1\,,\\
\Var{[w(\tau_{t+1}|\theta_t,\theta_{t+1})]} &\leq C_w \al_t^2\,, %
\end{align}
where the IS weight~$w(\tau_{t+1}|\theta_t,\theta_{t+1})$ is as defined in~\eqref{eq:IS-weights-gauss} and~$C_w \eqdef (2H^2 M_g + H M_h) (W+1)\,.$ 
\end{lemma}

\begin{proof}
The first identity follows from the definitions of the expectation and the IS weight. We now prove the second identity. 
For any~$\theta \in \R^d$, let~$p(\cdot|\pi_{\theta})$ denote the probability distribution induced by the policy~$\pi_{\theta}$ over the space of random trajectories of length~$H$ initialized with the state distribution~$\rho$. The probability density is then given by 
\begin{equation}
p(\tau|\pi_{\theta}) = \rho(s_0)\, \pi_{\theta}(a_0|s_0) \prod_{t=1}^{H-1} \mathcal{P}(s_t|s_{t-1},a_{t-1})\, \pi_{\theta}(a_t|s_t)\,, 
\end{equation}
where~$\tau = (s_0, a_0, \cdots, s_{H-1}, a_{H-1})\,.$

We use the shorthand notations~$\theta_1 = \theta_t, \theta_2 = \theta_{t+1}$ and~$\tau = \tau_{t+1}$ for the rest of this proof. Then, we have 
\begin{align}
\label{eq:second-moment-is}
\mathbb{E}[w(\tau|\theta_1, \theta_2)^2] 
&= \int \frac{p(\tau|\pi_{\theta_1})^2}{p(\tau|\pi_{\theta_2})} d\tau \nonumber\\
&= \int \rho(s_0)\, \pi_{\theta}(a_0|s_0) \prod_{t=1}^{H-1} \mathcal{P}(s_t|s_{t-1},a_{t-1})\, \frac{\pi_{\theta_1}(a_t|s_t)^2}{\pi_{\theta_2}(a_t|s_t)} d\tau\,.
\end{align}
We bound the above integral starting from the integral of the last term of the product\footnote{Notice that the integrand is nonnegative and we can integrate in any order by Tonelli's theorem.} which writes as follows: 
\begin{equation}
\label{eq:integral-intermediate}
\int \mathcal{P}(s_{H-1}|s_{H-2},a_{H-2}) \int \frac{\pi_{\theta_1}(a_{H-1}|s_{H-1})^2}{\pi_{\theta_2}(a_{H-1}|s_{H-1})} d a_{H-1} ds_{H-1}\,.
\end{equation}
We shall now compute the integral w.r.t.~$a_{H-1}$. In dimension~1 for the action variable~$a_{H-1}$ (similar derivations hold for higher dimensions), we have for every~$s$, 
\begin{align}
&\int_{-\infty}^{+\infty} \frac{\pi_{\theta_1}(x \mid s)^2}{\pi_{\theta_2}(x \mid s)} d x \nonumber\\
&=\frac{1}{\sigma \sqrt{2 \pi}} \int_{-\infty}^{+\infty} \exp \left[-\frac{2\left(x-\mu_{\theta_1}\left(s\right)\right)^2-\left(x-\mu_{\theta_2}\left(s\right)\right)^2}{2 \sigma^2}\right] d x \nonumber\\
&=\frac{1}{\sigma \sqrt{2 \pi}} \exp \left[-\frac{2 \mu_{\theta_1}\left(s\right)^2-\mu_{\theta_2}\left(s\right)^2}{2 \sigma^2}\right] \int_{-\infty}^{+\infty} \exp \left[-\frac{x^2-2\left(2 \mu_{\theta_1}\left(s\right)-\mu_{\theta_2}\left(s\right)\right) x}{2 \sigma^2}\right] d x \nonumber\\
&=\frac{1}{\sigma\sqrt{2 \pi}} \exp \left[-\frac{2 \mu_{\theta_1}\left(s\right)^2-\mu_{\theta_2}\left(s\right)^2-\left(2 \mu_{\theta_1}\left(s\right)-\mu_{\theta_2}\left(s\right)\right)^2}{2 \sigma^2}\right] \int_{-\infty}^{+\infty} \exp \left[-\frac{\left(x-\left(2 \mu_{\theta_1}\left(s\right)-\mu_{\theta_2}\left(s\right)\right)\right)^2}{2 \sigma^2}\right] d x \nonumber\\
&=\exp \left[-\frac{2 \mu_{\theta_1}\left(s\right)^2-\mu_{\theta_2}\left(s\right)^2-\left(2 \mu_{\theta_1}\left(s\right)-\mu_{\theta_2}\left(s\right)\right)^2}{2 \sigma^2}\right] \nonumber\\
&=\exp \left[\frac{\left.\left(\mu_{\theta_2}\left(s\right)-\mu_{\theta_1}\left(s\right)\right)^2\right]}{\sigma^2}\right]
\end{align}
As a consequence, we obtain
\begin{align}
&\int \mathcal{P}(s_{H-1}|s_{H-2},a_{H-2}) \int \frac{\pi_{\theta_1}(a_{H-1}|s_{H-1})^2}{\pi_{\theta_2}(a_{H-1}|s_{H-1})} d a_{H-1} ds_{H-1} \nonumber\\
&\leq \int \mathcal{P}(s_{H-1}|s_{H-2},a_{H-2}) \exp \left(\frac{\sqnorm{ \mu_{\theta_1}\left(s_{H-1}\right)-\mu_{\theta_2}\left(s_{H-1}\right) } }{\sigma^2}\right) ds_{H-1} \nonumber\\
&\stackrel{(i)}{\leq} \int \mathcal{P}(s_{H-1}|s_{H-2},a_{H-2}) \exp \left(\frac{l_{\mu}^2 \sqnorm{ \theta_1 - \theta_2} }{\sigma^2}\right) ds_{H-1}  \nonumber\\ 
&\stackrel{(ii)}{=} \exp \left(\frac{l_{\mu}^2 \alpha_t^2 }{\sigma^2}\right)\,, 
\end{align}
where~(i) follows from the $l_{\mu}$-Lipschitzness of the parametrized mean in Assumption~\ref{hyp:gauss-policy-param} and~(ii) utilizes the normalized update rule as well as the fact that~$\mathcal{P}(\cdot|s_{H-2}, a_{H-2})$ is a transition probability kernel. 
Using a similar reasoning to bound the different integrals like in~\eqref{eq:integral-intermediate} backward from~$H-1$ to~$0$ successively, we obtain the following bound on the second moment of IS weights in~\eqref{eq:second-moment-is}:
\begin{equation}
\mathbb{E}[w(\tau|\theta_1, \theta_2)^2] \leq \exp \left(\frac{H l_{\mu}^2 \alpha_t^2 }{\sigma^2}\right)\,.
\end{equation}
Therefore, similarly to the argument in Lemma~\ref{lem:variance-IS-weights-control}, we obtain: 
\begin{equation}
\Var(w(\tau|\theta_1,\theta_2)) \leq W\,,
\end{equation}
where~$W = \cO\rb{ 1 }$ is a numerical constant, which can be ensured, for example, by setting the step-sizes as~$\al_t = \al = T^{-2/3}$.
As a consequence, we can apply Lemma B.1 in \cite{xu-et-al20iclr} and derive the bound on the variance of IS weights: 
$$
\Var{[w(\tau_{t+1}|\theta_t,\theta_{t+1})]} \leq C_w \|\theta_{t+1} - \theta_t\|^2 = C_w \al_t^2\,,
$$
where the last step follows by the update rule $\theta_{t+1} = \theta_t + \al \fr{d_t}{ \norm{d_t} }$. This concludes the proof. 
\end{proof}

Given the above results, we immediately obtain convergence of Algorithm~\ref{alg:N-VR-PG} for Gaussian policy parametrization.

\begin{corollary}[\textbf{Stationary convergence of \algname{N-VR-PG}}]
    \label{thm:N-VR-PG_stat}
    Let Assumption~\ref{hyp:gauss-policy-param} hold. 
   Let~$\alpha_0 > 0$ and let~$T$ be an integer larger than~$1$. 
    Set~$\al_t = \fr{\al_0}{T^{\nfr{2}{3}}}$, $\eta_t = \rb{ \fr{2}{t+1} }^{\nfr{2}{3}}$ and~$H = \rb{1-\g}^{-1}{\log(T + 1)}$. Let~$\bar{\theta}_T$ be sampled from the iterates~$\{\theta_1, \cdots, \theta_T\}$ of \algname{N-VR-PG} (Algorithm~\ref{alg:N-VR-PG}) uniformly at random. Then we have %
    \begin{equation}\label{eq:rate-nvrpg-gaussian}
    \Exp{\norm{\nabla J(\bar{\theta}_T)}}
     \leq  \cO\rb{ \fr{ J^* - J(\theta_1) }{\al_0  T^{\nfr{1}{3}} } +  \fr{ V + (L_g + G C_{w}^{\nfr{1}{2}}) \al_0  }{T^{\nfr{1}{3}}}   } ,
    \end{equation}
    where $V$, $L_g$, $G$, and $C_w$ are defined in Lemma~\ref{lem:for-cum-sum-gaussian} and \ref{lem:bounded-var-is-weights-gaussian}. Moreover, if we set $ \al_0 = 1-\gamma$, then
    $$\Exp{ \norm{\nabla J(\bar{\theta}_T)} } \leq \cO\rb{   \fr{ 1   }{(1-\gamma)^{2} T^{\nfr{1}{3}}}   }.$$
\end{corollary}
\begin{proof}
    Given the results of Lemma~\ref{lem:for-cum-sum-gaussian} and \ref{lem:bounded-var-is-weights-gaussian}, the proof of this statement follows immediately from the result of Corollary~\ref{thm:N-VR-PG_stat}. We notice that in order to specify the dependence on $1-\gamma$, we invoke Lemma~\ref{lem:for-cum-sum-gaussian}, which is analogous to the corresponding Lemma~\ref{lem:for-cum-sum} for softmax policy parameterization. The only difference in terms of the dependence on~$1-\gamma$ is in the bound for $L_g$. However, this fact does not affect the final dependence on $1-\gamma$ since it is dominated by other terms in~\eqref{eq:rate-nvrpg-gaussian}. 
\end{proof}

\section{Proofs for Section~\ref{subsec:glob-opt}: Global optimality convergence} 
\label{sec:app-proofs-globopt}

\subsection{Proof of Theorem~\ref{thm:glob-opt-gen-ut-tabular} (General utilities setting)}

In this section, to prove our global convergence result under an additional concave reparametrization assumption, we refine the result of Lemma~\ref{lem:ascent-like-lemma-gen-ut}. The proof is similar to the proof of Lemma~5.12 in \citet{zhang-et-al21}. Nevertheless, we would like to mention that it deviates from the latter in that our algorithm is significantly different and its normalized nature requires a significantly different treatment. In particular, the reader can appreciate from the statement of the result that the error term~$\|e_t\|$ to the gradient estimation is not squared unlike in \cite{zhang-et-al21} and controlling its magnitude required different proof techniques given the different recursive loopless variance reduction mechanism that we consider. 

\begin{lemma}\label{lem:ascent_like_hidden_conv}
    Let Assumptions~\ref{hyp:policy-param}, \ref{hyp:smoothness-F} and Assumption~\ref{hyp:F-concave} hold. Additionally, let Assumption~\ref{hyp:overparam-global-opt} be satisfied with some positive~$\bar{\epsilon}\,.$ Then, the sequence~$(\theta_t)$ generated by Algorithm~\ref{algo-gen-ut} and the sequence~$(e_t)$ satisfy for every positive real~$\epsilon \leq \min\cb{\bar{\epsilon}, \fr{\al_t (1-\g)}{2 \ell_{\theta}} }$ and every integer~$t$, 
\begin{equation}
F(\lambda(\theta^*)) - F(\lambda(\theta_{t+1})) \leq (1-\epsilon) (F(\lambda(\theta^*)) - F(\lambda(\theta_t))) + 2 \al_t \|e_t\| + \frac{4 L_{\theta} l_{\theta}^2}{(1-\gamma)^2} \epsilon^2 + \frac{4}{3} \al_t D_{\lambda} \gamma^H + \frac{L_{\theta}}{2} \al_t^2\,.
\end{equation}
\end{lemma}

\begin{proof}
Lemma~\ref{lem:ascent-like-lemma-gen-ut} provides the following inequality:
\begin{equation}
\label{eq:smoothness-1}
F(\lambda(\theta_{t+1}))  \geq F(\lambda(\theta_t)) + \frac{\al_t}{3} \|\nabla_{\theta} F(\lambda(\theta_t))\| - 2 \al_t \|e_t\| - \frac{4}{3} D_{\lambda} \gamma^H \al_t 
- \frac{L_{\theta}}{2} \al_t^2\,.
\end{equation}

Now, for any~$\epsilon < \bar{\epsilon}$, the concavity reparametrization assumption implies that~$(1-\epsilon) \lambda(\theta_t) + \epsilon \lambda(\theta^*) \in \mathcal{V}_{\lambda(\theta_t)}$ and therefore we have 
\begin{equation}
\label{def:theta-eps}
\theta_{\epsilon} \eqdef (\lambda|_{\mathcal{U}_{\theta_t}})^{-1}((1-\epsilon) \lambda(\theta_t) + \epsilon \lambda(\theta^*)) \in \mathcal{U}_{\theta_t}\,.
\end{equation}
It also follows from the smoothness of the objective function~$\theta \mapsto F(\lambda(\theta))$ that 
\begin{equation}
\label{eq:smoothness-2}
F(\lambda(\theta_t)) \geq F(\lambda(\theta_{\epsilon})) - \ps{\nabla_{\theta} F(\lambda(\theta_t)), \theta_{\epsilon}- \theta_t} - \frac{L_{\theta}}{2} \|\theta_{\epsilon} - \theta_t\|^2\,.
\end{equation}
Combining~\eqref{eq:smoothness-1} and~\eqref{eq:smoothness-2} yields
\begin{multline}
\label{eq:interm-res-smoothness-1-2}
 F(\lambda(\theta_{t+1})) 
 \geq F(\lambda(\theta_{\epsilon})) - \ps{\nabla_{\theta} F(\lambda(\theta_t)), \theta_{\epsilon}- \theta_t} - \frac{L_{\theta}}{2} \|\theta_{\epsilon} - \theta_t\|^2\\
 + \frac{\al_t}{3} \|\nabla_{\theta} F(\lambda(\theta_t))\| - 2 \al_t \|e_t\| - \frac{4}{3} D_{\lambda} \gamma^H \al_t 
- \frac{L_{\theta}}{2} \al_t^2\,.
\end{multline}
Then, we notice that:
\begin{enumerate}[label=(\roman*)]
\item By assumption, the mapping~$\lambda \circ (\lambda|_{\mathcal{U}_{\theta_t}})^{-1}$ coincides with the identity mapping on the set~$\mathcal{U}_{\theta_t}$\,. Hence, given the definition of~$\theta_{\epsilon}$ in~\eqref{def:theta-eps}, we have
\begin{align}
\label{eq:item1}
F(\lambda(\theta_{\epsilon})) &= F((1-\epsilon) \lambda(\theta_t) + \epsilon \lambda(\theta^*))\nonumber\\
&\geq (1-\epsilon) F(\lambda(\theta_t)) + \epsilon F( \lambda(\theta^*))\,,
\end{align}
where the last step follows from the concavity of the function~$F\,.$
\item Again since the mapping~$\lambda \circ (\lambda|_{\mathcal{U}_{\theta_t}})^{-1}$ coincides with the identity mapping on the set~$\mathcal{U}_{\theta_t}$ and using the (uniform) lipschitzness of the inverse mapping~$(\lambda|_{\mathcal{U}_{\theta_t}})^{-1}$, we have
\begin{align}
\label{eq:item2}
\|\theta_{\epsilon} - \theta_t\| &= \|(\lambda|_{\mathcal{U}_{\theta_t}})^{-1}((1-\epsilon) \lambda(\theta_t) + \epsilon \lambda(\theta^*)) - (\lambda|_{\mathcal{U}_{\theta_t}})^{-1}(\lambda(\theta_t))\| \nonumber\\
&\leq l_{\theta} \epsilon \|\lambda(\theta_t) - \lambda(\theta^*)\| \nonumber\\
&\leq \frac{2 l_{\theta} \epsilon}{(1-\gamma)}\,.
\end{align}

\item Using the Cauchy-Schwarz inequality together with the inequality established in the previous item gives
\begin{align} 
\label{eq:inner-prod-theta-eps-bound}
|\ps{\nabla_{\theta} F(\lambda(\theta_t)), \theta_{\epsilon} - \theta_t }| 
&\leq \|\nabla_{\theta} F(\lambda(\theta_t))\| \cdot \|\theta_{\epsilon} - \theta_t\|\nonumber\\
&\leq  \frac{2 l_{\theta} \epsilon}{1-\gamma} \|\nabla_{\theta} F(\lambda(\theta_t))\|\,.
\end{align}

\end{enumerate}
Substituting the inequalities~\eqref{eq:item1}, \eqref{eq:item2}  and~\eqref{eq:inner-prod-theta-eps-bound} into~\eqref{eq:interm-res-smoothness-1-2} leads to
\begin{eqnarray}
F(\lambda(\theta_{t+1})) 
 &\geq& (1-\epsilon) F(\lambda(\theta_t)) + \epsilon F( \lambda(\theta^*)) + \left(\frac{\al_t}{3} - \frac{2 l_{\theta} \epsilon}{1-\gamma}  \right) \|\nabla_{\theta} F(\lambda(\theta_t))\| - \frac{4 L_{\theta} l_{\theta}^2}{(1-\gamma)^2}\epsilon^2 
- 2 \al_t \|e_t\| \notag \\
&& \qquad - \frac{4}{3} D_{\lambda} \gamma^H \al_t 
- \frac{L_{\theta}}{2} \al_t^2\, \notag \\
&\geq& (1-\epsilon) F(\lambda(\theta_t)) + \epsilon F( \lambda(\theta^*)) - \frac{4 L_{\theta} l_{\theta}^2}{(1-\gamma)^2}\epsilon^2 
- 2 \al_t \|e_t\| - \frac{4}{3} D_{\lambda} \gamma^H \al_t 
- \frac{L_{\theta}}{2} \al_t^2\,,
\end{eqnarray}
where the last step follows from the condition $\epsilon \leq  \fr{\al_t (1-\g)}{2 \ell_{\theta}} $.

Finally, substracting~$F(\lambda(\theta^*))$ from both sides and rearranging the terms gives the desired result:
\begin{equation}
F(\lambda(\theta^*)) - F(\lambda(\theta_{t+1})) \leq (1-\epsilon) (F(\lambda(\theta^*)) - F(\lambda(\theta_t))) + 2 \al_t \|e_t\| + \frac{4 L_{\theta} l_{\theta}^2}{(1-\gamma)^2} \epsilon^2 + \frac{4}{3} \al_t D_{\lambda} \gamma^H + \frac{L_{\theta}}{2} \al_t^2\,.
\end{equation}

\end{proof}

\begin{theorem}[\textbf{Global convergence of \algname{N-VR-PG} for general utilities}]
    \label{thm:N-VR-PG}
    Let Assumptions~\ref{hyp:policy-param} and~\ref{hyp:F-concave} hold. Additionally, let Assumption~\ref{hyp:overparam-global-opt} be satisfied with $\bar{\epsilon} \geq \fr{ \al_0 (1-\g)}{2 \ell_{\theta} (T+1)^a}$ for some integer $T \geq 1$ and reals $\al_0>0$, $a\in(0,1)$.
    Set~$\al_t = \fr{\al_0}{(T+1)^a}$, $\eta_t = \fr{2}{t+1}$ for every integer~$t$ and~$H = \rb{1-\g}^{-1}{\log(T + 1)}$. Then the output~$\theta_T$ of \algname{N-VR-PG} (see Algorithm~\ref{algo-gen-ut}) satisfies
    $$
        F(\lambda(\theta^*)) - \Exp{ F(\lambda(\theta_T))  } \leq \cO%
        \left( \fr{ \al_0^2  }{(1-\gamma)^3 (T+1)^{ 2 a -\fr{3}{2}}} \right)  ,
    $$
    where~$F(\lambda(\theta^*))$ is the optimal utility value. Therefore, the sample complexity to achieve $F(\lambda(\theta^*)) - \Exp{ F(\lambda(\theta_T))  } \leq \varepsilon$ is $\cO\rb{ \varepsilon^{\fr{-2}{4a - 3}}}$.
\end{theorem}

\begin{proof}
    Define $\delta_t \eqdef \Exp{F(\lambda(\theta^*)) - F(\lambda(\theta_t))}$. Applying expectation to the result of Lemma~\ref{lem:ascent_like_hidden_conv}, we have for $\epsilon \leq \min\cb{\bar{\epsilon}, \fr{\al_t (1-\g)}{2 \ell_{\theta}} }\,,$
\begin{eqnarray}\label{eq:deltat-global}
\delta_{t+1} &\leq& (1-\epsilon) \delta_t + 2 \al_t \Exp{ \norm{ e_t } } + \frac{4 L_{\theta} l_{\theta}^2}{(1-\gamma)^2} \epsilon^2 + \frac{4}{3} \al_t D_{\lambda} \gamma^H + \frac{L_{\theta}}{2} \al_t^2\, \notag \\
&\leq& (1-\epsilon) \delta_t + 2 \al_t \Exp{ \norm{ \hat{e}_t } } + 2 C_1 \al_t \Exp{ \norm{ \tilde{e}_t } } + 2 C_2 \al_t \al_{t-1} + \frac{4 L_{\theta} l_{\theta}^2}{(1-\gamma)^2} \epsilon^2 + \frac{4}{3} \al_t D_{\lambda} \gamma^H + \frac{L_{\theta}}{2} \al_t^2\, 
\end{eqnarray}
where in the last step we apply Lemma~\ref{lem:e-t-bound} with~$C_1 \eqdef \frac{2 L_{\lambda}^2 l_{\psi}}{(1-\gamma)^2}\,$ and~$C_2 \eqdef \frac{2 L_{\lambda} L_{\lambda,\infty} l_{\psi}}{(1-\gamma)^2}\,.$

By Lemma~\ref{lem:recursive-tilde-e-t} (Equation~\eqref{eq:recursive-tilde-e-t_solved}), for $\eta_t = \fr{2}{t+1} $, we have
\begin{equation}\label{eq:tilde-et-global}
\mathbb{E}[\|\tilde{e}_t\|] \leq  \frac{4}{(1-\gamma)}\eta_t \cdot t^{\nfr{1}{2}} + \frac{2 C_w^{\nfr{1}{2}}}{(1-\gamma)} \al_{t-1}  \cdot t^{\nfr{1}{2}}   .
\end{equation}

With the same $\eta_t$ as above, by Lemma~\ref{lem:recursive-hat-e-t} (Equation~\eqref{eq:recursive-hat-e-t_solved}), we have
\begin{eqnarray}\label{eq:hat-et-global}
\mathbb{E}[\|\hat{e}_t\|]
&\leq & \fr{2 \hat{E} }{t+1} + 2 C_3^{\nfr{1}{2}} \eta_{t} \cdot t^{\nfr{1}{2}} + C_4^{\nfr{1}{2}} \al_{t-2}  \cdot t^{\nfr{1}{2}} ,
\end{eqnarray}
where~$C_3 \eqdef  \frac{288 C l_{\psi}^2 L_{\lambda}^2}{(1-\gamma)^6} 
+ \frac{32 l_{\lambda}^2 l_{\psi}^2}{(1-\gamma)^4} $, $C_4 \eqdef  \frac{12 l_{\lambda}^2[(l_{\psi}^2 + L_{\psi})^2 + C_w l_{\psi}^2]}{(1-\gamma)^4}  + \frac{144 C_w l_{\psi}^2 L_{\lambda}^2}{(1-\gamma)^6}  $, and $C_w =  H ((8H+2) l_{\psi}^2 + 2 L_{\psi}) (W+1)\,$.

Unrolling \eqref{eq:deltat-global} from $t = T-1$ to $t = 0$, using \eqref{eq:tilde-et-global} and \eqref{eq:hat-et-global} and setting $\al_t = \al $, we have
\begin{eqnarray}
\delta_{T} &\leq& (1-\epsilon)^T \delta_0 + 2 \al \sum_{t=0}^{T-1}  \rb{ \Exp{ \norm{ \hat{e}_t } } + C_1  \Exp{ \norm{ \tilde{e}_t } } } + 2 C_2 \al^2 T + \frac{4 L_{\theta} l_{\theta}^2}{(1-\gamma)^2} \epsilon + \frac{4}{3} \fr{\al}{\epsilon} D_{\lambda} \gamma^H + \frac{L_{\theta}}{2} \fr{\al^2}{\epsilon} \, \notag \\
&\leq&  (1-\epsilon)^T \delta_0 + 2\al \sum_{t=0}^{T-1} \rb{ \fr{2 \hat{E} }{t+1} + 2 C_3^{\nfr{1}{2}} \eta_{t} \cdot t^{\nfr{1}{2}} + C_4^{\nfr{1}{2}} \al  \cdot t^{\nfr{1}{2}} } +  2 C_2 \al^2 T \notag \\
&& \qquad + 2 C_2 \al \sum_{t=0}^{T-1}   \rb{ \frac{4}{(1-\gamma)}\eta_t \cdot t^{\nfr{1}{2}} + \frac{2 C_w^{\nfr{1}{2}}}{(1-\gamma)} \al  \cdot t^{\nfr{1}{2}}  } + \frac{4 L_{\theta} l_{\theta}^2}{(1-\gamma)^2} \epsilon + \frac{4}{3} \fr{\al}{\epsilon} D_{\lambda} \gamma^H + \frac{L_{\theta}}{2} \fr{\al^2}{\epsilon} \, \notag \\
&\leq&  (1-\epsilon)^T \delta_0 +  4 \al \hat{E} \log(T) + 8 \al C_3^{\nfr{1}{2}}  (T+1)^{\nfr{1}{2}} + 2 C_4^{\nfr{1}{2}} \al^2  \cdot (T+1)^{\nfr{3}{2}}  +  2 C_2 \al^2 T \notag \\
&& \qquad +   \frac{16 C_2 \al }{(1-\gamma)}  (T+1)^{\nfr{1}{2}} + \frac{4 C_2 C_w^{\nfr{1}{2}}}{(1-\gamma)} \al^2   (T+1)^{\nfr{3}{2}}   + \frac{4 L_{\theta} l_{\theta}^2}{(1-\gamma)^2} \epsilon + \frac{4}{3} \fr{\al}{\epsilon} D_{\lambda} \gamma^H + \frac{L_{\theta}}{2} \fr{\al^2}{\epsilon} \, \notag . 
\end{eqnarray}
Notice that $(1-\epsilon)^T \leq \exp\rb{T \log(1-\epsilon) } \leq \exp\rb{ - \epsilon T }$.
Finally setting $\al = \fr{\al_0}{(T+1)^a}$, for $0 < a < 1$ and $\epsilon = \min\cb{\bar{\epsilon}, \fr{\al (1-\g)}{2 \ell_{\theta}} } = \fr{\al (1-\g)}{2 \ell_{\theta}  } $, we obtain 
\begin{eqnarray}
\delta_T &\leq& \exp\rb{ - \fr{\al_0 (1-\g) }{2 \ell_{\theta} } T^{1-a} } +  \fr{ 4 \hat{E} \log(T)  \al_0}{(T+1)^a} + \rb{ 8  C_3^{\nfr{1}{2}} + \frac{16 C_2 }{(1-\gamma)}  } \fr{\al_0}{(T+1)^{a - \nfr{1}{2}}}   +   \fr{2 C_2 \al_0^2}{(T+1)^{2a - 1}}  \notag \\
&& \qquad  + \rb{ 2 C_4^{\fr{1}{2}} + \frac{4 C_2 C_w^{\nfr{1}{2}}}{(1-\gamma)} } \fr{\al_0^2}{(T+1)^{2a - \fr{3}{2}}}   + \frac{4 L_{\theta} l_{\theta}^2}{(1-\gamma)^2} \epsilon + \frac{4}{3} \fr{\al}{\epsilon} D_{\lambda} \gamma^H + \frac{L_{\theta}}{2} \fr{\al^2}{\epsilon} \notag \\
&\leq& \exp\rb{ - \fr{\al_0 (1-\g) }{2 \ell_{\theta} } T^{1-a} } +  \fr{ 4 \hat{E} \log(T)  \al_0}{(T+1)^a} + \rb{ 8  C_3^{\nfr{1}{2}} + \frac{16 C_2 }{(1-\gamma)}  } \fr{\al_0}{(T+1)^{a - \nfr{1}{2}}}   +   \fr{2 C_2 \al_0^2}{(T+1)^{2a - 1}}  \notag \\
&& \qquad  + \rb{ 2 C_4^{\fr{1}{2}} + \frac{4 C_2 C_w^{\nfr{1}{2}}}{(1-\gamma)} } \fr{\al_0^2}{(T+1)^{2a - \fr{3}{2}}}   + \fr{3 L_{\theta} l_{\theta} \al_0 }{(1-\gamma) (T+1)^a }   +  \fr{8 \ell_{\theta} D_{\lambda} }{3 (1-\g) }    \gamma^H  \notag \\
&\leq& \cO\rb{ \fr{ 1 }{ (1-\gamma)^3 (T+1)^{2 a - \fr{3}{2}} }  }\,,   \, \notag 
\end{eqnarray}
where the last step follows by setting $H = (1-\gamma)^{-1} \log(T)$ and noticing that $2 a - \fr{3}{2} <  a - \fr{1}{2}$ for $a \in (0,1)$, $C_4 = \cO\rb{(1-\gamma)^{-6}}$, $C_w = \cO\rb{(1-\gamma)^{-2}}$, $C_2 = \cO\rb{(1-\gamma)^{-2}}$. 
\end{proof}

\subsection{Proof of Corollary~\ref{cor:glob-opt-standard-rl-tabular} (Cumulative reward setting)}
\label{subsec:app-globopt-standard-rl}

We first recall that similarly to Section~\ref{subsec:soft-max-cumsum-fos}, for cumulative reward setting, we redefine the error sequence~$(e_t)$ as
\begin{equation}
e_t = d_t - \nabla J_H(\theta_t)\,, \notag 
\end{equation}
where the truncated cumulative reward~$J_H(\theta)$ is defined as
\begin{equation}
J_H(\theta) = \mathbb{E}\left[\sum_{t=0}^{H-1} \gamma^t r(s_t,a_t)\right]\,. \notag 
\end{equation}

Now we state a complete version of Corollary~\ref{cor:glob-opt-standard-rl-tabular}, which we shall prove in this section.  

\begin{corollary}[\textbf{Global convergence of \algname{N-VR-PG}}]
    \label{thm:N-VR-PG}
    Let Assumptions~\ref{hyp:policy-param} and~\ref{hyp:F-concave} hold. Additionally, let Assumption~\ref{hyp:overparam-global-opt} be satisfied with $\bar{\epsilon} \geq \fr{ \al_0 (1-\g)}{2 \ell_{\theta} (T+1)^a}$ for some integer $T \geq 1$ and reals $\al_0>0$, $a\in(0,1)$.
    Set~$\al_t = \fr{\al_0}{(T+1)^a}$, $\eta_t = \fr{2}{t+2} $ and~$H = \rb{1-\g}^{-1}{\log(T + 1)}$. Then the output~$\theta_T$ of \algname{N-VR-PG} (see Algorithm~\ref{alg:N-VR-PG}) satisfies
    $$
        J^* - \Exp{ J(\theta_T) } \leq \cO%
        \left( \fr{ \al_0 V }{(T+1)^{ a -\fr{1}{2}}} \right)  ,
    $$
    where~$J^*$ is the optimal expected return and~$V$ is defined in Lemma~\ref{lem:for-cum-sum}. Therefore, the sample complexity to achieve $J^* - \Exp{ J(\theta_T) } \leq \varepsilon$ is $\cO\rb{ \varepsilon^{\fr{-2}{2a - 1}}}$.
\end{corollary}

\begin{remark}
    If we are allowed to select $\al_0$ based on the problem parameters (only the bound on $(1-\gamma)$ is actually needed here), then the dependence on $(1-\gamma)^{-1}$ in the above theorem can be made arbitrary small.
\end{remark}

\begin{proof}

By Lemma~\ref{lem:et_bound_cumul}, we have the control of the variance sequence for $\eta_t = \fr{2}{t+2}$ as

\begin{equation}\label{eq:et_bound_cumul}
\Exp{ \norm{e_t} } \leq   4 V \eta_t \cdot t^{\nfr{1}{2}} + 2(L_g + G C_w^{\nfr{1}{2}} )  \al_{t-1} \cdot t^{\nfr{1}{2}}  .
\end{equation}

Define $\delta_t \eqdef \Exp{J(\theta^*) - J(\theta_t)}$, where in the cumulative reward case $F(\lambda(\theta)) = J(\theta)$. Let $\al_t = \al $ for all $t = 0, \ldots, T-1$. Then applying full expectation to the result of Lemma~\ref{lem:ascent_like_hidden_conv}, we have for $\epsilon \leq \min\cb{\bar{\epsilon}, \fr{\al (1-\g)}{2 \ell_{\theta}} }$
\begin{eqnarray}
\delta_{t+1} &\leq& (1-\epsilon) \delta_t + 2 \al \Exp{ \norm{ e_t } } + \frac{4 L_{\theta} l_{\theta}^2}{(1-\gamma)^2} \epsilon^2 + \frac{4}{3} \al D_{\lambda} \gamma^H + \frac{L_{\theta}}{2} \al^2\, \notag .
\end{eqnarray}
Unrolling the recursion from $t = 0$ to $t = T-1$, we have
\begin{eqnarray}
\delta_{T} &\leq& (1-\epsilon)^T \delta_0 + 2 \al \sum_{t=0}^{T-1}  \Exp{ \norm{ e_t } } + \frac{4 L_{\theta} l_{\theta}^2}{(1-\gamma)^2} \epsilon + \frac{4}{3} \fr{\al}{\epsilon} D_{\lambda} \gamma^H + \frac{L_{\theta}}{2} \fr{\al^2}{\epsilon} \, \notag \\
&\leq&  (1-\epsilon)^T \delta_0 + 8 V \al \sum_{t=0}^{T-1}   \eta_t \cdot t^{\nfr{1}{2}}   + 4 (L_g + G C_w^{\nfr{1}{2}} ) \al^2  T^{\nfr{1}{2}} + \frac{4 L_{\theta} l_{\theta}^2}{(1-\gamma)^2} \epsilon + \frac{4}{3} \fr{\al}{\epsilon} D_{\lambda} \gamma^H + \frac{L_{\theta}}{2} \fr{\al^2}{\epsilon} \, \notag \\
&\leq& (1-\epsilon)^T \delta_0 +   8 V \al (T+1)^{\nfr{1}{2}}   + 4 (L_g + G C_w^{\nfr{1}{2}} ) \al^2  T^{\nfr{1}{2}} + \frac{4 L_{\theta} l_{\theta}^2}{(1-\gamma)^2} \epsilon + \frac{4}{3} \fr{\al}{\epsilon} D_{\lambda} \gamma^H + \frac{L_{\theta}}{2} \fr{\al^2}{\epsilon} \, \notag .
\end{eqnarray}
Notice that $(1-\epsilon)^T \leq \exp\rb{T \log(1-\epsilon) } \leq \exp\rb{ - \epsilon T }$.
Finally setting $\al = \fr{\al_0}{(T+1)^a}$, for $0 < a < 1$ and $\epsilon = \min\cb{\bar{\epsilon}, \fr{\al (1-\g)}{2 \ell_{\theta}} } = \fr{\al (1-\g)}{2 \ell_{\theta}  } $, we obtain 
\begin{eqnarray}
\delta_T &\leq& \exp\rb{ - \fr{\al_0 (1-\g) }{2 \ell_0 } T^{1-a} } + \fr{8 \al_0 V }{ (T+1)^{a - \fr{1}{2}} }  + \fr{ 4 \al_0^2 (L_g + G C_w^{\nfr{1}{2}} ) }{  T^{2a - \fr{1}{2}} } + \frac{4 L_{\theta} l_{\theta}^2}{(1-\gamma)^2} \epsilon + \frac{4}{3} \fr{\al}{\epsilon} D_{\lambda} \gamma^H + \frac{L_{\theta}}{2} \fr{\al^2}{\epsilon} \, \notag \\
&\leq& \exp\rb{ - \fr{\al_0 (1-\g) }{2 \ell_0 } T^{1-a} } + \fr{8 \al_0 V }{ (T+1)^{a - \fr{1}{2}} }  + \fr{ 4 \al_0^2 (L_g + G C_w^{\nfr{1}{2}} ) }{  T^{2a - \fr{1}{2}} }   + \frac{2 L_{\theta} l_{\theta} \al_0 }{(1-\gamma) (T+1)^a }   + \frac{4}{3} \fr{\al}{\epsilon} D_{\lambda} \gamma^H + \frac{L_{\theta}}{2} \fr{\al^2}{\epsilon} \, \notag \\
&\leq& \exp\rb{ - \fr{\al_0 (1-\g) }{2 \ell_0 } T^{1-a} } + \fr{8 \al_0 V }{ (T+1)^{a - \fr{1}{2}} }  + \fr{ 4 \al_0^2 (L_g + G C_w^{\nfr{1}{2}} ) }{  T^{2a - \fr{1}{2}} }  + \fr{3 L_{\theta} l_{\theta} \al_0 }{(1-\gamma) (T+1)^a }   +  \fr{8 \ell_{\theta} D_{\lambda} }{3 (1-\g) }    \gamma^H    \, \notag  \\
&\leq& \cO\rb{ \fr{ 1 }{ (T+1)^{a - \fr{1}{2}} }  } ,  \, \notag 
\end{eqnarray}
where the last step follows by setting $H = (1-\gamma)^{-1} \log(T)$.
\end{proof}

\subsection{Global optimality in the cumulative reward setting for continuous state-action space and Gaussian policy}
\label{subsec:app-globopt-gaussian}

We first present our set of assumptions to derive global convergence results under the Gaussian policy parameterization. We start by assuming that our Gaussian policy parametrization is Fisher-non-degenerate, meaning that the Fisher information matrix induced by the policy parametrization is (uniformly) positive definite. 
This assumption is standard in the literature~\citep{liu-et-al20,ding-et-al22,Vanilla_PL_Yuan_21,Masiha_SCRN_KL,KL_PAGER_Fatkhullin}. We remark that \citet{Fatkhullin_SPG_FND_2023} recently obtained a~$\cO(\varepsilon^{-2})$ sample complexity under similar assumptions using a similar proof technique. The key difference between our \algname{N-VR-PG} method and their \algname{(N)-HARPG} algorithm is that our algorithm does not require the use of second-order information. The bound of IS weights is automatically ensured by the normalization step of the algorithm and the specific structure of the Gaussian policy parametrization (Lemma~\ref{lem:bounded-var-is-weights-gaussian}).

\begin{assumption}\label{hyp:fisher-non-degeneracy} There exists~$\mu_F > 0$ such that for every~$\theta \in \R^d$, the Fisher information matrix satisfies
\begin{equation*}
F_{\rho}(\theta) \eqdef \bb E_{s \sim d_{\rho}^{\pi_\theta},\, a \sim \pi_{\theta}(\cdot|s)} [\nabla \log \pi_{\theta}(a|s) \nabla \log \pi_{\theta}(a|s)^{\top}] \succeq \mu_F I_d \,,
\end{equation*}
where~$d_{\rho}^{\pi_\theta}(\cdot) \eqdef (1-\gamma) \sum_{t = 0}^{\infty} \gamma^t \bb P_{\rho,\pi_{\theta}}(s_t \in \cdot)$ is the discounted state visitation measure.
\end{assumption}

For Gaussian policies with fixed covariance matrix and linear mean parametrization $\mu_{\theta}(s) = \phi(s)^{\top} \theta $, the Fisher information matrix can be written explicitly. Namely, we have $F_{\rho}(\theta) = \sigma^{-2} \phi(s) \phi(s)^{\top}$ for every $s\in \mathcal{S}$. Therefore, the above assumption is satisfied if we assume that the feature map $\phi(s)$ has full-row-rank. 

Now we introduce an assumption which characterizes the expressivity of our policy parameterization class via the framework of compatible function approximation \citep{PGM_Sutton_1999,agarwal-et-al21}. In order to state this assumption, we first define the advantage function. Define for every policy~$\pi$ the state-action value function~$Q^{\pi}: \mathcal{S} \times \mathcal{A} \to \R$ for every~$s \in \mathcal{S}, a \in \mathcal{A}$ as: 
\begin{equation*}
Q^{\pi}(s,a) \eqdef \mathbb E_{\pi}\left[\sum_{t=0}^{\infty} \gamma^t r(s_t,a_t) | s_0 = a, a_0 = a \right]\,.
\end{equation*}
Under the same policy~$\pi$, the state-value function~$V^{\pi}: \mathcal{S} \to \R$ and the advantage function~$A^{\pi}:  \mathcal{S} \times \mathcal{A} \to \R$ are defined for every~$s \in \mathcal{S}, a \in \mathcal{A}$ as follows:
\begin{align*}
V^{\pi}(s) &\eqdef \mathbb E_{a \sim \pi(\cdot|s)}[Q^{\pi}(s,a)]\,, \\ 
A^{\pi}(s, a) &\eqdef Q^{\pi}(s, a) - V^{\pi}(s)\,.
\end{align*}

Now we are ready to state the \textit{compatible function approximation error} assumption.
\begin{assumption}
\label{hyp:tranf-compatib-fun-approx}
There exists~$\varepsilon_{\text{bias}} \geq 0$ s.t. for every~$\theta \in \R^d$, the transfer
error satisfies: 
\begin{equation*}
\bb E[ (A^{\pi_{\theta}}(s,a) - (1-\gamma)w^*(\theta)^{\top}\, \nabla \log \pi_{\theta}(a|s))^2] \leq \varepsilon_{\text{bias}}\,,
\end{equation*}
where~$A^{\pi_{\theta}}$ is the advantage function, 
$w^*(\theta) \eqdef F_{\rho}(\theta)^{\dag} \nabla J(\theta)$ where~$F_{\rho}(\theta)^{\dag}$ is the pseudo-inverse of the matrix~$F_{\rho}(\theta)$ and expectation is taken over $ s \sim d_{\rho}^{\pi^*},\, a \sim \pi^*(\cdot|s)$ where~$\pi^*$ is an optimal policy (maximizing~$J(\pi)$). 
\end{assumption}
The above assumption requires that the policy parametrization~$\pi_{\theta}$ should be able to approximate the advantage function~$A^{\pi_{\theta}}$ by the score function~$\nabla \log \pi_{\theta}$. 
Naturally~$\varepsilon_{\text{bias}}$ is necessarily positive for a parameterization~$\pi_{\theta}$ that does not cover the set of all stochastic policies and~$\varepsilon_{\text{bias}}$ is small for a rich neural policy \citep{wang-et-al20}. We note that this is a common assumption which was used for instance in \citep{agarwal-et-al21,liu-et-al20, ding-et-al22,Vanilla_PL_Yuan_21}.

Equipped with Assumptions~\ref{hyp:gauss-policy-param}, \ref{hyp:fisher-non-degeneracy}, \ref{hyp:tranf-compatib-fun-approx}, and following the derivations of \citet{ding-et-al22}, we obtain a relaxed weak gradient dominance inequality. 

\begin{lemma}[Relaxed weak gradient domination, \citep{ding-et-al22}] 
\label{lem:relaxed-w-grad-dom}
Let Assumptions~\ref{hyp:gauss-policy-param}, \ref{hyp:fisher-non-degeneracy} and~\ref{hyp:tranf-compatib-fun-approx} hold. 
Then%
\begin{eqnarray}
\label{eq:relaxed-w-grad-dom}
\forall\, \theta \in \R^d, \quad 
\varepsilon' + \|\nabla J(\theta)\| \geq \sqrt{2\mu}\, (J^* - J(\theta))\,,
\end{eqnarray}
where~$J^*$ is the optimal expected return, $\varepsilon^{\prime} = \frac{\mu_F \sqrt{\varepsilon_{\text{bias}}}}{M_g (1-\gamma)}\,$ and~$\mu = \frac{\mu_F^2}{2M_g^2}.$ 
\end{lemma}

\begin{corollary}[\textbf{Global convergence of \algname{N-VR-PG}}]
    \label{thm:N-VR-PG-gauss}
    Let Assumptions~\ref{hyp:gauss-policy-param}, \ref{hyp:fisher-non-degeneracy} and \ref{hyp:tranf-compatib-fun-approx} hold. 
    Set~$\al_t = \fr{3}{\sqrt{2\mu} (T+1)^a}$, for some $0 < a < 1$ , $\eta_t = \fr{2}{t+1} $ and~$H = \rb{1-\g}^{-1}{\log(T + 1)}$. Then the output~$\theta_T$ of \algname{N-VR-PG} (see Algorithm~\ref{alg:N-VR-PG}) satisfies
    $$
        J^* - \Exp{ J(\theta_T) } \leq \cO\left( \fr{ 1 }{ (1-\gamma)^{\nfr{3}{2}}  (T+1)^{a - \fr{1}{2}} }  \right)  + \fr{\sqrt{\varepsilon_{\text{bias}}} }{1-\gamma}  ,
    $$
    where~$J^*$ is the optimal expected return. Therefore, the sample complexity to achieve $J^* - \Exp{ J(\theta_T) } \leq \varepsilon + \fr{\sqrt{\varepsilon_{\text{bias}}} }{1-\gamma}$ is $\cO\rb{ \varepsilon^{\fr{-2}{2a - 1}}}$.
\end{corollary}

\begin{proof}   
    As in the case of softmax parametrization, given the result of Lemma~\ref{lem:for-cum-sum-gaussian}, and following the steps in the proof of Lemma~\ref{lem:et_bound_cumul}, we can derive the control of the variance sequence for $\eta_t = \fr{2}{t+1}$ as
\begin{equation}
\Exp{ \norm{e_t} } \leq   4 V \eta_t \cdot t^{\nfr{1}{2}} + 2(L_g + G C_w^{\nfr{1}{2}} )  \al_{t-1} \cdot t^{\nfr{1}{2}} .
\end{equation}

 Similarly to Lemma~\ref{lem:ascent-like-lemma-gen-ut}, we can obtain 
    \begin{eqnarray}
    J(\theta_{t+1})  &\geq& J(\theta_t) + \frac{\al_t}{3} \|\nabla J(\theta_t)\| - 2 \al_t \|e_t\| - \frac{4}{3} D_g \gamma^H \al_t - \frac{L_{\theta}}{2} \al_t^2 \, \notag \\
     &\geq& J(\theta_t) + \frac{\al_t \sqrt{2\mu}}{3} \rb{ J^* - J(\theta_t) } - 2 \al_t \|e_t\| - \frac{4}{3} D_g \gamma^H \al_t - \frac{L_{\theta}}{2} \al_t^2 - \fr{\varepsilon' \al_t}{3}, 
    \end{eqnarray}
    where in the last step we applied the relaxed weak gradient dominance condition (Lemma~\ref{lem:relaxed-w-grad-dom}). Now we define $\delta_t \eqdef \Exp{J(\theta^*) - J(\theta_t)}$. Let $\al_t = \al $ for all $t = 0, \ldots, T$. Then applying full expectation to the result of Lemma~\ref{lem:ascent_like_hidden_conv}, we have
\begin{eqnarray}
    \delta_{t+1} &\leq& \rb{1 - \frac{\al \sqrt{2\mu}}{3} } \delta_t  + 2 \al \Exp{ \|e_t\| } + \frac{4}{3} D_g \gamma^H \al + \frac{L_{\theta}}{2} \al^2 + \fr{\varepsilon' \al}{3} \notag .
    \end{eqnarray}
    Unrolling the recursion from $t = 0$ to $t = T-1$, we have
    \begin{eqnarray}
\delta_{T} &\leq& \rb{ 1-\frac{\al \sqrt{2\mu}}{3} }^T \delta_0 + 2 \al \sum_{t=0}^{T-1}  \Exp{ \norm{ e_t } } + \frac{4}{\sqrt{2\mu}} D_g \gamma^H + \frac{3}{\sqrt{2\mu}} \frac{L_{\theta}}{2} \al + \fr{\varepsilon' }{\sqrt{2\mu}} \, \notag \\
&\leq& \rb{ 1-\frac{\al \sqrt{2\mu}}{3} }^T \delta_0 + 8 V \al \sum_{t=0}^{T-1}   \eta_t \cdot t^{\nfr{1}{2}}   + 4 (L_g + G C_w^{\nfr{1}{2}} ) \al^2  T^{\nfr{1}{2}} + \frac{4}{\sqrt{2\mu}} D_g \gamma^H + \frac{3}{\sqrt{2\mu}} \frac{L_{\theta}}{2} \al + \fr{\varepsilon' }{\sqrt{2\mu}} \, \notag \\
&\leq& \rb{ 1-\frac{\al \sqrt{2\mu}}{3} }^T \delta_0 +   8 V \al (T+1)^{\nfr{1}{2}}   + 4 (L_g + G C_w^{\nfr{1}{2}} ) \al^2  T^{\nfr{1}{2}} + \frac{4}{\sqrt{2\mu}} D_g \gamma^H + \frac{3}{\sqrt{2\mu}} \frac{L_{\theta}}{2} \al + \fr{\varepsilon' }{\sqrt{2\mu}}  \, \notag .
\end{eqnarray}
Finally, setting $\al = \fr{3}{\sqrt{2\mu} (T+1)^a}$, for $0 < a < 1$ and noticing that $(1-(T+1)^{-a})^T \leq \exp\rb{T \log(1-(T+1)^{-a}) } \leq \exp\rb{ - T^{1-a} }$, we obtain 
\begin{eqnarray}
    \delta_T & \leq & \exp\rb{ - T^{1-a} } \delta_0 + \fr{24 V}{\sqrt{2\mu}  (T+1)^{a-\fr{1}{2} }}    + \fr{18 (L_g + G C_w^{\nfr{1}{2}} ) }{ \mu \cdot  T^{2a - \fr{1}{2}} } + \frac{4}{\sqrt{2\mu}} D_g \gamma^H + \frac{9 L_{\theta} }{8\mu}  \fr{1}{(T+1)^a} + \fr{\varepsilon' }{\sqrt{2\mu}}  \, \notag \\
    &\leq& \cO\rb{ \fr{ V }{ \sqrt{\mu} (T+1)^{a - \fr{1}{2}} }  }  + \fr{\varepsilon' }{\sqrt{2\mu}},  \, \notag 
\end{eqnarray}
where the last step follows by setting $H = (1-\gamma)^{-1} \log(T)$. It only remains to notice from Lemma~\ref{lem:for-cum-sum-gaussian} and \ref{lem:relaxed-w-grad-dom} that 
$$
\varepsilon^{\prime} = \frac{\mu_F \sqrt{\varepsilon_{\text{bias}}}}{M_g (1-\gamma)} = \cO\rb{\fr{1}{1-\gamma} } ,\qquad V = \frac{ M_g \|r\|_{\infty}}{(1-\gamma)^{\nfr{3}{2}}} = \cO\rb{\fr{1}{(1-\gamma)^{\nfr{3}{2}}} }.
$$
\end{proof}

\section{Proofs for Section~\ref{sec:large-state-action-space}: Large state-action space setting}
\label{appendix:large-sa-setting}

 \subsection{Unbiased estimates of the occupancy measure at state-action pairs}
 \label{subsec:unbiased-estim-occup-measure}

\noindent\textbf{Notation.} For a given set~$A$, the indicator function~$\mathds{1}_A$ is equal to one on the set~$A$ and zero otherwise. 

In this section, we provide two different estimators: the first one is a Monte-Carlo estimate of the truncated occupancy measure whereas the second one is an unbiased estimate of the true occupancy measure. Notice that we can also slightly modify the second estimator to a obtain a minibatch estimator via sampling (independently) similarly~$N$ different state-action pairs~$(s_H^{(i)}, a_H^{(i)})_{0 \leq i \leq N}$ via the same sampling procedure as in Algorithm~\ref{algo:geom-rollout-estimate-true-occup-measure} and averaging out the outputs, i.e., considering the following estimator: 
\begin{equation}
\hat{\lambda}^{\pi_{\theta}}(s,a) = \frac{1}{N} \sum_{i=1}^N  \mathds{1}_{\{s_H^{(i)} = s,\, a_H^{(i)} = a \}}\,.
\end{equation}

\begin{algorithm}[H]
   \caption{Monte-Carlo estimate of the truncated state-action occupancy measure for $(s,a)$:~$\lambda_H^{\pi_{\theta}}(s,a)$}
   \label{algo:MC-estimate-lambda-pi-theta}
\begin{algorithmic}
   \STATE {\bfseries Input:} Initial state distribution~$\rho$, state-action pair~$(s,a) \in \mathcal{S} \times \mathcal{A}$, policy~$\pi_{\theta}$, discount factor~$\gamma \in [0,1)$, truncation horizon~$H$. 
   
   \STATE Sample a trajectory~$\tau = (s_t,a_t)_{0 \leq t \leq H-1}$ from the MDP controlled by policy~$\pi_{\theta}$
    \STATE $\hat{\lambda}_H^{\pi_{\theta}}(s,a) = \sum_{t=0}^H \gamma^t 
        \mathds{1}_{\{s_t = s,\, a_t = a \}}$
   \STATE {\bfseries Return:} $\hat{\lambda}_H^{\pi_{\theta}}(s,a)\,.$
\end{algorithmic}
\end{algorithm}

\begin{algorithm}[H]
   \caption{Unbiased estimator of the state-action occupancy measure for $(s,a)$:~$\lambda^{\pi_{\theta}}(s,a)$}
   \label{algo:geom-rollout-estimate-true-occup-measure}
\begin{algorithmic}
   \STATE {\bfseries Input:} Initial state distribution~$\rho$, state-action pair~$(s,a) \in \mathcal{S} \times \mathcal{A}$, policy~$\pi_{\theta}$, discount factor~$\gamma \in [0,1)$, $h= 0$.  
   \STATE $s_0 \sim \rho, a_0 \sim \pi_{\theta}(\cdot|s_0)$
   \STATE Draw~$H$ from the geometric distribution~$\text{Geom}(1-\gamma)$
   \FOR{$h= 0, \ldots, H-1$}
        \STATE $s_{h+1} \sim P(\cdot|s_h, a_h)$; $a_{h+1} \sim \pi_{\theta}(\cdot|s_h)$
    \ENDFOR
   \STATE $\hat{\lambda}^{\pi_{\theta}}(s,a) = \mathds{1}_{\{s_H = s,\, a_H = a \}} $
   \STATE {\bfseries Return:} $\hat{\lambda}^{\pi_{\theta}}(s,a)\,.$
\end{algorithmic}
\end{algorithm}

\subsection{Proof of Theorem~\ref{thm:fosp-gen-ut-lin-fa2-eps-stat-approx}: Convergence analysis under bounded statistical and approximation errors} 
\label{app:proof-thm:fosp-gen-ut-lin-fa2-eps-stat-approx}

We first state a more detailed version of Theorem~\ref{thm:fosp-gen-ut-lin-fa2-eps-stat-approx}. 

\begin{theorem}
\label{thm:fosp-gen-ut-lin-fa2-eps-stat-approx2}
Let Assumptions~\ref{hyp:policy-param}, \ref{hyp:smoothness-F}, \ref{hyp:bounded-stat-error} and~\ref{hyp:bounded-approx-error} hold true. In addition, suppose that there exists~$\rho_{\min} > 0$ s.t. the initial distribution~$\rho$ satisfies~$\rho(s) \geq \rho_{\min}$ for all~$s \in \mathcal{S}\,.$ Let~$T \geq 1$ be an integer and let~$(\theta_t)$ be the sequence generated by Algorithm~\ref{algo-gen-ut-func-approx2} with a positive step size~$\alpha \leq \min(1/\sqrt{5 \tilde{C}_1}, 1/2L_{\theta})$ (see~$\tilde{C}_1$ below) and batch size~$N \geq 1$.
Then, we have 
\begin{equation}
 \mathbb{E}[\|\nabla_{\theta} F(\lambda(\bar{\theta}_T))\|^2] 
\leq \frac{16  (F^* - \mathbb{E}[F(\lambda(\theta_1))]) + \alpha \tilde{C}_4}{\alpha T}
+ \frac{\tilde{C}_3}{N} + 2 D_{\lambda}^2 \gamma^{2H}
+ \tilde{C}_2 (\epsilon_{\text{stat}} + \epsilon_{\text{approx}})\,, 
\end{equation}
where~$\bar{\theta}_T$ be a random iterate drawn uniformly at random from~$\{\theta_1, \cdots, \theta_T\}$, $\tilde{C}_1 \eqdef \frac{48 l_{\psi}^3 L_{\lambda,\infty}^2}{(1-\gamma)^6}, \tilde{C}_2 \eqdef \frac{48 l_{\psi}^2 L_{\lambda}^2}{(1-\gamma)^4} \frac{|\mathcal{A}|}{\rho_{\min}}$, $\tilde{C}_3 \eqdef \frac{24 l_{\lambda}^2 l_{\psi}^2}{(1-\gamma)^4}, \tilde{C}_4 \eqdef \frac{8 l_{\lambda}^2 l_{\psi}^2}{(1-\gamma)^4}$ and~$D_{\lambda}$ is defined in Lemma~\ref{lem:trunc_grad}.  
\end{theorem}

\begin{proof}
We introduce the shorthand notation~$u_t \eqdef \frac{1}{N}\sum_{i=1}^N g(\tau_t^{(i)}, \theta_t, r_{t-1})$ for this proof. 
The smoothness of the objective function~$\theta \mapsto F(\lambda(\theta))$ (see Lemma~\ref{lem:smoothness-obj}) together with the update rule of the sequence~$(\theta_t)$ yields  
\begin{align}
\label{eq:smoothness-batch}
F(\lambda(\theta_{t+1})) 
&\geq F(\lambda(\theta_t)) + \ps{\nabla_{\theta} F(\lambda(\theta_t)), \theta_{t+1} - \theta_t} - \frac{L_{\theta}}{2} \|\theta_{t+1} - \theta_t \|^2\nonumber\\
&= F(\lambda(\theta_t)) + \alpha \ps{\nabla_{\theta} F(\lambda(\theta_t)), u_t} - \frac{L_{\theta} \alpha^2}{2}\|u_t\|^2\nonumber\\
&= F(\lambda(\theta_t))+ \alpha \ps{\nabla_{\theta} F(\lambda(\theta_t)) - u_t, u_t} + \alpha \left( 1- \frac{L_{\theta}\alpha}{2}\right) \|u_t\|^2\nonumber\\
&\geq F(\lambda(\theta_t)) - \frac{\alpha}{2} \|\nabla_{\theta} F(\lambda(\theta_t)) - u_t\|^2 - \frac{\alpha}{2} \|u_t\|^2 + \alpha \left( 1- \frac{L_{\theta}\alpha}{2}\right) \|u_t\|^2\nonumber\\
&= F(\lambda(\theta_t)) - \frac{\alpha}{2} \|\nabla_{\theta} F(\lambda(\theta_t)) - u_t\|^2 + \frac{\alpha}{2}  (1- L_{\theta}\alpha) \|u_t\|^2\nonumber\\
&\stackrel{(i)}{\geq} F(\lambda(\theta_t)) - \frac{\alpha}{2} \|\nabla_{\theta} F(\lambda(\theta_t)) - u_t\|^2 + \frac{\alpha}{4} \|u_t\|^2 \nonumber\\ 
&= F(\lambda(\theta_t)) - \frac{\alpha}{2} \|\nabla_{\theta} F(\lambda(\theta_t)) - u_t\|^2 + \frac{\alpha}{8} \|u_t\|^2 + \frac{\alpha}{8} \|u_t\|^2 \nonumber\\
&\stackrel{(ii)}{\geq} F(\lambda(\theta_t)) + \frac{\alpha}{16} \|\nabla_{\theta} F(\lambda(\theta_t))\|^2 - \frac{5}{8} \alpha \|\nabla_{\theta} F(\lambda(\theta_t)) - u_t\|^2 + \frac{\alpha}{8} \|u_t\|^2\,, 
\end{align}
where~(i) follows from the condition~$\alpha \leq 1/2L_{\theta}$ and~(ii) from~$\frac{1}{2}\|\nabla_{\theta}F(\lambda(\theta_t))\|^2 \leq \|u_t\|^2 + \|\nabla_{\theta}F(\lambda(\theta_t)) - u_t\|^2$\,.

We now control the last error term in the above inequality in expectation. 
Observe first that 
\begin{align}
\label{eq:bound1-trunc-error}
\mathbb{E}[\|\nabla_{\theta}F(\lambda(\theta_t)) - u_t\|^2] 
&\leq 2 \, \mathbb{E}[\|\nabla_{\theta}F(\lambda(\theta_t)) - \nabla_{\theta}F(\lambda_H(\theta_t))\|^2 ] + 2 \mathbb{E}[\|\nabla_{\theta}F(\lambda_H(\theta_t)) - u_t\|^2]\nonumber\\
&\leq 2 D_{\lambda}^2 \gamma^{2H} + 2 \mathbb{E}[\|\nabla_{\theta}F(\lambda_H(\theta_t)) - u_t\|^2]\,,
\end{align}
where the last inequality stems from Lemma~\ref{lem:trunc_grad}. 
Now, it remains to control~$\mathbb{E}[\|\nabla_{\theta}F(\lambda_H(\theta_t)) - u_t\|^2]\,.$ Using the notation~$r_t \eqdef \nabla_{\lambda} F(\lambda_H(\theta_t))$, we have the following decomposition:  
\begin{align}
\nabla_{\theta}F(\lambda_H(\theta_t)) - u_t
&= \nabla_{\theta}F(\lambda_H(\theta_t)) - [\nabla_{\theta}\lambda(\theta_t)]^T r_{t-1}^{*} + [\nabla_{\theta}\lambda(\theta_t)]^T r_{t-1}^{*} - [\nabla_{\theta}\lambda(\theta_t)]^T r_{t-1} + [\nabla_{\theta}\lambda(\theta_t)]^T r_{t-1} - u_t\,.
\end{align}
Then it follows that 
\begin{multline}
\label{eq:decomp-3-terms}
\mathbb{E}[\|\nabla_{\theta}F(\lambda_H(\theta_t)) - u_t\|^2] 
\leq 3 \mathbb{E}[\|\nabla_{\theta}F(\lambda_H(\theta_t)) - [\nabla_{\theta}\lambda(\theta_t)]^T r_{t-1}^{*}\|^2] 
+  3 \mathbb{E}[\|[\nabla_{\theta}\lambda(\theta_t)]^T (r_{t-1}^{*} - r_{t-1})\|^2]\\
+  3 \mathbb{E}[\|[\nabla_{\theta}\lambda(\theta_t)]^T r_{t-1} - u_t\|^2]\,.
\end{multline}

We control each term in the above decomposition separately in what follows. 

\noindent\textbf{Term 1 in~\eqref{eq:decomp-3-terms}:} For this term,
we have the following series of inequalities 
\begin{align}
\label{eq:term1-bound-lin-fa}
\|\nabla_{\theta}F(\lambda_H(\theta_t)) - [\nabla_{\theta}\lambda(\theta_t)]^T r_{t-1}^{*}\|^2
&= \|[\nabla_{\theta}\lambda(\theta_t)]^T (r_t^{*} - r_{t-1}^{*})\|^2 \nonumber\\
& \stackrel{(a)}{\leq} \frac{4 l_{\psi}^2}{(1-\gamma)^4} \|r_{t-1}^* - r_t^*\|_{\infty}^2\nonumber\\
&\stackrel{(b)}{\leq} \frac{4 l_{\psi}^2 L_{\lambda,\infty}^2}{(1-\gamma)^4} \|\lambda_H(\theta_{t-1}) - \lambda_H(\theta_t)\|_1^2\nonumber\\
&\stackrel{(c)}{\leq} \frac{8 l_{\psi}^3 L_{\lambda,\infty}^2}{(1-\gamma)^6} \|\theta_t - \theta_{t-1}\|^2\nonumber\\
&\stackrel{(d)}{=} \frac{8 l_{\psi}^3 L_{\lambda,\infty}^2}{(1-\gamma)^6} \|u_{t-1}\|^2 \cdot \alpha^2\,, %
\end{align}
where (a)~follows from similar derivations to~\eqref{eq:grad-theta-lambda-grad-lambda-diff}-\eqref{eq:bound-pg-thm-temp-35} using~\eqref{eq:expected-reinforce}, (b)~stems from Assumption~\ref{hyp:smoothness-F}, (c)~is an immediate consequence of Lemma~\ref{lem:smoothness-obj}-\ref{lem:smoothness-obj-ii} and (d)~uses the update rule of Algorithm~\ref{algo-gen-ut-func-approx2}.

\noindent\textbf{Term 2 in~\eqref{eq:decomp-3-terms}:} 
For this term, we start with the following inequalities: 
\begin{align}
\label{eq:term2-bound-lin-fa-prelim}
\mathbb{E}[\|[\nabla_{\theta}\lambda(\theta_t)]^T (r_{t-1}^{*} - r_{t-1})\|^2]
&\stackrel{(i)}{\leq} \frac{4l_{\psi}^2}{(1-\gamma)^4} \mathbb{E}[\|r_{t-1} - r_{t-1}^*\|_{\infty}^2] \nonumber\\
&\stackrel{(ii)}{\leq} \frac{4l_{\psi}^2L_{\lambda}^2}{(1-\gamma)^4} \mathbb{E}[\|\hat{\lambda}_{t-1} - \lambda_H(\theta_{t-1})\|^2]\,,
\end{align}
where (i)~follows from similar derivations to~\eqref{eq:grad-theta-lambda-grad-lambda-diff}-\eqref{eq:bound-pg-thm-temp-35} using~\eqref{eq:expected-reinforce} and (ii)~follows from Assumption~\ref{hyp:smoothness-F}. Then we decompose and upper bound the above error as follows: 
\begin{align}
\label{eq:splitting-stat-approx-errors}
\mathbb{E}[\|\hat{\lambda}_{t-1} - \lambda_H(\theta_{t-1})\|^2] 
&= \mathbb{E}[\|\ps{\phi(\cdot, \cdot), \hat{\omega}_{\theta_{t-1}}} - \lambda_H(\theta_{t-1})\|^2] \nonumber\\
&= \mathbb{E}[\|\ps{\phi(\cdot, \cdot), \hat{\omega}_{\theta_{t-1}} - \omega_{*}(\theta_{t-1})} + \ps{\phi(\cdot, \cdot), \omega_{*}(\theta_{t-1})} - \lambda_H(\theta_{t-1})\|^2] \nonumber\\
&\leq 2\, \mathbb{E}[\|\ps{\phi(\cdot, \cdot), \hat{\omega}_{\theta_{t-1}} - \omega_{*}(\theta_{t-1})}\|^2] 
+ 2\, \mathbb{E}[\|\ps{\phi(\cdot, \cdot), \omega_{*}(\theta_{t-1})} - \lambda_H(\theta_{t-1})\|^2]\,.%
\end{align}
Our task now is to upper bound each one of the above errors, the first one being related to the statistical error whereas the second one relates to the approximation error. Recall the definition of the regression loss function for every~$\theta \in \R^d, \omega \in \R^{m}$, 
\begin{equation}
L_{\theta}(\omega) = \mathbb{E}_{s \sim \rho, a \sim \mathcal{U}(\mathcal{A})}[(\lambda_H^{\pi_{\theta}}(s,a) - \ps{\phi(s,a), \omega})^2]\,,
\end{equation}
where~$\mathcal{U}(\mathcal{A})$ is the uniform distribution over the action space~$\mathcal{A}\,.$ 

\noindent\textbf{(a) Bounding term 1 in~\eqref{eq:splitting-stat-approx-errors} by the statistical error.}
First, observe for this term that 
\begin{align}
\label{eq:error-to-bound-term1}
\mathbb{E}[\|\ps{\phi(\cdot, \cdot), \hat{\omega}_{\theta_{t-1}} - \omega_{*}(\theta_{t-1})}\|^2] 
&= \mathbb{E}\left[\sum_{s \in \mathcal{S}, a \in \mathcal{A}} \ps{\phi(s, a), \hat{\omega}_{\theta_{t-1}} - \omega_{*}(\theta_{t-1})}^2\right]\nonumber\\
&\leq \frac{|\mathcal{A}|}{\rho_{\min}} \mathbb{E}\left[\sum_{s \in \mathcal{S}, a \in \mathcal{A}} \frac{\rho(s)}{|\mathcal{A}|} \ps{\phi(s, a), \hat{\omega}_{\theta_{t-1}} - \omega_{*}(\theta_{t-1})}^2\right]\nonumber\\
&= \frac{|\mathcal{A}|}{\rho_{\min}} \mathbb{E}\left[\mathbb{E}_{s \sim \rho, a \sim \mathcal{U}(\mathcal{A})}[\ps{\phi(s, a), \hat{\omega}_{\theta_{t-1}} - \omega_{*}(\theta_{t-1})}^2] \right]\,.
\end{align}

Then, we have for all~$\omega \in \R^m$,
\begin{align}
\label{eq:relate-to-stat-error}
&L_{\theta_{t-1}}(\omega) - L_{\theta_{t-1}}(\omega_*(\theta_{t-1})) \nonumber\\
&=  \mathbb{E}_{s \sim \rho, a \sim \mathcal{U}(\mathcal{A})}[(\ps{\phi(s,a), \omega} - \lambda^{\pi_{\theta_{t-1}}}(s,a))^2] - L_{\theta_{t-1}}(\omega_*(\theta_{t-1}))\nonumber\\
&=  \mathbb{E}_{s \sim \rho, a \sim \mathcal{U}(\mathcal{A})}[(\ps{\phi(s,a), \omega - \omega_*(\theta_{t-1})} + \ps{\phi(s,a),\omega_*(\theta_{t-1})} - \lambda^{\pi_{\theta_{t-1}}}(s,a))^2] - L_{\theta_{t-1}}(\omega_*(\theta_{t-1}))\nonumber\\
&= \mathbb{E}_{s \sim \rho, a \sim \mathcal{U}(\mathcal{A})}[\ps{\phi(s,a), \omega - \omega_*(\theta_{t-1})}^2] + 2 \ps{\omega - \omega_*(\theta_{t-1}), \mathbb{E}_{s \sim \rho, a \sim \mathcal{U}(\mathcal{A})}[(\ps{\phi(s,a), \omega_*(\theta_{t-1})} - \lambda^{\pi_{\theta_{t-1}}}(s,a) )\phi(s,a)] } \nonumber\\ 
&= \mathbb{E}_{s \sim \rho, a \sim \mathcal{U}(\mathcal{A})}[\ps{\phi(s,a), \omega - \omega_*(\theta_{t-1})}^2] + \ps{\omega - \omega_*(\theta_{t-1}), \nabla_{\omega} L_{\theta_{t-1}}(\omega_*(\theta_{t-1})) } \nonumber\\
&\geq \mathbb{E}_{s \sim \rho, a \sim \mathcal{U}(\mathcal{A})}[\ps{\phi(s,a), \omega - \omega_*(\theta_{t-1})}^2]\,,
\end{align}
where the last inequality stems from the first-order optimality condition for~$\omega_*(\theta_{t-1}) \in \argmin_{\omega} L_{\theta_{t-1}}(\omega)\,,$ which gives the inequality~$\ps{\omega - \omega_*(\theta_{t-1}), \nabla_{\omega} L_{\theta_{t-1}}(\omega_*(\theta_{t-1})) } \geq 0$ for every~$\omega \in \R^m\,.$

Combining~\eqref{eq:error-to-bound-term1} with~\eqref{eq:relate-to-stat-error} and using Assumption~\ref{hyp:bounded-stat-error}, we obtain 
\begin{equation}
\label{eq:bound-term1-eps-stat-bound}
\mathbb{E}[\|\ps{\phi(\cdot, \cdot), \hat{\omega}_{\theta_{t-1}} - \omega_{*}(\theta_{t-1})}\|^2] 
\leq \frac{|\mathcal{A}|}{\rho_{\min}} \mathbb{E}[L_{\theta_{t-1}}(\hat{\omega}_{\theta_{t-1}}) - L_{\theta_{t-1}}(\omega_*(\theta_{t-1}))] \leq \frac{|\mathcal{A}|}{\rho_{\min}} \epsilon_{\text{stat}}\,.
\end{equation}

\noindent\textbf{(b) Bounding term 2 in~\eqref{eq:splitting-stat-approx-errors} by the approximation error.}
Similar derivations as for the previous term yield
\begin{align}
\label{eq:bound-term2-eps-approx-bound}
\mathbb{E}[\|\ps{\phi(\cdot, \cdot), \omega_{*}(\theta_{t-1})} - \lambda_H(\theta_{t-1})\|^2] 
&= \mathbb{E}\left[\sum_{s \in \mathcal{S}, a \in \mathcal{A}} (\ps{\phi(s, a), \omega_{*}(\theta_{t-1})} - \lambda_H^{\pi_{\theta_{t-1}}}(s,a))^2\right]\nonumber\\
&\leq \frac{|\mathcal{A}|}{\rho_{\min}}  \mathbb{E}\left[\sum_{s \in \mathcal{S}, a \in \mathcal{A}} \frac{\rho(s)}{|\mathcal{A}|} (\ps{\phi(s, a), \omega_{*}(\theta_{t-1})} - \lambda_H^{\pi_{\theta_{t-1}}}(s,a))^2 \right] \nonumber\\
&=  \frac{|\mathcal{A}|}{\rho_{\min}} \mathbb{E}\left[ \mathbb{E}_{s \sim \rho, a \sim \mathcal{U}(\mathcal{A})}[ (\ps{\phi(s, a), \omega_{*}(\theta_{t-1})} - \lambda_H^{\pi_{\theta_{t-1}}}(s,a))^2 ] \right]\nonumber\\ 
&=  \frac{|\mathcal{A}|}{\rho_{\min}} \mathbb{E}[L_{\theta_{t-1}}(\omega_*(\theta_{t-1}))] \nonumber\\
& \leq \frac{|\mathcal{A}|}{\rho_{\min}} \epsilon_{\text{approx}}\,.
\end{align}

Combining~\eqref{eq:term2-bound-lin-fa-prelim}, \eqref{eq:splitting-stat-approx-errors}, \eqref{eq:bound-term1-eps-stat-bound} and~\eqref{eq:bound-term2-eps-approx-bound} yields
\begin{equation}
\label{eq:term2-bound-lin-fa}
\mathbb{E}[\|[\nabla_{\theta}\lambda(\theta_t)]^T (r_{t-1}^{*} - r_{t-1})\|^2] 
\leq \frac{8l_{\psi}^2L_{\lambda}^2}{(1-\gamma)^4} \frac{|\mathcal{A}|}{\rho_{\min}} (\epsilon_{\text{stat}} + \epsilon_{\text{approx}})\,.
\end{equation}

\noindent\textbf{Term 3 in~\eqref{eq:decomp-3-terms}:} For this last term, we have
\begin{align}
\label{eq:term3-bound-lin-fa}
\mathbb{E}[\|[\nabla_{\theta}\lambda(\theta_t)]^T r_{t-1} - u_t\|^2] 
&\stackrel{(a)}{=} \mathbb{E}\left[\left\|\frac{1}{N} \sum_{i=1}^N  ([\nabla_{\theta}\lambda(\theta_t)]^T r_{t-1} - g(\tau_t^{(i)}, \theta_t, r_{t-1})) \right\|^2\right]  \nonumber\\
&\stackrel{(b)}{=} \frac{1}{N} \mathbb{E}[\|g(\tau_t^{(i)}, \theta_t, r_{t-1})- [\nabla_{\theta}\lambda(\theta_t)]^T r_{t-1}\|^2] \nonumber\\
&\stackrel{(c)}{\leq}  \frac{1}{N}  \mathbb{E}[\|g(\tau_t^{(i)}, \theta_t, r_{t-1})\|^2]\nonumber\\
&\stackrel{(d)}{\leq} \frac{4 l_{\lambda}^2 l_{\psi}^2}{N (1-\gamma)^4}\,,
\end{align}
where (a) stems from the definition of~$u_t$, (b) follows from using Lemma~\ref{lem:pg-estimate-lambda-estimate} and recalling that the trajectories~$(\tau_t^{(i)})_{1 \leq i \leq N}$ are independently drawn in Algorithm~\ref{algo-gen-ut-func-approx2}, (c) is due to the inequality~$\Var(X) \leq \mathbb{E}[\|X\|^2]$ for any random vector~$X \in \R^d$ and~(d) uses a similar bound to~\eqref{eq:stochastic-pg-bound}. 

Combining~\eqref{eq:decomp-3-terms}, \eqref{eq:term1-bound-lin-fa}, \eqref{eq:term2-bound-lin-fa} and~\eqref{eq:term3-bound-lin-fa} together with~\eqref{eq:bound1-trunc-error} gives 
\begin{equation}
\label{eq:grad-u-t-bound}
\mathbb{E}[\|\nabla_{\theta} F(\lambda(\theta_t)) - u_t\|^2] \leq \tilde{C}_1 \alpha^2 \|u_{t-1}\|^2 + \tilde{C}_2 (\epsilon_{\text{stat}} + \epsilon_{\text{approx}}) + \frac{\tilde{C}_3}{N} + 2 D_{\lambda}^2 \gamma^{2H}\,,
\end{equation}
where~$\tilde{C}_1 = \frac{48 l_{\psi}^3 L_{\lambda,\infty}^2}{(1-\gamma)^6}, \tilde{C}_2 = \frac{48 l_{\psi}^2 L_{\lambda}^2}{(1-\gamma)^4} \frac{|\mathcal{A}|}{\rho_{\min}}$ and~$\tilde{C}_3 = \frac{24 l_{\lambda}^2 l_{\psi}^2}{(1-\gamma)^4}\,.$

Rearranging~\eqref{eq:smoothness-batch}, dividing by~$\frac{\alpha}{16}$ and taking full expectation yields
\begin{equation}
\label{eq:rearranged-smoothness-batch}
\mathbb{E}[\|\nabla_{\theta} F(\lambda(\theta_t))\|^2] \leq \frac{16}{\alpha} (F(\lambda(\theta_{t+1})) - F(\lambda(\theta_t))) - 2 \mathbb{E}[\|u_t\|^2] + 10 \mathbb{E}[\|\nabla_{\theta} F(\lambda(\theta_t)) - u_t\|^2]\,.
\end{equation}
Plugging~\eqref{eq:grad-u-t-bound} into~\eqref{eq:rearranged-smoothness-batch}, summing the resulting inequality for~$t = 1, \cdots, T$ and dividing by~$T$ gives
\begin{multline}
\label{eq:interm-fos-before-last-step}
\frac{1}{T} \sum_{t=1}^T \mathbb{E}[\|\nabla_{\theta} F(\lambda(\theta_t))\|^2] 
\leq \frac{16}{\alpha T} \mathbb{E}[F(\lambda(\theta_{T+1})) - F(\lambda(\theta_1))] 
+ \frac{1}{T} \sum_{t=1}^T \left(10 \tilde{C}_1 \alpha^2 \mathbb{E}[\|u_{t-1}\|^2] - 2 \mathbb{E}[\|u_t\|^2]\right)\\ 
+ \tilde{C}_2 (\epsilon_{\text{stat}} + \epsilon_{\text{approx}}) + \frac{\tilde{C}_3}{N} + 2 D_{\lambda}^2 \gamma^{2H}
\end{multline}

Then, we upper bound the remaining sum in the right-hand side of~\eqref{eq:interm-fos-before-last-step} as follows: 
\begin{align}
\frac{1}{T} \sum_{t=1}^T \left(10 \tilde{C}_1 \alpha^2 \mathbb{E}[\|u_{t-1}\|^2] - 2 \mathbb{E}[\|u_t\|^2]\right) 
&= \frac{1}{T} \sum_{t=1}^T (10 \tilde{C}_1 \alpha^2 -2) \mathbb{E}[\|u_{t-1}\|^2]   +  \frac{2}{T} \sum_{t=1}^T \mathbb{E}[\|u_{t-1}\|^2 - \|u_t\|^2]\nonumber\\ 
&\stackrel{(a)}{\leq} \frac{2}{T} \sum_{t=1}^T \mathbb{E}[\|u_{t-1}\|^2 - \|u_t\|^2] \nonumber\\ 
&\stackrel{(b)}{\leq} \frac{2 \mathbb{E}[\|u_0\|^2]}{T} \nonumber\\ 
&\stackrel{(c)}{\leq} \frac{\tilde{C}_4}{T}\,, 
\end{align}
where~$\tilde{C}_4 = \frac{8 l_{\lambda}^2 l_{\psi}^2}{(1-\gamma)^4},$ (a) stems from the condition~$\alpha \leq \frac{1}{\sqrt{5 \tilde{C}_1}}$, (b) follows telescoping the sum and upper bounding the remaining resulting negative term by zero and~(c) is a consequence of a similar bound to~\eqref{eq:stochastic-pg-bound}. 

Finally, we obtain 
\begin{equation}
 \mathbb{E}[\|\nabla_{\theta} F(\lambda(\bar{\theta}_T))\|^2] 
\leq \frac{16  (F^* - \mathbb{E}[F(\lambda(\theta_1))]) + \alpha \tilde{C}_4}{\alpha T}
+ \frac{\tilde{C}_3}{N} + 2 D_{\lambda}^2 \gamma^{2H}
+ \tilde{C}_2 (\epsilon_{\text{stat}} + \epsilon_{\text{approx}})\,, 
\end{equation}
where~$\bar{\theta}_T$ be a random iterate drawn uniformly at random from~$\{\theta_1, \cdots, \theta_T\}$. This concludes the proof.
\end{proof}

\subsection{Proof of Corollary~\ref{cor:fosp-gen-ut-lin-fa2}: Sample complexity analysis}

In order to establish the total sample complexity of our algorithm, we shall use Theorem~1 in \citet{bach-moulines13} for the least-mean-square algorithm corresponding to SGD for least-squares regression to explicit the number of samples needed in the occupancy measure estimation subroutine of Algorithm~\ref{algo:sgd-subroutine2}. In other words, our objective here is to precise the number of iterations of SGD needed to approximately solve our regression problem. In particular, we will show that we can achieve~$\epsilon_{\text{stat}} = \mathcal{O}(1/K)$ where~$K$ is the number of iterations of the SGD subroutine. We first report Theorem~1 from \citet{bach-moulines13} before applying it to our specific case. 

\begin{theorem}[Theorem 1, \cite{bach-moulines13}]
\label{thm:bach-moulines13}
Let~$\mathcal{H}$ be an $m$-dimensional Euclidean space with~$m \geq 1\,.$ Let~$(x_n, z_n) \in \mathcal{H} \times \mathcal{H}$ be independent and identically distributed observations.  
Assume the following: 
\begin{enumerate}[label=(\roman*)]

\item The expectations $\mathbb{E}[\|x_n\|^2]$ and $\mathbb{E}[\|z_n\|^2]$ are finite; the covariance matrix~$\mathbb{E}[x_n x_n^T]$ is invertible.  

\item The global minimum of~$f(\omega) = \frac{1}{2} \mathbb{E}[\ps{\omega, x_n}^2 - 2 \ps{\omega, z_n}]$ is attained at a certain~$\omega_* \in \mathcal{H}$. Denoting by~$\xi_n \eqdef z_n - \ps{\omega_*, x_n} x_n$ the residual, assume that~$\mathbb{E}[\xi_n] = 0\,.$

\item There exist~$R > 0, \sigma > 0$ s.t. $\mathbb{E}[\xi_n \xi_n ^T] \preccurlyeq \sigma^2  \mathbb{E}[x_n x_n^T]$ and~$\mathbb{E}[\|x_n\|^2 x_n x_n^T] \preccurlyeq R^2  \mathbb{E}[x_n x_n^T]\,,$ where for two matrices~$A, B \in \R^{m \times m}$, $A \preccurlyeq B$ if and only if~$B-A$ is positive semi-definite. 
\end{enumerate}
Consider the Stochastic Gradient Descent (SGD) recursion started at~$\omega_0 \in \mathcal{H}$ and defined for every integer~$n \geq 1$ as 
\begin{equation}
\omega_n = \omega_{n-1} - \beta' \, (\ps{\omega_{n-1}, x_n} x_n - z_n)\,,
\end{equation}
where~$\beta' > 0\,.$ 
Then for a constant step size~$\beta' = \frac{1}{4 R^2}$ the averaged iterate~$\bar{\omega}_n \eqdef \frac{1}{n+1} \sum_{k=0}^n \omega_k$ satisfies 
\begin{equation}
\mathbb{E}[f(\bar{\omega}_n) - f(\omega_*)] \leq \frac{2}{n} (\sigma \sqrt{m} + R \|\omega_0 - \omega_* \|)^2\,.
\end{equation}
\end{theorem}

\noindent\textbf{Proof of Corollary~\ref{cor:fosp-gen-ut-lin-fa2}.}
It follows from Theorem~\ref{thm:fosp-gen-ut-lin-fa2-eps-stat-approx} that 
\begin{equation}
\label{eq:fos-eps-stat-approx-interm}
 \mathbb{E}[\|\nabla_{\theta} F(\lambda(\bar{\theta}_T))\|^2] 
\leq \frac{16  (F^* - \mathbb{E}[F(\lambda(\theta_1))]) + \tilde{C}_4}{\alpha T}
+ \frac{\tilde{C}_3}{N} + 2 D_{\lambda}^2 \gamma^{2H}
+ \tilde{C}_2 (\epsilon_{\text{stat}} + \epsilon_{\text{approx}})\,, 
\end{equation}
where~$\bar{\theta}_T$ is a random iterate drawn uniformly at random from~$\{\theta_1, \cdots, \theta_T\}$. 
We now upper bound the statistical error~$\epsilon_{\text{stat}}$ as a function of the number~$K$ of SGD iterations (see Algorithm~\ref{algo:sgd-subroutine2}) by applying Theorem~\ref{thm:bach-moulines13}. Let~$\omega_* \in \argmin_{\omega} L_{\theta}(\omega)\,$ where~$\theta \in \R^d$ is fixed (at each iteration of Algorithm~\ref{algo-gen-ut-func-approx2}). To do so, we successively verify each assumption of the latter theorem in the Euclidean space~$\R^m$. Recall from~\eqref{eq:stoch-grad-lin-reg} that the stochastic gradient of the loss function~$L_{\theta}(\omega)$ is given for every~$\theta \in \R^d, \omega \in \R^m$ by 
\begin{equation}
\hat{\nabla}_{\omega} L_{\theta}(\omega) \eqdef 2 (\ps{\phi(s,a),\omega} - \hat{\lambda}_H^{\pi_{\theta}}(s,a))\, \phi(s,a)\,.
\end{equation}
Note here that we consider the unbiased estimator~$\hat{\lambda}_H^{\pi_{\theta}}(s,a)$ of the truncated state-action occupancy measure as computed in Algorithm~\ref{algo:MC-estimate-lambda-pi-theta}.

\begin{remark}
One could also consider the unbiased estimator~$\hat{\lambda}_H^{\pi_{\theta}}(s,a)$ of the true state-action occupancy measure (without truncation) using Algorithm~\ref{algo:geom-rollout-estimate-true-occup-measure} and slightly modify the definition of the expected loss with~$\hat{\lambda}^{\pi_{\theta}}(s,a)$ instead of~$\hat{\lambda}_H^{\pi_{\theta}}(s,a)\,.$ The latter procedure would lead to the same result since the truncation error can be made as small as desired via setting the horizon large enough, the error being of the order of~$\gamma^H\,.$
\end{remark}

Take~$x_n = \phi(s,a) \in \R^m,  z_n = \hat{\lambda}_H^{\pi_{\theta}}(s,a) \phi(s,a) \in \R^m\,.$ The observations~$(x_n, z_n)$ are indeed independent and identically distributed (for each state-action pair~$(s,a)$ sample).  

\begin{enumerate}[label=(\roman*)]

\item Given Assumption~\ref{hyp:feature-map}, we have~$\mathbb{E}[\|x_n\|^2] = \mathbb{E}[\|\phi(s,a)\|^2] \leq B^2\,.$ Similarly we have$\mathbb{E}[\|z_n\|^2] \leq B^2/(1-\gamma)^2\,.$ Moreover, the covariance matrix~$\mathbb{E}[\phi(s,a)\phi(s,a)^T]$ has full rank by Assumption~\ref{hyp:feature-map}.

\item Take~$f = L_{\theta}\,.$ Define the residual~$\xi \eqdef (\hat{\lambda}_H^{\pi_{\theta}}(s,a) - \ps{\omega_*,\phi(s,a)}) \phi(s,a)\,.$ Then we conclude the verification of the second item by observing that 
\begin{equation}
\mathbb{E}[\xi] = \mathbb{E}\left[\frac{1}{2} \hat{\nabla}_{\omega} L_{\theta}(\omega_{*})\right] = \frac{1}{2}\nabla_{\omega} L_{\theta}(\omega_{*}) = 0\,,
\end{equation}
where the last identity stems from the definition of the optimal solution~$\omega_*\,.$

\item As for this last item, recall again that~$\|\phi(s,a)\| \leq B$ which immediately implies that~$\mathbb{E}[\|x_n\|^2 x_n x_n^T] \preccurlyeq R^2  \mathbb{E}[x_n x_n^T]$ with~$R = B\,.$ It remains to show that the covariance matrix of~$\xi$ satisfies~$\mathbb{E}[\xi \xi^T] \preccurlyeq \sigma^2 \mathbb{E}[\phi(s,a)\phi(s,a)^T]$ for some positive constant~$\sigma$ that we will now determine.  First, we write
\begin{align}
\mathbb{E}[\xi \xi^T] 
&= \mathbb{E}[(\hat{\lambda}_H^{\pi_{\theta}}(s,a) - \ps{\omega_*, \phi(s,a)})^2 \phi(s,a) \phi(s,a)^T] \nonumber\\
&= \mathbb{E}\left[ \mathbb{E}[ (\hat{\lambda}_H^{\pi_{\theta}}(s,a) - \ps{\omega_*, \phi(s,a)})^2 | s,a] \phi(s,a) \phi(s,a)^T \right]\,,
\end{align}
where the conditional expectation~$\mathbb{E}[\cdot|s,a]$ is w.r.t. randomness induced by sampling the state-action pair~$(s,a)\,.$ 
Then, we have for every~$s \in \mathcal{S}, a \in \mathcal{A}\,,$
\begin{equation}
\label{eq:to-bound-for-sigma}
\mathbb{E}[(\hat{\lambda}_H^{\pi_{\theta}}(s,a) - \ps{\omega_*, \phi(s,a)})^2 | s,a] 
= \mathbb{E}[(\hat{\lambda}_H^{\pi_{\theta}}(s,a))^2  - 2 \hat{\lambda}_H^{\pi_{\theta}}(s,a) \ps{\omega_*, \phi(s,a)} + \ps{\omega_*, \phi(s,a)}^2| s,a]\,.
\end{equation}
We know that~$|\hat{\lambda}_H^{\pi_{\theta}}(s,a)| \leq \frac{1}{1-\gamma}\,.$ It remains to bound~$\|\omega_*\|$ to be able to upper bound the quantity of~\eqref{eq:to-bound-for-sigma}. Recall for this that~$\nabla_{\omega} L_{\theta}(\omega_*)= 0\,$, i.e., $\mathbb{E}[(\ps{\phi(s,a), \omega_*} - \lambda_H^{\pi_{\theta}}(s,a)) \phi(s,a)] = 0\,,$ which can be rewritten as follows: 
\begin{equation*}
\mathbb{E}[ \lambda_H^{\pi_{\theta}}(s,a) \phi(s,a)] = \mathbb{E}[\phi(s,a) \phi(s,a)^T]\, \omega_*\,.
\end{equation*}
Therefore, we obtain by invoking Assumption~\ref{hyp:feature-map} that~$\omega_* = \mathbb{E}[\phi(s,a)\phi(s,a)^T]^{-1} \mathbb{E}[\lambda_H^{\pi_{\theta}}(s,a) \phi(s,a)]\,$ and hence 
\begin{equation}
\label{eq:omega-star-bound}
\|\omega_*\| \leq \frac{B}{\mu (1-\gamma)}\,.
\end{equation}
Using this inequality, it follows from~\eqref{eq:to-bound-for-sigma} that: 
\begin{align}
\mathbb{E}[(\hat{\lambda}_H^{\pi_{\theta}}(s,a) - \ps{\omega_*, \phi(s,a)})^2 | s,a] 
&\leq \frac{1}{(1-\gamma)^2} + \frac{2B}{1-\gamma} \|\omega_*\| + B^2 \|\omega_*\|^2  \nonumber\\
&\leq \frac{1}{(1-\gamma)^2} \left(  1 + \frac{2B^2}{\mu} + \frac{B^4}{\mu^2}\right)\nonumber\\ 
&= \frac{1}{(1-\gamma)^2} \left(1 + \frac{B^2}{\mu} \right)^2\,.
\end{align}
Hence, the missing part of item~(iii) is satisfied with~$\sigma = \frac{1}{1-\gamma} (1 + \frac{B^2}{\mu})\,.$ 
\end{enumerate}
We conclude the proof by using the result of Theorem~\ref{thm:bach-moulines13} with~$\beta' = 2 \beta = \frac{1}{4 B^2}$ and~$\omega_0 = 0$ to obtain after~$K$ iterations of the SGD subroutine (see Algorithm~\ref{algo:sgd-subroutine2})
\begin{align}
\mathbb{E}[L_{\theta}(\bar{\omega}_K) - L_{\theta}(\omega_*)] 
&\leq \frac{4}{K} (\sigma \sqrt{m} + R \|\omega_*\|)^2  \nonumber\\ 
&= \frac{4}{(1-\gamma)^2 K} \left(\frac{B^2}{\mu}(1+ \sqrt{m}) + \sqrt{m}\right)^2\,,
\end{align}
where~$\bar{\omega}_K$ is the output of Algorithm~\ref{algo:sgd-subroutine2}.
As a consequence, we have 
\begin{equation}
\epsilon_{\text{stat}} \leq \frac{4}{(1-\gamma)^2 K} \left(\frac{B^2}{\mu}(1+ \sqrt{m}) + \sqrt{m}\right)^2\,.
\end{equation}
Plugging this inequality into~\eqref{eq:fos-eps-stat-approx-interm} leads to 
\begin{multline}
\label{eq:fos-to-conclude-proof}
 \mathbb{E}[\|\nabla_{\theta} F(\lambda(\bar{\theta}_T))\|^2] 
\leq \frac{16  (F^* - \mathbb{E}[F(\lambda(\theta_1))]) + \tilde{C}_4}{\alpha T}
+ \frac{\tilde{C}_3}{N} + 2 D_{\lambda}^2 \gamma^{2H}\\
+ \frac{4 \tilde{C}_2}{(1-\gamma)^2 K} \left(\frac{B^2}{\mu}(1+ \sqrt{m}) + \sqrt{m}\right)^2 +  \tilde{C}_2 \epsilon_{\text{approx}}\,, 
\end{multline}

We set the number of iterations~$T$, the batch size~$N$, the number of iterations~$K$ in the subroutine of Algorithm~\ref{algo:sgd-subroutine2} and the horizon~$H$ to guarantee that~$\mathbb{E}[\|\nabla_{\theta} F(\lambda(\bar{\theta}_T))\|^2] \leq \mathcal{O}(\epsilon^2) + \mathcal{O}(\epsilon_{\text{approx}})$ where the expectation is taken over both the randomness inherent to the sequence produced by the algorithm together with the uniform sampling defining~$\bar{\theta}_T\,.$ Given~\eqref{eq:fos-to-conclude-proof}, choosing~$T = \mathcal{O}(\epsilon^{-2})$, $N = \mathcal{O}(\epsilon^{-2})$, $K = \mathcal{O}(\epsilon^{-2})$ and~$H = \mathcal{O}(\log(\frac{1}{\epsilon}))$ concludes the proof. In particular, the total sample complexity to solve the RL problem with general utilities with occupancy measure approximation in order to achieve an $\epsilon$-approximate stationary point of the objective function (up to the~$\mathcal{O}(\sqrt{\epsilon_{\text{approx}}})$ error floor) is given by~$T \times (K + N) \times H = \tilde{\mathcal{O}}(\epsilon^{-4})\,,$ where~$\tilde{O}$ hides a logarithmic factor in~$\epsilon\,.$

\section{Useful technical lemma}

In this section, we gather a few technical results that are useful throughout the proofs of our results. 

\subsection{Smoothness, Lipschitzness and truncation error technical lemmas}

The following result from \cite{zhang-et-al21}(Lemma~5.3) ensures in particular that the objective function~$\theta \mapsto F(\lambda^{\pi_{\theta}})$ is smooth which is used to derive an ascent-like lemma in our convergence analysis. 

\begin{lemma}
\label{lem:smoothness-obj}
Let Assumptions~\ref{hyp:policy-param} and~\ref{hyp:smoothness-F} hold. Then, the following statements hold: 
\begin{enumerate}[label=(\roman*)]
\item \label{lem:smoothness-obj-i} $\forall \theta \in \R^d\,, \forall (s,a) \in \mathcal{S} \times \mathcal{A},$ $\|\nabla \log \pi_{\theta}(a|s)\| \leq 2 l_{\psi}\,, \|\nabla_{\theta}^2 \log \pi_{\theta}(a|s)\| \leq 2(L_{\psi} + l_{\psi}^2)\,,$ and~$\|\nabla_{\theta} F(\lambda(\theta))\| \leq \frac{2 l_{\psi} l_{\lambda}}{(1-\gamma)^2}\,.$
\item \label{lem:smoothness-obj-ii} $\forall \theta_1, \theta_2 \in \R^d\,, \|\lambda^{\pi_{\theta_1}} - \lambda^{\pi_{\theta_2}}\|_1 \leq \frac{2 l_{\psi}}{(1-\gamma)^2} \|\theta_1 - \theta_2\|\,$ and~$\|\lambda_H(\theta_1) - \lambda_H(\theta_2)\|_1 \leq \frac{2 l_{\psi}}{(1-\gamma)^2} \|\theta_1 - \theta_2\|\,.$
\item The objective function~$\theta \mapsto F(\lambda^{\pi_{\theta}})$ is $L_{\theta}$-smooth with $L_{\theta} = \frac{4 L_{\lambda,\infty} l_{\psi}^2}{(1-\gamma)^4} + \frac{8 l_{\psi}^2 l_{\lambda}}{(1-\gamma)^3} + \frac{2 l_{\lambda} (L_{\psi} + l_{\psi}^2)}{(1-\gamma)^2}\,.$
\end{enumerate}
\end{lemma}
\begin{proof}
See Lemma~5.3 in \cite{zhang-et-al21}. The second part of item~$(ii)$ was not reported in the aforementioned reference but the proof follows the same lines upon replacing the infinite horizon but the finite one~$H$ for the truncated state-action occupancy measures. 
\end{proof}

The next lemma controls the truncation error due to truncating simulated trajectories to the horizon~$H$ in our infinite horizon setting. Notably, this error vanishes geometrically fast with the horizon~$H$\,.
\begin{lemma}%
\label{lem:trunc_grad}
    Let Assumptions~\ref{hyp:policy-param} and~\ref{hyp:smoothness-F} be satisfied. Then, we have for any~$H \geq 1$ and every~$\theta \in \R^d$: 
    \begin{enumerate}[label=(\roman*)]
        \item $\|\nabla_{\theta} F(\lambda_H(\theta)) - \nabla_{\theta} F(\lambda(\theta))\| \leq D_{\lambda} \gamma^H$ where~$D_{\lambda}^2 = \frac{8 l_{\psi}^2 L_{\lambda}^2}{(1-\gamma)^6} + 16 l_{\psi}^2 l_{\lambda}^2\left(\frac{(H+1)^2}{(1-\gamma)^2} + \frac{1}{(1-\gamma)^4}\right)\,.$ 
        \item $\norm{ \nabla J_{H}(\theta) - \nabla J(\theta) } \leq D_g \g^{H}$ with~$D_g \eqdef \frac{2 l_{\psi} \|r\|_{\infty}}{1-\gamma} \sqrt{\frac{1}{1-\gamma}+H}$ and~$r$ is the fixed reward function in the cumulative reward setting\,.
    \end{enumerate}
\end{lemma}

\begin{proof}
See Lemma~E.3 in \cite{zhang-et-al21} for the first item. The second item is standard and follows directly from using the policy gradient expression together with Lemma~\ref{lem:smoothness-obj}-(i).
\end{proof}

The following result whose proof follows immediately from \eqref{eq:pg-estimate} and Assumption~\ref{hyp:policy-param}  (see Lemma E.2, \cite{zhang-et-al21}) establishes the Lipschitz continuity of the policy gradient estimator w.r.t. the policy parameter and the reward variable. 
\begin{lemma}
\label{lem:lipschitz-pg-estimate}
Let Assumption~\ref{hyp:policy-param} hold true and 
let~$\tau = \{s_0, a_0, \cdots, s_{H-1}, a_{H-1}\}$ be an arbitrary trajectory of length~$H$. Then the following statements hold: 
\begin{enumerate}[label=(\roman*)]
    \item \label{lem:lipschitz-pg-estimate-r} $\forall \theta \in \R^d, \forall r_1, r_2 \in \R^{|\mathcal{S}| \times |\mathcal{A}|}, \|g(\tau, \theta, r_1) -g(\tau, \theta, r_2) \| \leq \frac{2 l_{\psi}}{(1-\gamma)^2} \|r_1 - r_2\|_{\infty}\,,$
    \item \label{lem:lipschitz-pg-estimate-theta} $\forall \theta_1, \theta_2 \in \R^d, \forall r \in \R^{|\mathcal{S}| \times |\mathcal{A}|}, \|g(\tau, \theta_1, r) -g(\tau, \theta_2, r) \| \leq L_{g} \|\theta_1 - \theta_2\| \,$  where~$L_{g} \eqdef \frac{2 (l_{\psi}^2 + L_{\psi}) \|r\|_{\infty}}{(1-\gamma)^2}\,,$
\end{enumerate}
\end{lemma}

\subsection{Technical lemma for solving a recursion}

The next lemma is useful for solving recursions appearing in our analysis to derive convergence rates. 

\begin{lemma}\label{le:aux_rec0}
Let $\tau$ be a positive integer and let $\cb{r_t}_{t\geq 1}$ be a non-negative sequence satisfying for every integer~$t \geq 1$
$$
r_{t}  \leq (1- \eta_t) r_{t-1} + \beta_t,
$$
where $\cb{\beta_t}_{t\geq 1}$ is a non-negative sequence. Then for $\eta_t = \fr{2 }{ t+\tau }$ we have for every integer $ T \geq 1$
$$
r_T \leq \fr{ \tau^2 r_{0} }{(T+\tau)^2} + \fr{  \sum_{t=1}^{T}\beta_t (t + \tau )^2 }{( T+\tau )^2} .
$$
\end{lemma}

\begin{proof}
    Notice that $1 - \eta_t = \fr{t+\tau -2}{t+\tau}$. Then for all $t\geq 1$
    $$
    r_{t} \leq \fr{t+\tau -2}{t+\tau} r_{t-1} +   \beta_t.
    $$ 
    Multiplying both sides by $(t+ \tau)^2$, we get
    \begin{eqnarray}
    (t+\tau)^2 r_{t} &\leq& (t+\tau -2)(t+\tau) r_{t-1} +    \beta_t (t+\tau)^2 \notag \\
    &\leq& (t+\tau -1)^2 r_{t-1} +   \beta_t (t+\tau)^2 \notag .
    \end{eqnarray}
    By summing this inequality from $t= 1$ to~$T$, we obtain
    $$
    (T+\tau)^2 r_T \leq  \tau^2 r_{0} +  \sum_{t=1}^{T}\beta_t (t + \tau )^2.
    $$
\end{proof}

\subsection{Technical lemma for decreasing stepsizes}

\begin{lemma}\label{le:prod_bound}
   Let $q \in [0,1]$ and let $\eta_t = \rb{ \fr{2}{t+2} }^{q}$ for every integer~$t$. Then for every integer~$t$ and any integer~$T \geq 1$ we have
    \begin{eqnarray}\label{eq:tech1}
    \eta_t (1 - \eta_{t+1}) \leq \eta_{t+1},
    \end{eqnarray}
    $$
    \prod_{t = 0}^{T-1} (1 - \eta_{t+1}) \leq \eta_T\,.
    $$
\end{lemma}

\begin{proof}
For every integer~$t$ we have
    \begin{eqnarray}
        1 - \eta_{t+1} &=& 1 - \rb{\fr{2}{t+3}}^{q} 
        \leq 1 - \fr{1}{t+3}  =
        \fr{t+2}{t+3}  \leq 
        \fr{\eta_{t+1}}{\eta_t} \notag .
    \end{eqnarray}
    Using the above result, we can write
    \begin{eqnarray*}
        \prod_{t = 0}^{T-1} (1 - \eta_{t+1}) &\leq&  \prod_{t = 0}^{T-1} \fr{\eta_{t+1}}{\eta_{t}} =  \fr{\eta_{T}}{\eta_{0}} = \eta_T.
    \end{eqnarray*}
\end{proof}

\begin{lemma}\label{le:sum_prod_bound1}
Let $q \in [0, 1)$, $p \geq 0$, $\beta_0 > 0$ and let $\eta_t =  \rb{ \fr{2}{t+2} }^q$, $\beta_t = \beta_0 \rb{ \fr{2}{t+2} }^p$ for every integer~$t$. Then for any integers~$t$ and~$T \geq 1$, it holds 
    $$
     \sum_{t = 0}^{T-1} \beta_{t} \prod_{\tau = t+1 }^{T-1} (1 - \eta_{\tau}) \leq C \beta_T \eta_T^{-1},
    $$
     where $C > 1$ is an absolute constant depending on $p$ and $q$.
\end{lemma}
\begin{proof}
See, for instance, \citep[Proposition~B.1]{Gadat_SHB_18} or \citep[Lemma 15]{Fatkhullin_SPG_FND_2023}. 
\end{proof}

\end{document}